\documentclass{article}

\usepackage{microtype}
\usepackage{graphicx}
\usepackage{subfigure}
\usepackage{booktabs} 

\usepackage{hyperref}

\usepackage{makecell}

\usepackage{colortbl}
\usepackage{color}
\definecolor{Cyan}{rgb}{0.88,1,1}

\usepackage[accepted]{icml2023}

\usepackage{amsmath}
\usepackage{amssymb}
\usepackage{mathtools}
\usepackage{amsthm}

\usepackage[capitalize,noabbrev]{cleveref}

\usepackage[textsize=tiny]{todonotes}

\usepackage{amsmath,amssymb,amsthm,amsfonts}
\usepackage{dsfont}

\theoremstyle{plain}
\newtheorem{theorem}{Theorem}[section]

\newtheorem{lemma}[theorem]{Lemma}

\theoremstyle{definition}

\newtheorem{assumption}[theorem]{Assumption}
\theoremstyle{remark}
\newtheorem{remark}[theorem]{Remark}

\newcommand{\Ac}{\mathcal{A}}

\newcommand{\Cc}{\mathcal{C}}
\newcommand{\Dc}{\mathcal{D}}

\newcommand{\Gc}{\mathcal{G}}

\newcommand{\Lc}{\mathcal{L}}
\newcommand{\Mc}{\mathcal{M}}
\newcommand{\Nc}{\mathcal{N}}

\newcommand{\Rc}{\mathcal{R}}
\newcommand{\Sc}{\mathcal{S}}

\newcommand{\Uc}{\mathcal{U}}

\newcommand{\bP}{\mathbf{P}}

\newcommand{\oneb}{\mathds{1}}

\newcommand{\Bb}{\mathbb{B}}

\newcommand{\Eb}{\mathbb{E}}

\newcommand{\Pb}{\mathbb{P}}
\newcommand{\Qb}{\mathbb{Q}}
\newcommand{\Rb}{\mathbb{R}}

\newcommand{\ber}[1]{\text{br}(#1)}

\newcommand{\KL}[1]{\text{KL}\left(#1\right)}
\DeclareMathOperator*{\argmax}{arg\,max}
\DeclareMathOperator*{\argmin}{arg\,min}
\DeclareMathOperator*{\essinf}{essinf}

\usepackage{hyperref}
\usepackage{natbib}

\allowdisplaybreaks

\ifodd 0
\newcommand{\shir}[1]{{\color{magenta}#1}}
\else
\newcommand{\shir}[1]{#1}
\fi

\icmltitlerunning{Provably Efficient Offline Reinforcement Learning with Perturbed Data Sources}

\begin{document}

\twocolumn[
\icmltitle{Provably Efficient Offline Reinforcement Learning with Perturbed Data Sources}



\icmlsetsymbol{equal}{*}

\begin{icmlauthorlist}
\icmlauthor{Chengshuai Shi}{uva}
\icmlauthor{Wei Xiong}{hkust}
\icmlauthor{Cong Shen}{uva}
\icmlauthor{Jing Yang}{psu}
\end{icmlauthorlist}

\icmlaffiliation{uva}{Department of Electrical and Computer Engineering, University of Virginia}
\icmlaffiliation{hkust}{Department of Mathematics, The Hong Kong University of Science and Technology}
\icmlaffiliation{psu}{Department of Electrical Engineering, The Pennsylvania State University}

\icmlcorrespondingauthor{Chengshuai Shi}{cs7ync@virginia.edu}
\icmlcorrespondingauthor{Cong Shen}{cong@virginia.edu}

\icmlkeywords{Offline Reinforcement Learning}

\vskip 0.3in
]



\printAffiliationsAndNotice{} 

\begin{abstract}
Existing theoretical studies on offline reinforcement learning (RL) mostly consider a dataset sampled directly from the target task. In practice, however, data often come from several heterogeneous but related sources. Motivated by this gap, this work aims at rigorously understanding offline RL with multiple datasets that are collected from \emph{randomly perturbed versions} of the target task instead of from itself. An information-theoretic lower bound is derived, which reveals a necessary requirement on the number of involved sources in addition to that on the number of data samples. Then, a novel HetPEVI algorithm is proposed, which simultaneously considers the \emph{sample uncertainties} from a finite number of data samples per data source and the \emph{source uncertainties} due to a finite number of available data sources. Theoretical analyses demonstrate that HetPEVI can solve the target task as long as the data sources \emph{collectively} provide a good data coverage. Moreover, HetPEVI is demonstrated to be optimal up to a polynomial factor of the horizon length. Finally, the study is extended to offline Markov games and offline robust RL, which demonstrates the generality of the proposed designs and theoretical analyses.
\end{abstract}

\section{Introduction}\label{sec:intro}
Offline reinforcement learning (RL) \citep{levine2020offline}, a.k.a. batch RL \citep{lange2012batch}, has received growing interest in recent years. It aims at training RL agents using accessible datasets collected \emph{a priori} and thus avoids expensive online interactions. 
Along with its tremendous empirical successes, recent studies have also advanced the theoretical understandings of offline RL \citep{rashidinejad2021bridging,jin2021pessimism, xie2021policy}.

Despite these progresses, most theoretical studies on offline RL focus on learning via data collected exactly from the target task. In practice, however, it is difficult to ensure such a perfect match. Instead, it is more reasonable to model that data are collected from different sources that are perturbed versions of the target task in some applications. For example, when training a chatbot \citep{jaques2020human}, the offline dialogue datasets typically consist of conversations from different people with naturally varying language habits. The training objective is the common underlying language structure, e.g., the basic grammar, which cannot be completely reflected in any individual dialogue but can be holistically learned from the aggregation of them. Similar examples can be found in healthcare with records from different hospitals \citep{tang2021model}, recommender systems with histories from different customers \citep{afsar2022reinforcement}, and others; more discussions are provided in Appendix~\ref{app:setting}.

While a few empirical investigations have been reported (in particular, under the offline meta-RL framework, e.g., \citet{dorfman2021offline,lin2022model,mitchell2021offline}), theoretical understandings of effectively and efficiently learning with heterogeneous while related data sources are lacking. Motivated by this limitation, this work makes progress in answering the following open question: 
\begin{center}
    \textit{Can we design provably efficient offline RL for a target task with multiple randomly perturbed data sources?}
\end{center}
\noindent\textbf{Challenges.}
Existing offline RL studies typically deal with one type of uncertainty, i.e., the \emph{sample uncertainty} associated with the finite data sampled directly from the target task, which results in distributional shift and partial coverage.
In addition to these, randomly perturbed data sources bring new challenges. First, since multiple data sources are involved, it is important to \emph{jointly aggregate} their sample uncertainties, and to leverage their \emph{collective information}. Moreover, even if every data source is perfectly known, the target task may \emph{not} be fully revealed as the data sources are perturbations of the target. Thus, importantly, an additional type of uncertainty due to a finite number of available data sources should be jointly considered, which is referred to as the \emph{source uncertainty}. 

\begin{table*}[tbh]
    \small
    \centering
    \caption{Related works and their studied settings; see Appendix~\ref{app:related} and Fig.~\ref{fig:related} for more discussions and graphical illustrations.}
    \label{tbl:related}
    \begin{tabular}{c|c|c}
        \hline
          & \textbf{Data source} & \textbf{Evaluation of the learned policy}\\
         \hline
         \Gape{\makecell[c]{Canonical offline RL \\ \citep{xie2021policy,li2022settling}}} & The target MDP & Performance on the target MDP\\
        \Gape{\makecell[c]{Offline robust RL \\\citep{shi2022distributionally,zhou2021finite}}} & The nominal MDP & \makecell[c]{Worst-case performance in \\an uncertainty set around the nominal MDP}\\
         \Gape{\makecell[c]{Offline latent RL \\ (Offline version of \citet{kwon2021rl})}} & \Gape{\makecell[c]{A set of potential MDPs}} & \Gape{\makecell[c]{Average performance on unknown MDPs \\randomly selected from the data source set}}\\
        \Gape{\makecell[c]{Offline federated/multi-task RL \\\citep{zhou2022federated,lu2021power}}} & \Gape{\makecell[c]{A set of task MDPs}} & \Gape{\makecell[c]{Performance of the learned \\ task-dependent policy on each task MDP}}\\
        \rowcolor[gray]{0.9} \Gape{\makecell[c]{Offline RL with perturbed data sources\\ (this work)}} & \Gape{\makecell[c]{A set of MDPs perturbed from\\ the target MDP (Assumption~\ref{asp:relationship})}} & Performance on the target MDP\\
        \hline
    \end{tabular}
\end{table*}

\noindent\textbf{Contributions.}
 To the best of our knowledge, this is the first theoretical work that studies the fundamental limits of offline RL with multiple perturbed data sources and develops provably efficient algorithm designs, which can benefit relevant applications of RL using multiple heterogeneous data sources (e.g., offline meta-RL).  
The contributions are summarized as follows: \\
$\bullet$ We study a new offline RL problem where the datasets are collected from multiple heterogeneous source Markov Decision Processes (MDPs), with possibly different reward and transition dynamics, as opposed to directly from the target MDP. Motivated by practical applications, the data source MDPs are modeled as randomly perturbed versions of the target MDP. \\
$\bullet$ A novel information-theoretic lower bound is derived. It illustrates that in addition to ensuring sufficient sample complexity (i.e., the amount of collected data samples), it is equally (if not more) important to guarantee sufficient source diversity (i.e., the number of involved data sources). This observation is new in the offline RL study and provides useful guidance for practical data collection.\\
$\bullet$ A novel HetPEVI algorithm is proposed with a carefully designed two-part penalty term to ensure pessimistic estimations during learning. Especially, the first part of the penalty jointly aggregates the sample uncertainties associated with each dataset, while the second one provides additional compensations for the source uncertainties, which is uniquely required to handle randomly perturbed data sources. \\
$\bullet$ Theoretical analysis demonstrates that as long as the perturbed data sources collectively provide a good data coverage, HetPEVI can learn the target task efficiently. This \emph{collective} coverage requirement is more practical than the previous \emph{individual} coverage requirement. In particular, it only requires that for each state-action pair induced by the optimal policy on the target task, there \emph{exists} a (potentially different) data source that can provide data samples for it. More importantly, compared with the lower bound, HetPEVI is shown to be optimal up to a polynomial factor of the horizon length regarding its requirements of sample complexity and source diversity. Additional experimental results further corroborate the effectiveness of HetPEVI.\\
$\bullet$ The design principle in HetPEVI is further extended to offline Markov games and offline robust RL with perturbed data sources, which showcases its generality. Importantly, these extensions further validate that learning with perturbed data sources is feasible given a good collective data coverage, while it requires guarantees of both sample complexity and source diversity.


\noindent \textbf{Related Works.} Theoretical understandings of offline RL \citep{levine2020offline} have been gaining increased interest in recent years, where the principle of ``pessimism'' plays an important role. In particular, \citet{xie2021policy, li2022settling, shi2022pessimistic, rashidinejad2021bridging} have investigated the standard tabular setting, while \citet{jin2021pessimism,yin2022near,xie2021bellman,uehara2021pessimistic} studied function approximations. Most of these theoretical advances assume that data are collected directly from the target task. On the other hand, practical RL research has seen growing interest in how to utilize data 
from heterogeneous sources, e.g., meta-RL \citep{mitchell2021offline,dorfman2021offline,lin2022model,li2020focal} and federated RL \citep{zhuo2019federated,jin2022federated}. As a first step to filling this theoretical gap, this work aims at understanding offline RL with multiple randomly perturbed sources. Some particularly related research domains in theoretical RL are discussed in Table~\ref{tbl:related}, which are compared with this work in the studied data sources and evaluation criteria. A detailed literature review can be found in Appendix~\ref{app:related}.

\section{Problem Formulation}\label{sec:formulation}
\subsection{Preliminaries of Episodic MDPs} 
We consider an episodic MDP $\Mc := (H, \Sc, \Ac, \Pb:=\{\Pb_h: h\in [H]\}, r:=\{r_h: h\in [H]\})$. In this tuple, $H$ is the length of each episode, $\Sc$ is the state space with $S:=|\Sc|$, $\Ac$ is the action space with $A := |\Ac|$, $\Pb_h(s'|s,a)$ gives the probability of transiting to state $s'$ if action $a$ is taken upon state $s$ at step $h$, and $r_h(s,a)$ is the deterministic reward in the interval of $[0,1]$ of taking action $a$ for state $s$ at step $h$.\footnote{The assumption of deterministic rewards is standard in theoretical RL \citep{jin2018q,jin2020provably} as the uncertainties in estimating rewards are dominated by those in estimating transitions.}
Specifically, at each step $h\in [H]$, the agent observes state $s_h\in \Sc$, picks action $a_h\in \Ac$, receives reward $r_h(s_h,a_h)$, and then transits to a next state $s_{h+1}\sim \Pb_h(\cdot|s_h,a_h)$. 

A policy $\pi := \{\pi_h(\cdot|s):(s,h)\in\Sc\times [H]\}$ consists of distributions $\pi_h(\cdot|s)$ over the action space $\Ac$. For convenience, we use $\pi_h(s)$ to refer to the chosen action at $(s,h)\in \Sc\times [H]$ by a deterministic policy $\pi$. To measure the performance, the value function of the policy $\pi$ is defined as $V^{\pi,\Mc}_h(s):= \Eb_{\pi,\Mc}[\sum\nolimits_{i=h}^H r_{i}(s_{i},a_{i})|s_h = s]$ for all $(s, h) \in \Sc\times [H]$,
where the expectation $\Eb_{\pi, \Mc}[\cdot]$ is with respect to (w.r.t.) the random trajectory induced by policy $\pi$ on MDP $\Mc$. Similarly, the $Q$-function of $\pi$ can be defined as $Q^{\pi,\Mc}_h(s,a):= \Eb_{\pi,\Mc}[\sum\nolimits_{i=h}^H r_{i}(s_{i},a_{i})|s_h = s, a_h = a]$ for all $(s,a,h)\in \Sc\times \Ac\times [H]$.
If the initial state is drawn from a distribution $\xi\in\Delta(\Sc)$, the following notation is adopted: $V^{\pi,\Mc}_1(\xi): = \Eb_{s\sim \xi}[V^{\pi, \Mc}_1(s)]$.

To characterize the state and state-action occupancy distribution induced by policy $\pi$ on MDP $\Mc$ at each step, we denote that $d^{\pi,\Mc}_h(s;\xi): = \Eb_{\pi,\Mc} \left[\oneb\left\{s_h = s\right\}|s_1\sim \xi\right]$ and $d^{\pi, \Mc}_h(s,a; \xi) := \Eb_{\pi,\Mc} \left[\oneb\left\{s_h = s, a_h = a\right\}|s_1\sim \xi\right]$,
where the expectation is conditioned on $s_1\sim \xi$. Whenever it is clear from the context, we simplify the notations as $d^{\pi,\Mc}_h(s) := d^{\pi,\Mc}_h(s;\xi)$ and $d^{\pi,\Mc}_h(s, a) := d^{\pi,\Mc}_h(s, a;\xi)$.

\subsection{Learning Goal}\label{subsec:target}
This work considers a target task modeled by an MDP $\Mc = (H, \Sc, \Ac, \Pb, r)$ as introduced above. The goal is to find a good policy for this target MDP using certain existing datasets, i.e., offline learning. Especially, the sub-optimality gap of a policy $\hat\pi$ on $\Mc$ w.r.t. $s_1\sim \xi$ is defined as follows:
\begin{equation}\label{eqn:gap_def}
    \text{Gap}(\hat\pi;\Mc, \xi) := V^{\pi^*,\Mc}_{1}(\xi) - V^{\hat\pi, \Mc}_{1}(\xi),
\end{equation}
where $\pi^*:= \argmax_{\pi} V^{\pi,\Mc}_1(\xi)$ is the optimal (deterministic) policy on the target MDP $\Mc$. Correspondingly, an output policy $\hat\pi$ is called $\varepsilon$-optimal if $\text{Gap}(\hat\pi;\Mc, \xi)\leq \varepsilon$. 

\subsection{Data Sources and The Task-Source Relationship}
Instead of assuming data sampled directly from the unknown target task MDP $\Mc$, this work considers that the learning agent has access to datasets from $L$ different data sources. Each data source is also an unknown MDP, and the $l$-th data source can be represented as $\Mc_l = (H, \Sc, \Ac,  \Pb_l, r_l)$. To capture heterogeneity, each data source $\Mc_l$ may not exactly match the target task $\Mc$. Concretely, despite the same episodic length, state space, and action space,  their transition and reward dynamics are not necessarily aligned, i.e., possibly, $\Pb_{h,l}(\cdot|s,a)\neq \Pb_{h}(\cdot|s,a)$ and $r_{h,l}(s,a)\neq r_{h}(s,a)$.

In practical applications, while being heterogeneous, the data source MDPs are often still related to the target task (e.g., the dialogue dataset example in Section~\ref{sec:intro}). In particular, data sources in offline meta-RL are often assumed to be sampled from one certain distribution \citep{mitchell2021offline}. Thus, the following relationship is considered between the target MDP and the data source MDPs.  
\begin{assumption}[Task--source relationship]\label{asp:relationship}
    Data source MDPs $\{\Mc_l = \{H, \Sc, \Ac, \Pb_{l}, r_l\}: l\in [L]\}$ are generated from an unknown set of distributions $g = \{g_h: h\in [H]\}$ such that for each $(l, h)\in [L]\times [H]$, the reward and transition $\{r_{h,l}, \Pb_{h,l}\}$ are independently sampled from the distribution $g_h(\cdot)$ whose expectation is $\{r_{h}, \Pb_{h}\}$ of the target MDP $\Mc=\{H, \Sc, \Ac, \Pb, r\}$.
\end{assumption}
The requirement that rewards are random samples with the expectation as the target task is commonly adopted in bandits literature \citep{shi2021federated,zhu2022random}, and the same requirement on the transition dynamics is a natural extension, where one representative example is to follow a Dirichlet distribution \citep{marchal2017sub}.

\shir{\begin{remark}\label{rmk:worst}
    This work essentially considers a ``worst-case'' scenario in the sense that our proposed designs and obtained results are for any generation process satisfying Assumption~\ref{asp:relationship}, i.e., the generation process $g$ exists and has an expectation as the target task $\Mc$; the other properties of $g$ (e.g., its variance) are not specified but our designs and results still hold.
\end{remark}}        

\begin{figure}[tbh]
	\centering
 \setlength{\abovecaptionskip}{-6pt}\includegraphics[width=0.4\textwidth]{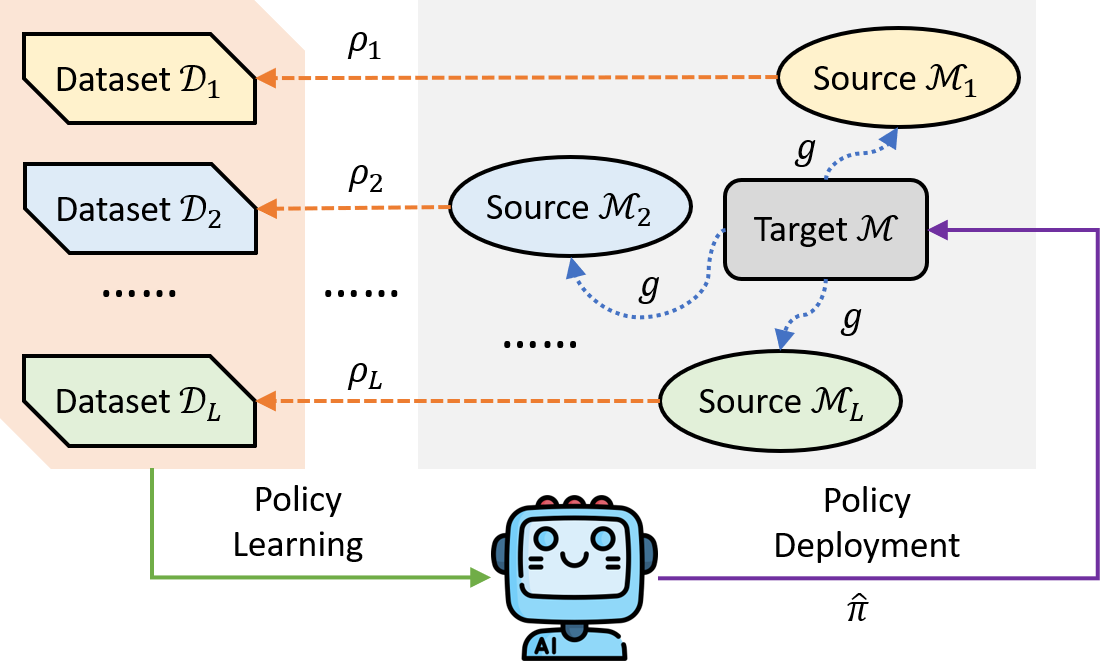}
 \caption{Problem overview: dotted (blue and orange) lines indicate the collection of datasets (with randomness from both source generation and data sampling), while the solid (green and purple) lines are for policy learning and deployment. Especially, the agent aims at solving a target MDP $\Mc$ but lacks direct access. Instead, available datasets $\{\Dc_l: l\in [L]\}$ are collected via behavior policies $\{\rho_l: l\in [L]\}$  from data source MDPs $\{\Mc_l: l\in [L]\}$ that are randomly perturbed from the target MDP $\Mc$ (through distribution $g$ in Assumption~\ref{asp:relationship}). With such datasets, the agent learns a policy $\hat{\pi}$ offline, which is deployed (potentially in the future) on the target MDP $\Mc$ with its performance gap measured by Eqn.~\eqref{eqn:gap_def}. See Fig.~\ref{fig:model_full} for additional graphical illustrations.}
	\label{fig:model}
\end{figure}

\subsection{Collections of Datasets}
We consider that from each data source $\Mc_l$, a dataset $\Dc_l := \{(s^{k}_{1,l}, a^{k}_{1,l}, r^k_{1,l}, \cdots, s^{k}_{H,l}, a^{k}_{H,l}, r^{k}_{H,l}):k\in [K]\}$ is collected, which consists of $K$ independent trajectories sampled by a (possibly different) unknown behavior policy $\rho_l$ with a (possibly different) initial state distribution $\xi_l$. More specifically, the $k$-th trajectory in dataset $\Dc_l$ is generated according to $s^k_{1,l} \sim \xi_l(\cdot)$, $a^k_{h,l} \sim \rho_{h,l}(\cdot|s^{k}_{h,l})$, $r^k_{h,l} = r_{h,l}(s^k_{h,l},a^{k}_{h,l})$, and $s^k_{h+1,l}\sim \Pb_{h,l}(\cdot|s^k_{h,l}, a^k_{h,l})$.

It can be observed that although the trajectories in the collected datasets are independently collected, the sampled transitions from the same episode are still correlated. A two-fold sub-sampling technique is developed in \citet{li2022settling} to alleviate such temporal dependencies and re-create datasets where the sampled transitions is independent of each other. To ease the presentation, we denote $\{\Dc'_l: l\in [L]\}$ as the datasets re-created from the original  $\{\Dc_l: l\in [L]\}$ with the two-fold sub-sampling. Details on the sub-sampling technique are provided in Appendix~\ref{app:subsampling}.

A compact overview of the studied offline learning problem is provided in Fig.~\ref{fig:model}, whose complete step-by-step version ca be found in Fig.~\ref{fig:model_full}.

\noindent\textbf{Miscellaneous.} 
Notations without subscripts $l$ generally refer to the target MDP $\Mc$, while subscript $l$ is added when discussing each individual data source $\Mc_l$. For any function $f:\Sc\to \Rb$, the transition operator and Bellman operator of the target MDP $\Mc$ at each step $h\in [H]$ are defined, respectively, as $(\Pb_h f)(s,a) := \Eb_{s' \sim \Pb_h(\cdot|s,a)}[f(s')|s,a]$ and $(\Bb_h f)(s,a) := r_h(s,a) + (\Pb_h f)(s,a)$. The notation $c$ is used throughout the paper with varying values to represent a constant of order $O(1)$. Lastly, the notation $y \gtrsim x$ compactly denotes that $y \geq x \log(KHSAL/\delta)$, where $\delta$ is a constant in $(0,1)$, while $x\vee y: = \max\{x,y\}$ and $x\wedge y := \min\{x,y\}$.

\section{Lower Bound Analysis}
With $\Lc_h(s,a): = \{l\in [L]: d^{\rho_l, \Mc_l}_{h}(s,a; \xi_l)>0\}\subseteq [L]$ as the set of data sources that can visit $(s, a, h)$, the following two quantities are introduced: the minimum number of sources that cover each possible visitations of the optimal policy $\pi^*$ on the target MDP $\Mc$ is defined as
    \begin{align*}
        L^\dagger:= \min\left\{|\Lc_h(s,a)|: (s, a, h) \text{ s.t. } d^{\pi^*, \Mc}_h(s,a;\xi) > 0\right\},
    \end{align*}
    and the collective coverage parameter is defined as
    \begin{equation*}
        C^{\dagger}:=\max_{(s, a, h)}\left\{\sum_{l\in \Lc_h(s,a)}\frac{\min\left\{d^{\pi^*, \Mc}_{h}(s,a;\xi), \frac{1}{S}\right\}}{|\Lc_h(s,a)|\cdot d^{\rho_l, \Mc_l}_{h}(s,a;\xi_l)}\right\},
    \end{equation*} 
    where we adopt the convention that $0/0=0$. Note that $L^{\dagger}$ captures how many sources provide information on the optimal policy by counting the useful ones, while $C^\dagger$ further characterizes how well these data sources collectively provide information by comparing their aggregated occupancy probability with the optimal policy. The clipping with $1/S$ follows the definition of the single-policy clipped coverage parameter recently proposed in \citet{li2022settling}.
    
   The following novel information-theoretic lower bound is established to provide fundamental limits.
\begin{theorem}\label{thm:lower_bound}
    For any $(H, S, L^{\dagger}, C^{\dagger}, \varepsilon)$ obeying $H\geq 4$, $C^\dagger\geq 4/S$ and $\varepsilon \leq c_0H$, 
    if either of the following two conditions is not satisfied:
    \begin{align*}
        L^\dagger K \geq c_1\frac{C^\dagger H^2S}{\varepsilon^2},\qquad  
        L^\dagger \geq c_2\frac{H^2}{\varepsilon^2},
    \end{align*}
    one can construct two target MDPs $\{\Mc^0, \Mc^1\}$, an initial state distribution $\xi$, a data source generation distribution set $g$, and datasets $\{\Dc_l: l\in [L]\}$, such that
    \begin{align*}
        \inf\nolimits_{\hat{\pi}}\max\nolimits_{\phi\in \{0,1\}}\left\{\bP_{\phi}\left({\normalfont\text{Gap}}(\hat\pi;\Mc^{\phi}, \xi) > \varepsilon \right)\right\} \geq \frac{1}{8},
    \end{align*}
    where $c_0$, $c_1$ and $c_2$ are universal constants, the infimum is taken over all estimators $\hat{\pi}$, and $\bP_0$ (resp. $\bP_1$) denotes the probability when the target MDP is $\Mc^0$ (resp. $\Mc^1$).
\end{theorem}
It can be observed that this lower bound has two requirements: (1) sample complexity, i.e., $L^\dagger K  = \Omega(C^\dagger H^2S/\varepsilon^2)$; (2) source diversity, i.e., $L^\dagger = \Omega(H^2/\varepsilon^2)$. While similar requirements on sample complexity have appeared in previous studies \cite{li2022settling, rashidinejad2021bridging}, the one derived here is established collectively through $C^\dagger$ and $L^\dagger$ on the aggregation of heterogeneous data sources with different behavior policies. To the best of our knowledge, the second requirement on source diversity appears in offline RL studies for the first time. It provides a key observation that without enough data sources that provide useful information, the target MDP cannot be efficiently learned even with infinite data samples from each data source. The combination of these two requirements indicates that it is equally (if not more) important to involve sufficient high-quality sources as to sample adequate data from each of them, which is a helpful principle to guide practical data collection.

\begin{algorithm}[tbh]
	\caption{HetPEVI}
	\label{alg:HetPEVI}
	\begin{algorithmic}[1]
	    \STATE {\bfseries Input:} Dataset $\Dc = \{D_l:l\in[L]\}$
     \STATE Obtain $\Dc'_l \gets \text{subsampling}(\Dc_l), \forall l\in [L]$
     \STATE For each $(s, a, h)\in \Sc \times \Ac \times [H]$, first obtain $\hat{\Lc}_h(s,a)$; then for each $l\in \hat{\Lc}_h(s,a)$, estimate $\hat{r}_{h,l}(s,a)$ and $\hat{\Pb}_{h,l}(\cdot|s,a)$; lastly, aggregate $\hat{r}_h(s,a)$ and $\hat{\Pb}_h(\cdot|s,a)$ \hfill\COMMENT{\textit{See Section~\ref{subsec:estimations}}}
     \STATE Initialize $\hat{V}_{H+1}(s) \gets 0, \forall s\in \Sc$ \hfill\COMMENT{\textit{See Section~\ref{subsec:PEVI}}}
	    \FOR{$h = H, H-1,\cdots, 1$} 
        \FOR{$(s,a) \in \Sc\times \Ac$}
        \STATE $\Gamma^\alpha_h(s,a) \gets c\sqrt{\sum\nolimits_{l\in \hat{\Lc}_h(s,a)}\frac{H^2\log(SAH/\delta)}{(\hat{L}_h(s,a))^2N_{h,l}(s,a)}}$
        \STATE $\Gamma^\beta_h(s,a) \gets c\sqrt{H^2\log(SAH/\delta)/{\hat{L}_h(s,a)}}$
        \STATE $\Gamma_h(s,a) \gets \min\{\Gamma^\alpha_h(s,a) + \Gamma^\beta_h(s,a), H\}$
        \STATE $\hat Q_h(s,a) \gets \max\{(\hat{\Bb}_h \hat{V}_{h+1})(s,a)- \Gamma_h(s,a), 0\}$
        \ENDFOR
        \FOR{$s\in \Sc$}
        \STATE $\hat{\pi}_h(s) \gets \argmax_{a\in\Ac} \hat{Q}_h(s,a)$
        \STATE $\hat{V}_h(s) \gets \hat{Q}_h(s,\hat{\pi}_h(s))$
	    \ENDFOR
        \ENDFOR
	    \STATE {\bfseries Output:} policy $\hat\pi = \{\hat\pi_h(s): (s,h)\in \Sc\times [H]\}$
	\end{algorithmic}
	\end{algorithm}

\section{The HetPEVI Algorithm}
In this section, we present a  novel model-based algorithm, termed HetPEVI, to perform offline RL with perturbed data sources, which is summarized in Algorithm~\ref{alg:HetPEVI}.

\subsection{Constructing Empirical Estimations}\label{subsec:estimations}
The HetPEVI algorithm begins by counting the number of visitations in the available datasets. Especially, we denote $N_{h,l}(s,a)$ and $N_{h,l}(s,a,s')$ as the amount of visitations on each tuple $(s,a,h)$ and $(s,a,h,s')$ in dataset $\Dc'_{l}$, respectively. Then, the subset of datasets that have non-zero visitations on tuple $(s,a,h)$ can be found as $
    \hat{\Lc}_h(s,a) := \{l \in [L]: N_{h,l}(s,a) >0\}$,
whose size is denoted as $\hat{L}_h(s,a) := |\hat{\Lc}_h(s,a)|$.
Empirical estimations of rewards and transitions are then obtained for each tuple $(s,a,h)\in \Sc\times \Ac\times [H]$ and each source $l\in \hat{\Lc}_h(s,a)$ as $\hat{r}_{h,l}(s,a) = r_{h,l}(s,a)$ and $\hat{\Pb}_{h,l}(s'|s,a) = N_{h,l}(s,a,s')/N_{h,l}(s,a)$.
These individual estimates are further aggregated into overall estimates for each tuple $(s, a, h, s')\in \Sc\times \Ac\times [H]\times \Sc$ as $\hat{r}_h(s,a)  = \sum\nolimits_{l\in \hat{\Lc}_h(s,a)}\hat{r}_{h,l}(s,a)/(\hat{L}_h(s,a)\vee 1)$ and $   \hat{\Pb}_h(s'|s,a) = \sum\nolimits_{l\in \hat{\Lc}_h(s,a)}\hat{\Pb}_{h,l}(s'|s,a)/(\hat{L}_h(s,a)\vee 1)$.
Note that in these estimations of
transitions and rewards, only the data sources that provide non-zero visitations are counted, which may differ for each tuple $(s, a, h)$.

\subsection{Considering Two Types of Uncertainties}\label{subsec:PEVI}
With the obtained estimations, HetPEVI iterates backward from the last step to the first step as
\begin{align}
        &\hat Q_h(s,a) = \max\big\{\big(\hat{\Bb}_h \hat{V}_{h+1}\big)(s,a)- \Gamma_h(s,a), 0\big\},\label{eqn:VI}\\
	        &\hat\pi_h(s) = \arg\max\nolimits_{a\in\Ac} \hat{Q}_h(s,a),\quad \hat{V}_h(s) = \hat{Q}_h(s,\hat\pi_h(s))\notag,
\end{align}
with $\hat{V}_{H+1}(s) = 0, \forall s\in \Sc$, and the empirical Bellman operator $\hat\Bb_h$  defined as $(\hat\Bb_h \hat{V}_{h+1}\big)(s,a) : = \hat{r}_h(s,a) + \big(\hat\Pb_h \hat{V}_{h+1})(s,a)$,
where $(\hat\Pb_h \hat{V}_{h+1})(s,a)$ is the empirical version of $(\Pb_h \hat{V}_{h+1})(s,a)$ using the estimated $\hat\Pb_h(\cdot|s,a)$.
The essence of this procedure is that instead of directly setting  $\hat{Q}_h(s,a)$ as $(\hat{\Bb}_h \hat{V}_{h+1})(s,a)$ (as in the standard value iteration), a penalty term $\Gamma_h(s,a)$ is subtracted, which serves the important role of keeping the estimations $\hat{V}_h(s)$ and $\hat{Q}_h(s,a)$ pessimistic and providing conservative actions. 

Previous offline RL studies \citep{jin2021pessimism, rashidinejad2021bridging,xie2021policy,li2022settling} only deal with one single data source (i.e., the target MDP) and thus only one type of uncertainty due to a finite number of data samples. Instead, the agent in this work needs to process multiple heterogeneous datasets, while none of them individually characterize the target task. Thus, as mentioned in Section~\ref{sec:intro}, the agent faces two \emph{coupled} uncertainties. First, the \emph{sample uncertainties} associated with each data source need to be jointly aggregated instead of being measured individually to leverage collective information. Second, even with perfect knowledge of each data source, the target MDP may not be fully revealed. As a result, the agent also needs to consider the uncertainties from the limited number of data sources, i.e., the \emph{source uncertainties}.

To address the two uncertainties, the penalty term is designed to have two parts as follows:
\begin{equation*}
    \Gamma_h(s,a) = \min\big\{\Gamma^\alpha_h(s,a) + \Gamma^\beta_h(s,a), H\big\},
\end{equation*}
where $\Gamma^\alpha_h(s,a)$ aggregates the sample uncertainties while $\Gamma^\beta_h(s,a)$ accounts for the source uncertainties.

\noindent\textbf{Penalties to Aggregate Sample Uncertainties.} The first part of the penalty, i.e., $\Gamma^\alpha_h(s,a)$, is designed as
\begin{equation*}
    \Gamma^\alpha_h(s,a) = c\sqrt{\frac{1}{(\hat{L}_h(s,a))^2}\sum\nolimits_{l\in \hat{\Lc}_h(s,a)}\frac{H^2\log(SAH/\delta)}{ N_{h,l}(s,a)}}.
\end{equation*}
Note that this design avoids the data sources that have zero visitations on this tuple $(s, a, h)$ and is a joint measure of sample uncertainties from the other sources (instead of directly summing up their individual sample uncertainties as $\tilde{O}(\frac{1}{\hat{L}_h(s,a)} \sum_{l\in \hat{\Lc}_h(s,a)}\sqrt{\frac{H^2}{N_{h,l}(s,a)}})$). These designs are important to accelerate learning and obtain the near-optimal performance illustrated later in Section~\ref{sec:theory}.

\noindent\textbf{Penalties to Account for Source Uncertainties.}
The second part of the penalty $\Gamma^\beta_h(s,a)$ serves the important role of measuring the uncertainties due to the limited amount of data sources, which is designed as:
\begin{equation*}
    \Gamma^\beta_h(s,a) = c\sqrt{\frac{H^2\log(SAH/\delta)}{\hat{L}_h(s,a)}}.
\end{equation*}
Intuitively, it shrinks with the number of datasets that provides information on the tuple $(s, a, h)$, i.e., $\hat{L}_h(s,a)$, which thus may differ among state-action pairs. Jointly using the two penalty terms, the overall uncertainties in the datasets can be compensated. Especially, with high probability, for all $(s,a,h)\in \Sc\times \Ac\times [H]$, it can be ensured that  $|(\hat{\Bb}_h \hat{V}_{h+1})(s,a) - (\Bb_h \hat{V}_{h+1})(s,a) |\leq \Gamma_{h}(s,a)$.

\begin{remark}\label{rmk:adaptive_design}
    The adopted penalty $\Gamma^\beta_h(s,a)$ for source uncertainty is intended to accommodate any unknown variance between sources and the task, i.e., a worst-case consideration \shir{as mentioned in Remark~\ref{rmk:worst}}. However, if there is prior knowledge of the variance, it is feasible to incorporate such information. In particular, if the rewards and transition vectors \shir{(at each $(s,a,h)$)} are generated via $\sigma_g$-sub-Gaussian distributions, the penalty for source uncertainties can be designed as $\Gamma^\beta_h(s,a) = c\sqrt{\frac{H^2\sigma_g^2\log(SAH/\delta)}{\hat{L}_h(s,a)}}$.
\end{remark}

\section{Performance Analysis}\label{sec:theory}
This section provides a theoretical analysis of HetPEVI. In particular, the following performance guarantee can be established.
\begin{theorem}[HetPEVI]\label{thm:HetPEVI}
    Under Assumption~\ref{asp:relationship}, with probability at least $1-\delta$, the output policy $\hat\pi$ of HetPEVI satisfies
    \begin{equation*}
        {\normalfont\text{Gap}}(\hat\pi;\Mc, \xi) = \tilde{O}\left(\sqrt{\frac{C^\dagger H^4S}{L^\dagger K}} +\sqrt{\frac{H^4}{L^\dagger}}\right),
    \end{equation*}
    when $K \gtrsim c/d^{\min}$, where $d^{\min}:=\min\{d^{\rho_l, \Mc_l}_h(s, a): (s, a, h, l)  \text{ s.t. } d^{\pi^*,\Mc}_h(s, a)>0, d^{\rho_l, \Mc_l}_h(s, a)>0\}$.
\end{theorem}
It can be observed that as long as  $C^\dagger< \infty$ and $L^\dagger >0$, a meaningful performance gap can be provided by Theorem~\ref{thm:HetPEVI}. To ensure these two conditions, it is equivalent to have the following assumption of collective coverage.
\begin{assumption}[Collective coverage, this work]\label{asp:collective_tabular}
    For each $(s,a,h)\in \Sc\times \Ac\times [H]$ that $d^{\pi^*,\Mc}_{h}(s,a;\xi)>0$, there \emph{exists} $l\in [L]$ such that the behavior policy $\rho_l$ satisfy that $d^{\rho_l, \Mc_l}_h(s,a;\xi_l)>0$.
\end{assumption}
It is beneficial to compare Assumption~\ref{asp:collective_tabular} with those from previous offline RL studies. Especially, with a dataset $\Dc$ sampled with a behavior policy $\rho$ and an initial state distribution $\xi'$ directly from the target MDP $\Mc$, the following individual coverage assumption is often required. 
\begin{assumption}[Individual coverage,  \citet{rashidinejad2021bridging,xie2021policy,li2022settling}]\label{asp:individual_tabular}
    For each $(s,a,h)\in \Sc\times \Ac\times [H]$ that $d^{\pi^*,\Mc}_{h}(s,a;\xi)>0$, the behavior policy $\rho$ satisfy that $d^{\rho, \Mc}_h(s,a;\xi')>0$.
\end{assumption}
Assumption~\ref{asp:individual_tabular} is strong as the behavior policy needs to individually cover the unknown optimal policy. Instead, Assumption~\ref{asp:collective_tabular} is more practical because it leverages collective information: different parts of the optimal trajectory can be covered by different behavior policies and data sources. In particular, it may be easier to reach some states and actions in certain data sources with their corresponding behavior policies.

Moreover, besides the burn-in cost of $K\gtrsim c/d^{\min}$ which does not scale with $\varepsilon$, Theorem~\ref{thm:HetPEVI} illustrates that to obtain an $\varepsilon$-optimal policy, HetPEVI only needs 
\begin{equation*}
    L^\dagger K = \tilde{O}\left(\frac{C^\dagger H^4S}{\varepsilon^2}\right) \quad \text{and } \quad L^\dagger = \tilde{O}\left(\frac{H^4}{\varepsilon^2}\right),
\end{equation*}
where the first requirement is on the \emph{sample complexity} while the second one is on the \emph{source diversity}. 
Compared with the lower bound in Theorem~\ref{thm:lower_bound}, it can be observed that HetPEVI is optimal (up to logarithmic factors) on the dependency of $L^\dagger, C^\dagger, K, S, \varepsilon$, and only incurs an additional $H^2$ factor, which demonstrates its effectiveness and efficiency. \shir{Note that if we further leverage prior variance information and adopt the adaptive penalty as illustrated in Remark~\ref{rmk:adaptive_design}, a corresponding variance-adaptive performance guarantee can be established following the proof of Theorem~\ref{thm:HetPEVI}.}


\section{Experimental Results}
To further empirically validate the effectiveness of HetPEVI, experimental results are reported in Fig.~\ref{fig:performance}. In particular, simulations are performed on a target MDP with $S = 2$, $A = 20$ and $H = 20$. The data source MDPs are randomly generated through a set of independent Dirichlet distributions \citep{marchal2017sub}. The baseline, labeled as `Avg-PEVI', is an aggregated policy. In particular, with each individual dataset $\Dc_l$, a policy $\hat{\pi}_l$ is learned via PEVI \citep{xie2021policy,jin2021pessimism}. Then, at each $(s, h)\in \Sc\times [H]$, Avg-PEVI selects an action uniformly among $\{\hat{\pi}_{h,l}(s): l\in [L]\}$. Additional experimental details can be found in Appendix~\ref{app:exp}. 

From Fig.~\ref{fig:performance}, it can be observed that HetPEVI consistently outperforms the baseline policy, and is capable of approaching the optimal performance particularly when sufficient data sources and data samples are available.

\begin{figure}[tb]
    \setlength{\abovecaptionskip}{-6pt}
	\centering
	 \includegraphics[width=0.7\linewidth]{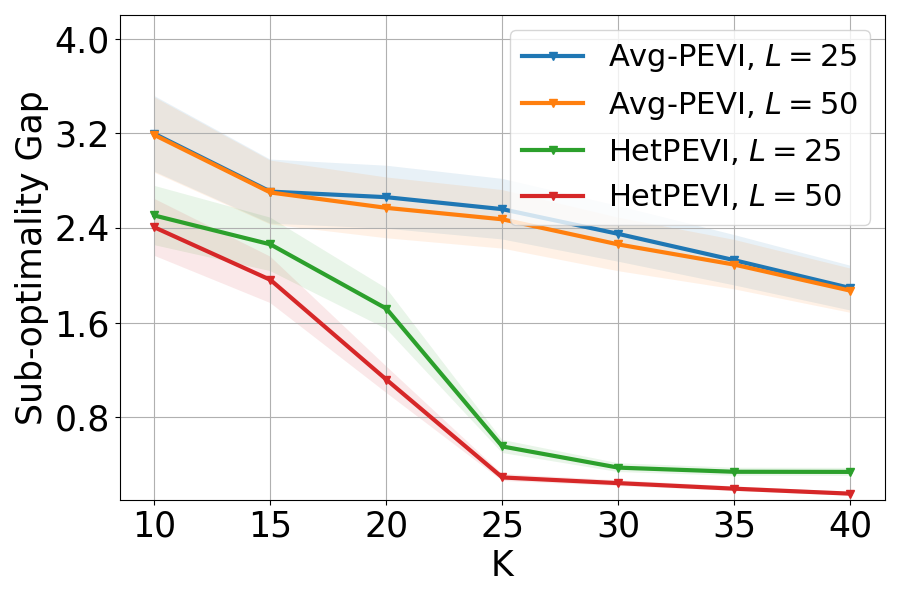}
	\caption{\small Performance comparisons between HetPEVI and the baseline with varying amounts of data samples and data sources.}
	\label{fig:performance}
\end{figure}

\section{Extension to Offline Markov Game}
We extend the study to the multi-player regime. In particular, the most representative two-player zero-sum Markov game (MG) is considered, which can be characterized by $\Gc := \{H, \Sc, \Ac := \Ac^1\times \Ac^2, \Pb:=\{\Pb_h:h\in [H]\}, r:=\{r_h:h\in [H]\}\}$. Besides $H$ as episode length and $\Sc$ as the state space, the major distinction in MG compared with MDP is that the entire action space $\Ac$ is a product of the individual action spaces of two players, i.e., $\Ac^1$ for the max-player and $\Ac^2$ for the min-player, respectively. We further denote that $A^1 := |\Ac^1|$ and $A^2 := |\Ac^2|$. Consequently, the transitions and rewards depend on the action pair $a = (a^1, a^2) \in \Ac^1\times \Ac^2$. Value functions of MG $\Gc$ can be defined similarly as those of MDP to be $V^{\pi, \Gc}_h(s)$ and $Q^{\pi, \Gc}_h(s,a)$ for a product policy $\pi = \mu\times \nu$ from the max-player's policy $\mu$ and the min-player's policy $\nu$. 

With $\Gc$ as the target MG, we are interested in approximating its Nash equilibrium (NE) policy pair $\pi^*=(\mu^*, \nu^*)$, where $\mu^*$ and $\nu^*$ are the best responses to each other. In particular, from the perspective of the max-player\footnote{The min-player's perspective is symmetrical and thus can be similarly solved with minor modifications on notations.}, the performance gap of a learned policy $\hat{\mu}$ with an initial state distribution $\xi$ is defined as
$\text{MGGap}(\hat{\mu}; \Gc, \xi) := V^{\mu^*\times \nu^*, \Gc}_1(\xi) - V^{ \hat{\mu}\times \ber{\hat{\mu}}, \Gc}_1(\xi)$,
where $\ber{\mu}$ denotes the best response of policy $\mu$, i.e., $\ber{\mu} = \argmin_{\nu} V^{\mu\times \nu, \Gc}_1(\xi)$. 

However, instead of having data directly sampled from the target MG as in \citet{cui2022offline, cui2022provably, zhong2022pessimistic, yan2022model}, we consider that there are $L$ data source MGs $\{\Gc_l: l\in [L]\}$ while $K$ trajectories being independently sampled from each data source MG $\Gc_l$ by a behavior policy $\rho_l$ and an initial state distribution $\xi_l$, where $\rho_l := \rho_l^1\times \rho_l^2$ with $\rho_l^1$ for the max-player and $\rho_l^2$ for the min-player. A task-source relationship similar to Assumption~\ref{asp:relationship} is considered between $\{\Gc_l: l\in [L]\}$ and $\Gc$, which is rigorously stated in Assumption~\ref{asp:relationship_game}.

\subsection{Algorithm Design and Analysis}
The HetPEVI-Game algorithm is generalized from HetPEVI to perform efficient offline learning in MG with multiple perturbed data sources. With algorithm details in Appendix~\ref{app_sub:alg_detail_game}, we note that HetPEVI-Game inherits most parts of HetPEVI while the major distinction being in the value iteration. In particular, the following is performed instead of Eqn.~\eqref{eqn:VI}:
\begin{align*}
        &\hat Q_h(s,a) = \max\big\{\big(\hat{\Bb}_h \hat{V}_{h+1}\big)(s,a)- \Gamma^g_h(s,a), 0\big\},\\
        &\left(\hat{\mu}_h(\cdot|s), \hat{\nu}_h(\cdot|s)\right) = \text{NE}(\hat{Q}_h(s,\cdot)),\\
	        &\hat{V}_h(s) = \Eb_{a\sim \hat{\mu}_h(\cdot|s)\times \hat{\nu}_h(\cdot|s)}[\hat{Q}_h(s,a)], 
    \end{align*}
where $\text{NE}(\cdot)$ finds the NE policy pair regarding the input and $\Gamma^g_h(s,a)$ is designed to have the same two-part structure as $\Gamma_h(s,a)$ of HetPEVI. Again, the two parts in $\Gamma^g_h(s,a)$ jointly capture the sample uncertainties and source uncertainties associated with the available datasets.

Similar to $C^\dagger$ and $L^\dagger$, we introduce the following quantities, $L^\dagger_g$ and $C^\dagger_g$, to measure the collective data coverage in MG:
\begin{equation*}\small
\begin{aligned}
    L^\dagger_g&:= \min\left\{|\Lc_h(s,a)|: (s, a, h) \text{ s.t. } \exists \nu, d^{\mu^*
        \times \nu, \Gc}_h(s,a;\xi) > 0\right\}, \\
        C^{\dagger}_g&:=\max_{(s, a, h)}\max_{\nu}\left\{\sum_{l\in \Lc_h(s,a)}\frac{\min\left\{d^{\mu^*\times \nu, \Gc}_{h}(s,a;\xi), \frac{1}{SA_1}\right\}}{|\Lc_h(s,a)|\cdot d^{\rho_l, \Gc_l}_{h}(s,a;\xi_l)}\right\},
\end{aligned}
\end{equation*}
where we reload the notation that $\Lc_h(s,a): = \{l\in [L]: d^{\rho_l, \Gc_l}_{h}(s,a;\xi_l)>0\}\subseteq [L]$. Then, the performance guarantee for HetPEVI-Game is established in the following theorem.
    
\begin{theorem}[HetPEVI-Game]\label{thm:HetPEVI_game}
    Under Assumption~\ref{asp:relationship_game}, with probability at least $1-\delta$, the output policy $\hat\mu$ of HetPEVI-Game satisfies
    \begin{equation*}
        {\normalfont\text{MGGap}}(\hat\mu;\Mc, \xi) = \tilde{O}\left(\sqrt{\frac{C^\dagger_g H^4SA_1}{L^\dagger_g K}} +\sqrt{\frac{H^4}{L^\dagger_g}}\right).
    \end{equation*}
    when  $K \gtrsim c/d^{\min}_g$, where $d^{\min}_g:=\min\{d^{\rho_l,\Gc_l}_h(s, a): (s, a, h, l)  \text{ s.t. } \exists \nu, d^{\mu^*\times \nu,\Gc}_h(s, a)>0, d^{\rho_l,\Gc_l}_h(s, a)>0\}$.
\end{theorem}
Thus, to obtain an $\varepsilon$-optimal policy, HetPEVI-Game requires $L^{\dagger}_g K = \tilde{O}(C^{\dagger}_g H^4SA_2/\varepsilon^2)$ and $L^{\dagger}_g  = \tilde{O}(H^4/\varepsilon^2)$ besides the burn-in requirement $K \gtrsim c/d^{\min}_g$. The conditions that $L^{\dagger}_g >0$ and $C^{\dagger}_g < \infty$ are further implied by the following collective unilateral coverage assumption.
\begin{assumption}[Collective unilateral coverage,  this work]\label{asp:collective_game}
    At each $(s,a,h)\in \Sc\times \Ac\times [H]$, if there exists $\nu$ that $d_h^{\mu^*\times \nu, \Gc}(s,a;\xi)>0$, then there \emph{exists} $l\in [L]$ such that the behavior policy $\rho_l$ satisfies that $d^{\rho_l, \Gc_l}_h(s,a;\xi_l)>0$.
\end{assumption}
The following individual coverage assumption is quoted, where dataset $\Dc$ is collected by behavior policy $\rho$ directly from the target MG $\Gc$ and initial state distribution $\xi'$.
\begin{assumption}[Individual unilateral coverage,  \citet{cui2022offline, cui2022provably, yan2022model}]\label{asp:individual_game}
    At each $(s,a,h)\in \Sc\times \Ac\times [H]$, if there exists $\nu$ that $d_h^{\mu^*\times \nu, \Gc}(s,a;\xi)>0$, then the behavior policy $\rho$ satisfies that $d_h^{\rho, \Gc}(s,a;\xi')>0$.
\end{assumption}
It can be observed that efficient learning in MG requires a stronger coverage assumption than that in MDP, i.e., the dataset needs to cover not only the NE pair $(\mu^*, \nu^*)$ but also the policy pair $(\mu^*, \nu)$ for any policy $\nu$ of the min-player. This requirement stated in Assumption~\ref{asp:individual_game} can be stringent as the only policy $\rho$ needs to satisfy it individually. On the other hand, Assumption~\ref{asp:collective_game} is more practical as multiple data sources can collectively provide coverage.

\section{Extension to Offline Robust RL}\label{sec:robust}
In addition to learning a single target task, it is often essential to learn a robustly good  policy in many practical applications. We are thus motivated to consider the distributionally robust RL problem \citep{zhou2021finite,shi2022distributionally}. In particular, it can be characterized by a nominal (i.e., center) MDP $\Mc^c:= \{H, \Sc, \Ac, \Pb^c, r\}$ and an associated uncertainty set $\Uc^\sigma(\Pb^c) = \otimes _{(s, a, h)\in \Sc\times \Ac\times [H]}\Uc^\sigma(\Pb^c_h(\cdot|s,a))$ for an uncertainty level $\sigma>0$, where $\otimes$ denotes the Cartesian product and $\Uc^\sigma(\Pb^c_h(\cdot|s,a)): = \{\Pb^\sigma_h(\cdot|s,a)\in \Delta(\Sc): \KL{\Pb^\sigma_h(\cdot|s,a)||\Pb^c_h(\cdot|s,a)} \leq \sigma\}$.
In other words, the uncertainty set contains the transition distributions whose KL-divergence from that of the nominal MDP at each $(s,a,h)\in \Sc\times \Ac\times [H]$ is at most $\sigma$. We further denote $\Rc:=\{\Mc^\sigma = \{H, \Sc, \Ac, \Pb^\sigma, r\}: \Pb^\sigma\in \Uc^\sigma(\Pb^c) \}$ as the collection of MDPs with transitions contained in the uncertainty set. Moreover, the robust value functions of a policy $\pi$ at $(s, a, h)\in \Sc\times \Ac\times [H]$ can be defined as $V^{\pi, \Rc}_h(s) = \inf\nolimits_{\Mc^\sigma \in \Rc} V^{\pi, \Mc^\sigma}_h(s)$ and $Q^{\pi, \Rc}_h(s, a) = \inf\nolimits_{\Mc^\sigma \in \Rc} Q^{\pi, \Mc^\sigma}_h(s, a)$,
which provide worst-case characterizations among the uncertainty set.

For this problem of robust RL, it has been established that there exists an optimal policy $\pi^*$ that is deterministic and maximizes the above value functions \citep{iyengar2005robust}. Thus, the quality of a learned policy $\hat{\pi}$ is measured by the following gap for a given initial state distribution $\xi$: $\text{RGap}(\hat{\mu}; \Rc, \xi) := V^{\pi^*, \Rc}_1(\xi) - V^{ \hat{\pi}, \Rc}_1(\xi)$.

Existing works considering the offline version of this robust RL problem typically assume a dataset sampled from the nominal MDP $\Mc^c$ \citep{shi2022distributionally,zhou2021finite}. Instead, we consider $L$ available datasets, i.e., $\{\Dc_l: l\in [L]\}$, and each $\Dc_l$ contains $K$ trajectories sampled independently by a behavior policy $\rho_l$ and an initial state distribution $\xi_l$ on the data source MDP $\Mc_l$. A stochastic relationship between the data sources $\{\Mc_l: l\in [L]\}$ and the nominal MDP $\Mc^c$ similar to Assumption~\ref{asp:relationship} is further considered, which is rigorously stated in Assumption~\ref{asp:relationship_robust}.

\subsection{Algorithm Design and Analysis}
A variant of HetPEVI is developed to find a robustly good policy with datasets from multiple perturbed data sources, termed HetPEVI-Robust. While the complete description of HetPEVI-Robust is deferred to Appendix~\ref{app_sub:alg_detail_robust}, it particularly performs the value iteration as follows:
\begin{equation*}
\small
    \begin{aligned}
        &\hat{Q}_h(s,a) = \max\left\{\hat{r}_h(s,a) + \left(\hat{\Pb}^{\inf}_h \hat{V}_{h+1}\right)(s,a) - \Gamma^\sigma_h(s,a), 0\right\},\\
        &\hat{\pi}_h(s)= \argmax\nolimits_{a\in \Ac} \hat{Q}_h(s,a), \quad  \hat{V}_h(s) = \hat{Q}_h(s,\hat{\pi}_h(s)), 
    \end{aligned}
\end{equation*}
where we define that
\begin{align*}
    &(\hat{\Pb}^{\inf}_h \hat{V}_{h+1})(s,a) \\
    &:= \inf\nolimits_{\hat{\Pb}^\sigma_h(\cdot|s,a) \sim \Uc^{\sigma}(\hat{\Pb}_h(\cdot|s,a))}\left(\hat{\Pb}^{\sigma}_h \hat{V}_{h+1}\right)(s,a)\\
    &=\sup_{\lambda \geq 0}\left\{-\lambda \log\left(\left[\hat{\Pb}_h\exp\left(-\hat{V}_{h+1}/\lambda\right)\right](s,a)\right) - \lambda \sigma\right\}.
\end{align*}
The above last equation holds due to the strong duality \citep{hu2013kullback} and can be efficiently solved \citep{panaganti2022sample,yang2021towards}. More importantly, the penalty term $\Gamma^\sigma_h(s,a)$ is specifically designed as
\begin{equation*}\small
    \begin{aligned}
        \Gamma^\sigma_h&(s,a):= \min\Bigg\{H, \\
        &\frac{c}{\sigma \hat{\Pb}^{\min}_{h}(s,a)}\sqrt{\sum_{l\in \hat{\Lc}_h(s,a)}\frac{H^2\log(SAH/\delta)}{(\hat{L}_h(s,a))^2 N_{h,l}(s,a)}}\\
        &+ \frac{c}{\sigma \hat{\Pb}^{\min}_{h}(s,a)}\sqrt{\frac{H^2\log(SAH/\delta)}{\hat{L}_h(s,a)}} + c\sqrt{\frac{\log(SAH/\delta)}{\hat{L}_h(s,a)}}\Bigg\},
    \end{aligned}
\end{equation*}
where $\hat{\Pb}^{\min}_{h}(s,a) = \min\{\hat{\Pb}_{h}(s'|s,a): s' \text{ s.t. } \hat{\Pb}_{h}(s'|s,a)>0\}$. First, it is noted that this penalty term once again has a two-part structure: the first term to aggregate sample uncertainties and the last two terms to compensate source uncertainties. Moreover, compared with $\Gamma_h(s,a)$, an additional factor $1/(\sigma \hat{\Pb}^{\min}_{h}(s,a))$ appears in the design of $\Gamma^\sigma_h(s,a)$, which is carefully crafted to maintain pessimism during the non-linear value iteration.

Following the same steps in the analyses of HetPEVI and HetPEVI-Game, the following two quantities are introduced as data coverage measurements regarding robust MDP:
\begin{equation*}\small
\begin{aligned}
    L^\dagger_\sigma:= \min\{&|\Lc_h(s,a)|: (s, a, h) \text{ s.t. } \\
    &\exists \Mc^\sigma\in \Rc, d^{\pi^*, \Mc^\sigma}_h(s,a;\xi) > 0\}, \\
        C^{\dagger}_\sigma:=\max_{(s, a, h)}&\max_{\Mc^\sigma\in \Rc}\left\{\sum_{l\in \Lc_h(s,a)}\frac{\min\left\{d^{\pi^*, \Mc^{\sigma}}_{h}(s,a;\xi), \frac{1}{S}\right\}}{|\Lc_h(s,a)|\cdot d^{\rho_l, \Mc_l}_{h}(s,a;\xi_l)}\right\},
\end{aligned}
\end{equation*}
where $\Lc_h(s,a):= \{l\in [L]: d^{\rho_l, \Mc_l}_h(s,a;\xi_l)>0\}$. With the following additional notations, 
\begin{equation*}\small
    \begin{aligned}
        L^{\min}_{\sigma}:=&\min\{|\Lc_h(s,a)|: (s,a, h) \text{ s.t. } |\Lc_h(s,a)|>0\},\\
        d^{\min}_{\sigma}: = &\min\{d^{\rho_l,\Mc_l}_h(s,a;\xi_l): (s, a,h, l) \\
        &\text{ s.t. } d^{\rho_l,\Mc_l}_h(s,a;\xi_l)>0\},\\
        \Pb^{\min}_{*}:= &\min\{\Pb^c_h(s'|s,a):  (s, a, h, s') \\
        &\text{ s.t. } d^{\pi^*, \Mc^c}_h(s,a;\xi) >0, \Pb^c_h(s'|s,a)>0\},\\
        \Pb^{\min}_{\sigma}:= &\min\{\Pb^c_h(s'|s,a): (s, a, h, s') \\
        & \text{ s.t. } \exists l\in [L], d^{\rho_l, \Mc_l}_h(s,a;\xi_l) >0, \Pb^c_h(s'|s,a)>0\},
    \end{aligned}
\end{equation*}
the following Theorem~\ref{thm:HetPEVI_robust} provides a characterization of the performance of HetPEVI-Robust.
\begin{theorem}[HetPEVI-Robust]\label{thm:HetPEVI_robust}
     Under Assumption~\ref{asp:relationship_robust}, with probability at least $1-\delta$, the output policy $\hat\mu$ of HetPEVI-Robust satisfies
    \begin{equation*}\small
    \begin{aligned}
        {\normalfont\text{RGap}}(\hat\pi; \Rc, \xi) = \tilde{O}\left(\frac{\sqrt{C^{\dagger}_{\sigma}H^4S}}{\sigma \Pb^{\min}_*\sqrt{L^\dagger_{\sigma} K }}+\frac{H + H^2/(\sigma \Pb^{\min}_{*})}{\sqrt{{L^\dagger_{\sigma}}}}\right),
    \end{aligned}
    \end{equation*}
    when  $K \gtrsim c/(d^{\min}_{\sigma}(\Pb^{\min}_{\sigma})^2)$ and $L^{\min}_{\sigma} \gtrsim c/(\Pb^{\min}_{\sigma})^2$.
\end{theorem}
It can be observed that besides the burn-in requirements on $K$ and $L^{\min}_{\sigma}$, with guarantees $L^\dagger_{\sigma} K = \tilde{O}(C^{\dagger}_{\sigma}H^4S(\sigma \Pb^{\min}_*)^{-2}/\varepsilon^2)$ and $L^\dagger_{\sigma} = \tilde{O}((H^2 + H^4(\sigma \Pb^{\min}_*)^{-2})/\varepsilon^2)$, HetPEVI-Robust can find an $\varepsilon$-optimal policy for the target robust RL.
To ensure $L^\dagger_\sigma>0$ and $C^\dagger_\sigma<\infty$, it is sufficient to have the following Assumption~\ref{asp:collective_robust} of collective robust coverage, which is also compared with the previously required Assumption~\ref{asp:individual_robust} on individual robust coverage \citep{shi2022distributionally}.
\begin{assumption}[Collective robust coverage, this work]\label{asp:collective_robust}
    At each $(s,a,h)\in \Sc\times \Ac\times [H]$, if there exists $\Mc^\sigma\in \Rc$ that $d^{\pi^* ,\Mc^\sigma}_{h}(s,a;\xi)>0$, then there \emph{exists} $l\in [L]$ such that the behavior policy $\rho_l$ satisfies that $d^{\rho_l, \Mc_l}_{h}(s,a;\xi_l)>0$.
\end{assumption}

\begin{assumption}[Individual robust coverage, \citet{shi2022distributionally}]\label{asp:individual_robust}
    At each $(s,a,h)\in \Sc\times \Ac\times [H]$, if there exists $\Mc^\sigma\in \Rc$ that $d^{\pi^* ,\Mc^\sigma}_{h}(s,a;\xi)>0$, then the behavior policy $\rho$ satisfies that $d^{\rho, \Mc^c}_{h}(s,a;\xi')>0$.
\end{assumption}
Once again, it can be observed that Assumption~\ref{asp:collective_robust} is more practical than Assumption~\ref{asp:individual_robust} as it leverages collective information from all data sources.

\shir{
\section{Discussions}\label{sec:discussion}
\textbf{Target-source Relationships.}
This work mainly targets one basic target-source relationship formulated in Assumption~\ref{asp:relationship}: the source MDPs are randomly perturbed versions of the target MDPs, where the expectation of the source generation process exactly matches the target MDP. This scenario itself captures key features of many applications as mentioned in Sections~\ref{sec:intro} and \ref{app:setting}, and the design ideas in HetPEVI can be similarly extended to other different task-source relationships. For example, instead of the stochastic relationship in Assumption~\ref{asp:relationship}, a static relationship can be considered such that the target MDP is a weighted average of source MDPs \citep{agarwal2022provable}: $r_h(s,a) = \sum_{l\in [L]}\omega_l(s,a) r_{h,l}(s,a)$ and $\Pb_h(s,a) = \sum_{l\in [L]}\omega_l(s,a) \Pb_{h,l}(s,a)$ for all $(s, a, h)\in \Sc\times \Ac\times [H]$. In this case, there is no need to consider source uncertainties and the penalty for sample uncertainties can be designed as $c\sqrt{\sum_{l\in[L]}(\omega_l(s,a))^2H^2/N_{h,l}(s,a)}$, and a corresponding performance guarantee can also be established following the procedure of Theorem~\ref{thm:HetPEVI}. 

We hope this work can be a starting point for further investigations into different target-source relationships. Especially, one interesting direction is to consider function approximation to accommodate large state/action spaces. For example, in linear MDP \cite{jin2020provably}, the random perturbations could potentially happen on the overall linear structures.

\textbf{Information Aggregation Schemes.} Another interesting direction to be further explored is how to aggregate information from different sources more effectively. As the first step to investigating this problem, we start with the simple approach of \emph{equally weighting} estimates from all sources that provide information. One limitation of this approach is that adding a data source with poor coverage would not necessarily improve the performance of HetPEVI since it is equally treated in information aggregation. It is conceivable that some other fine-tuned approaches might be more efficient in aggregating information. One promising idea is to aggregate the estimate from each source with weights according to their uncertainties, i.e., a small weight for a high-uncertainty source. This approach, intuitively, would be able to deal with sources with poor coverage more efficiently (by assigning a small weight to them) while additional investigations are needed to concretely design and analyze such algorithms.

}

\section{Conclusions}
This work studied a novel problem of offline RL with data sources that are randomly perturbed versions of the target task. An information-theoretic lower bound was derived, which reveals that guarantees on sample complexity and source diversity are simultaneously required for finding a good policy on the target task. Then, a novel HetPEVI algorithm was proposed, which adopts a two-part penalty term to jointly consider the uncertainties from the finite number of data samples and the limited amount of data sources. Theoretical analyses proved that as long as a good collective (as opposed to individual) data coverage can be provided by the data sources, HetPEVI can effectively solve the target task. Moreover, the required sample complexity and source diversity of HetPEVI is optimal up to a polynomial factor of the horizon length. Experimental results further illustrated the effectiveness of HetPEVI. At last, we extended the study to offline Markov games and offline robust RL with perturbed data sources with two generalized versions of HetPEVI proposed. These extensions further corroborated that offline RL with perturbed data sources is feasible given a good collective data coverage, while it requires sufficient source diversity besides adequate sample complexity.

\shir{\section*{Acknowledgements}
The work of CSs was supported in part by the US National Science Foundation (NSF) under awards  2143559, 2029978, 2002902, and 2132700, and the Bloomberg Data Science Ph.D. Fellowship. The work of JY was supported in part by the US NSF under awards 2003131, 1956276, 2030026, and 2133170.}


\bibliography{ref.bib}
\bibliographystyle{icml2023}

\newpage
\appendix
\onecolumn

\section{Additional Discussions}
\subsection{The Motivation and Setting of This Work}\label{app:setting}
\shir{The work is largely motivated by the need of finding generally applicable strategies in many real-world applications. Especially, concrete motivating examples include the chatbot training discussed in Section~\ref{sec:intro}, and the following ones (among many others) from the healthcare and recommendation systems:
\begin{itemize}
    \item In healthcare applications, standardized procedures for medical diagnosis and treatment are often valuable as they provide general guidelines for medical personnel to follow. The output policy in this work can serve this purpose as it is able to aggregate common information in historical medical records from different patients;
    \item In recommendation systems, when dealing with a new customer, online shopping platforms would need a generally effective mechanism for advertising and promoting. Our proposed approach can then leverage existing histories from different past customers to find the desired generic mechanism.
\end{itemize}
In other words, it would be helpful to interpret the target MDP in this work as the population-level response dynamics in these real-world applications (e.g., common reactions to certain medical treatments and promoting strategies), while the source MDPs characterize the sampled individuals. Motivated by the aforementioned practical needs, this work focuses on finding such generally applicable strategies via individually perturbed data sources.}

For a more detailed comparison with related works, we also recall the setting of this work and provide a step-by-step overview in Fig.~\ref{fig:model_full}. In particular, we note that the learning goal is the target MDP $\Mc$ but there is no direct access to it. Instead, a few data source MDPs $\{\Mc_l: l\in [L]\}$ are available, which are randomly perturbed from the target MDP $\Mc$. Datasets $\{\Dc_l: l\in [L]\}$ are collected from these data source MDPs via behavior policies $\{\rho_l: l\in [L]\}$. With these datasets, the agent finds (e.g., via the proposed HetPEVI algorithm) an output policy $\hat{\pi}$. Lastly, this learned policy $\hat{\pi}$ is intended to be deployed back to the target MDP $\Mc$, where its performance is measured.

\begin{figure}[htb]
    \centering
    \includegraphics[width = 
    \linewidth]{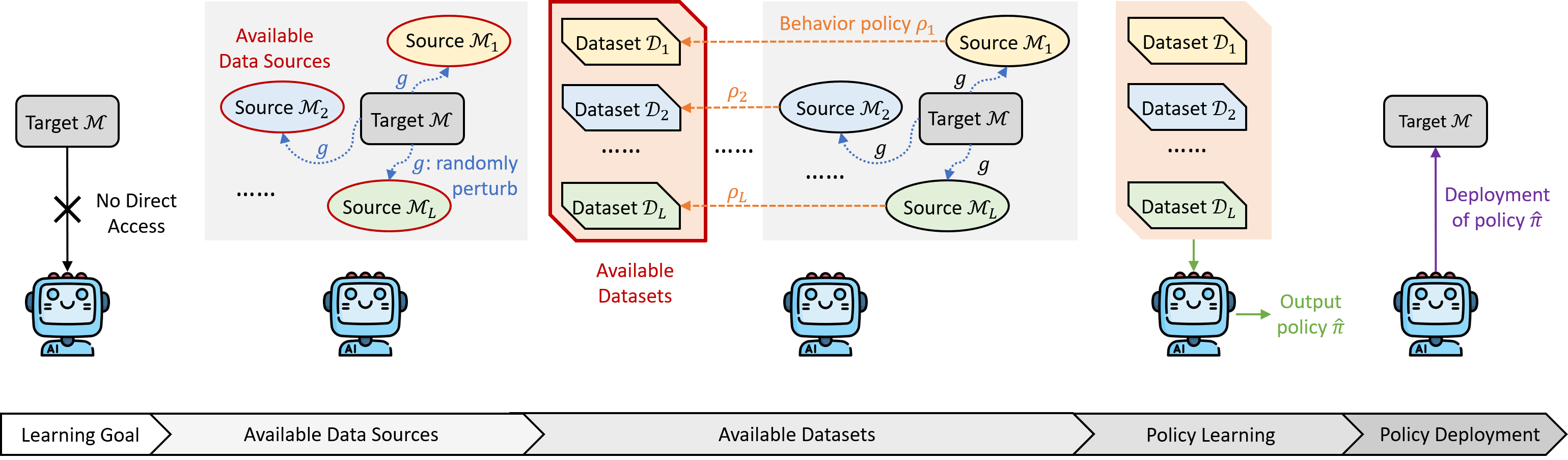}
    \caption{A step-by-step problem overview.}
    \label{fig:model_full}
\end{figure}

\shir{To further illustrate the considered setting, we specify a learning problem as follows.
\begin{itemize}
    \item \textbf{Target MDP.} Consider the target MDP to be the dynamics of picking up a bowl of a certain shape and weight.
    \item \textbf{Source MDPs.} One possible scenario is that we have four other bowls: one larger than the target bowl, one smaller, one heavier, and one lighter. Each data source is randomly given one bowl from these four and then its picking-up trajectories are collected.
    \item \textbf{Target-source Relationship.} The most basic setting to be considered is that the averaged shape and weight of the four other bowls match those of the target bowl. Then, the averaged picking-up dynamics at each location 
    and with each movement of these four bowls should conceivably match those of picking up the target bowl. Finally, since data sources are randomly selected from the four bowls, the expected source dynamics match the target dynamics, which is now stated as Assumption~\ref{asp:relationship}. With the proposed HetPEVI, one can learn how to pick the original bowl using these collected trajectories by picking other bowls.
\end{itemize}}

\subsection{Related Works}\label{app:related}
Reinforcement learning \citep{sutton2018reinforcement} has seen much progress over the past few years, particularly in its theoretical understanding; see the recent monograph \citep{agarwal2019reinforcement} for an overview. We will discuss the most related papers in the following, with a particular focus on the theoretical aspect as well as the offline setting. A compact summary of these topics and their comparisons with this work can be found in Table~\ref{tbl:related}. Graphical illustrations can be found in Fig.~\ref{fig:related}.

\textbf{Canonical offline RL.} Inspired by empirical advances \citep{yu2020mopo, kumar2020conservative}, the principle of ``pessimism'' is incorporated and proved efficient for offline RL \citep{jin2021pessimism,rashidinejad2021bridging,xie2021policy,li2022settling, shi2022pessimistic,yin2022near,xiong2022nearly,xie2021bellman,uehara2021pessimistic,zanette2021provable}, which is also adopted in the design of HetPEVI and its generalizations. However, these works focus on the classical setting of learning with a dataset sampled exactly from the target task, which is rather restricted for practical applications. 

\begin{figure}[tb]
    \centering
    \subfigure[Canonical offline RL]{\includegraphics[height = 1.6in]{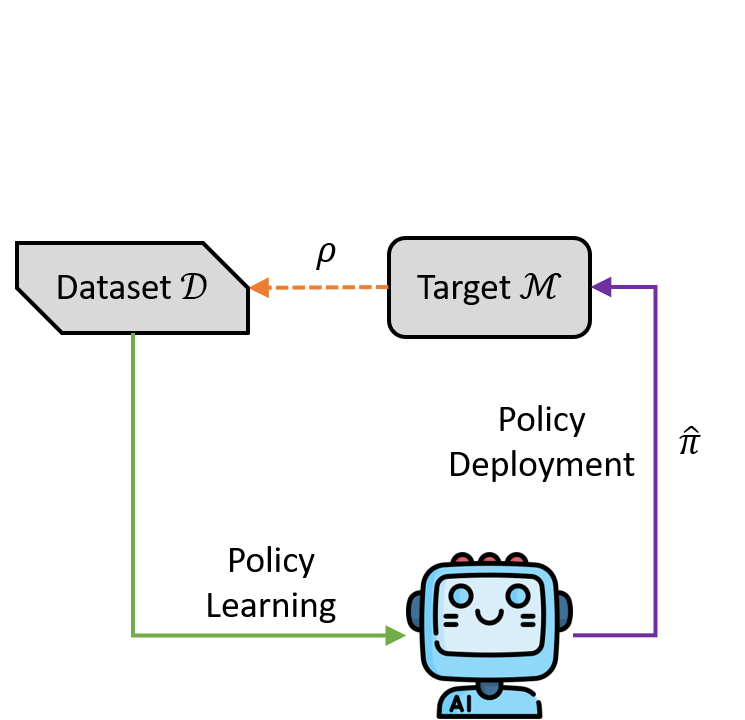}}
    \subfigure[Offline robust RL]{\includegraphics[height = 1.6in]{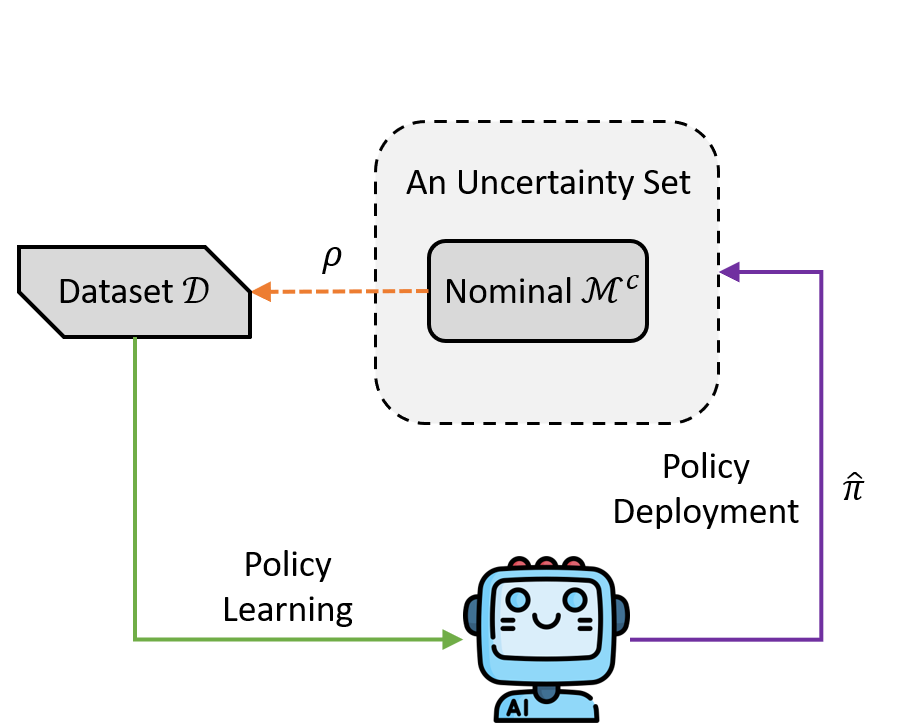}}
    \subfigure[Offline latent RL]{\includegraphics[height = 1.6in]{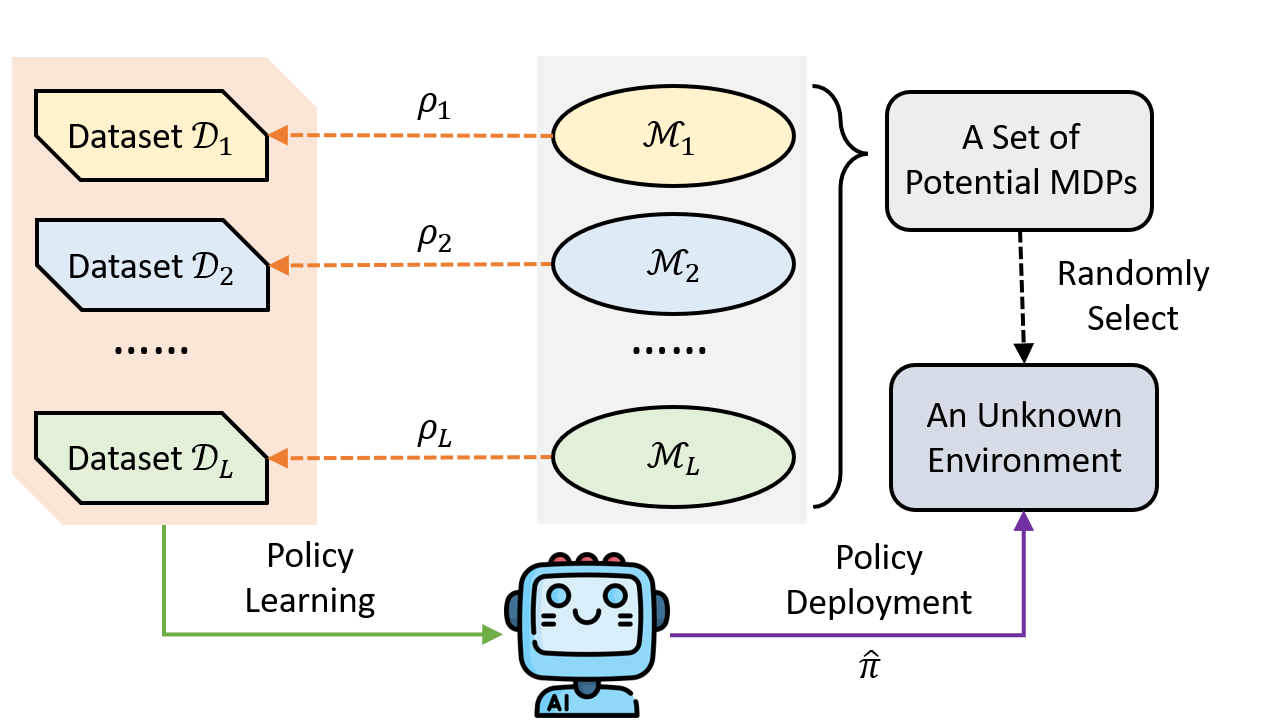}}
    \subfigure[Offline federated/multi-task RL]{\includegraphics[height = 1.6in]{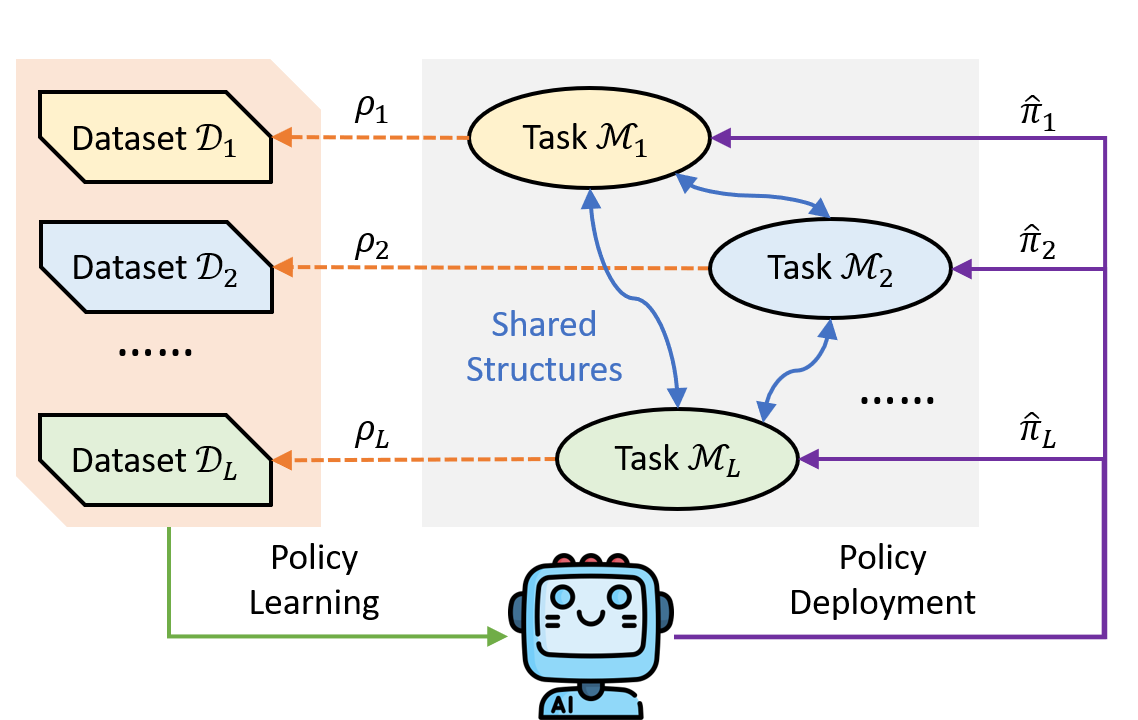}}
    \hspace{0.4in}
    \subfigure[Offline perturbed-source RL, this work]{\includegraphics[height = 1.6in]{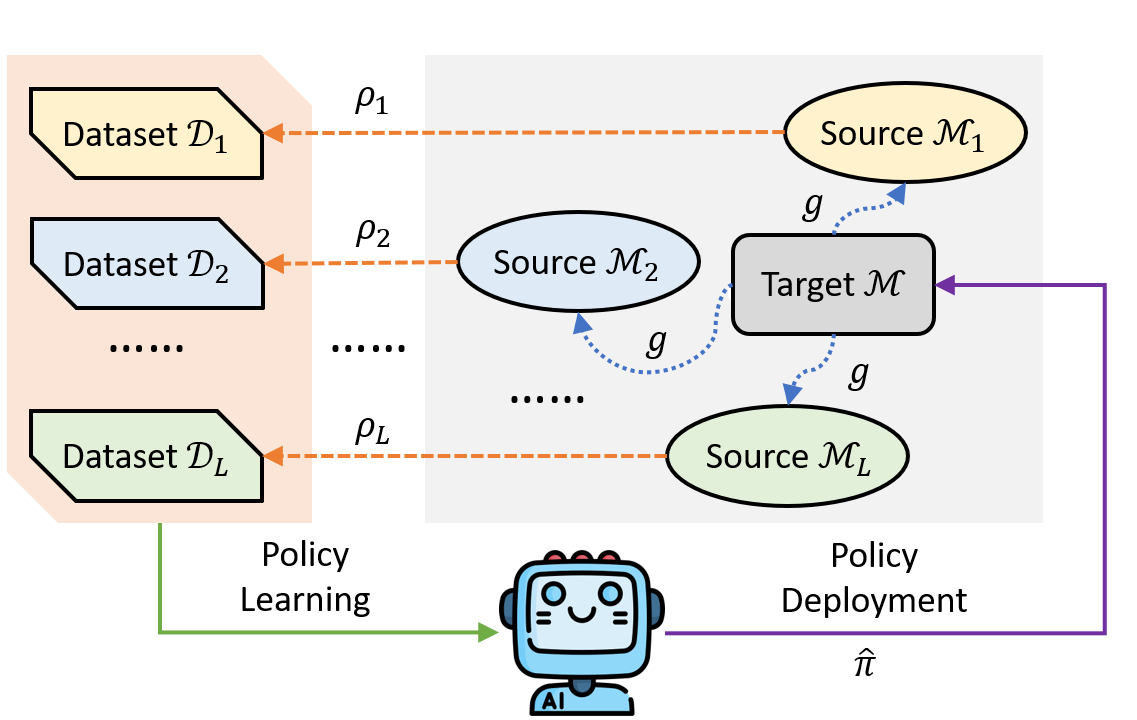}}
    \caption{Graphical illustrations of related research topics, which differ from this work in their studied data sources and evaluation criteria.}
    \label{fig:related}
\end{figure}

\textbf{Robust RL.} Recently, a series of work \citep{zhou2021finite,yang2021towards,panaganti2022robust,panaganti2022sample,shi2022distributionally,ma2022distributionally} has made theoretical advances on the topic of offline robust RL. In particular, a dataset from one data source, i.e., a nominal MDP, is collected, which is used by the learning agent to find an output policy. Then, the output policy is deployed to an uncertainty set around the nominal MDP, and its worst-case performance is adopted as the evaluation criteria. However, this work mainly considers multiple data sources while the learned policy is intended for deployment on the target task. Furthermore, in Section~\ref{sec:robust}, we generalize the study of offline robust RL to consider that the available data are not from multiple perturbed versions of the nominal MDP instead of itself.

\textbf{Latent RL.} Another related topic is latent RL \citep{kwon2021rl, zhou2022horizon}. These existing studies are mainly in the online setting. Following the same spirit, the corresponding offline version would require datasets from a set of potential MDPs. Then, the learning agent aims to find a good policy that performs well on average in an unknown environment randomly selected from the aforementioned set of potential MDPs (which is often modeled to be related to a latent variable). Thus, although both latent RL and this work need to deal with multiple data sources, this work considers data sources that are perturbed versions of a target MDP while latent RL poses no relationships among data sources. Moreover, this work evaluates the learned policy on the target MDP while latent RL targets at performing well on the potential MDPs on average.

\textbf{Federated and multi-task RL.} Federated RL has attracted much attention recently \citep{dubey2021provably,jin2022federated}, and \citet{zhou2022federated} studies its offline version. In particular, \citet{zhou2022federated}  considers datasets from different MDPs at different sites, which share certain representations. The design is to leverage the shared structure to accelerate learning of each individual site MDP, i.e., find a good policy for each site MDP. Similarly, multi-task RL attempts to leverage common structures of multiple tasks to facilitate learning each individual task \citep{zhang2021provably,lu2021power,yang2020impact,yang2022nearly}. However, this work considers a stochastic relationship between data source MDPs (Assumption~\ref{asp:relationship}) instead of explicitly shared structures, while aiming to find one good policy for the target MDP (but not for the data source MDPs).

\textbf{Meta-RL.} The most related literature of this work falls in the research domain of ``offline meta-RL'' \citep{mitchell2021offline,dorfman2021offline,lin2022model,li2020focal}, which however lacks rigorous theoretical understanding currently. Especially, the target MDP can be viewed as a learning target for the ``meta-training'' process of offline meta-RL \citep{mitchell2021offline}, which aims to extract information from the available data of multiple sources. In addition to ``meta-training'', the empirically studied offline meta-RL systems often feature another step of ``meta-testing'', which further utilizes the learned information and applies them to a specific task. Thus, we believe this work may contribute to the theoretical understanding of offline meta-RL systems, especially the meta-training process, which may also serve as the foundation for studies of the meta-testing process.

\textbf{Other related works.} Another conceptually related work is \citet{shrestha2020deepaveragers}. In particular, it looks for similar state-action pairs with small distances in the dataset, which can be thought of as available data sources in this work. Then, the Lipschitz continuity assumption is posed, which serves a similar role as Assumption~\ref{asp:relationship} to establish the connection between desired task information with the available datasets. From this perspective, the first term in Theorem 3.1 \citep{shrestha2020deepaveragers} can be interpreted as coming from the source uncertainty while the second term is from the sample uncertainty. However, we also note that the Lipschitz continuity assumption is a worst-case consideration that would not characterize the concentration of involving more data sources, which however is the key of this work.

\shir{Moreover, \citet{jeong2022calibrated} provided a set of results to jointly characterize the sample and source uncertainties. One distinction is that it focuses on leveraging multiple estimators on one randomly perturbed data source, while this work targets aggregating information from multiple heterogeneously perturbed data sources. Despite the difference, the methods proposed in \citet{jeong2022calibrated} may still be of value in the future study of offline RL with randomly perturbed data sources, especially its utilization of between-dataset information in quantifying uncertainties.}

\subsection{Future Works}\label{app:future}
Some discussions on future works are included in Section~\ref{sec:discussion}. A few other potential directions are discussed as follows.

\textbf{Coverage Assumptions.} While the collective coverage requirements of Assumptions~\ref{asp:collective_tabular}, \ref{asp:collective_game} and \ref{asp:collective_robust} are relatively weak, it is still of major interest to further explore how to perform offline RL (especially with heterogeneous data sources) under weaker conditions. This direction is particularly interesting with multiple data sources since the heterogeneity naturally enriches the data diversity.

\textbf{Unknown Source Identities.} This work considers the scenario where each data sample is known to belong to a particular source. One interesting direction is to investigate the scenarios without such information, i.e., unknown source identities. A potential solution is to first cluster the data samples and then adopt the algorithms proposed in this work. However, it is challenging to design clustering algorithms with provable performance guarantees. One candidate clustering technique is developed in \cite{kwon2021rl} for the study of latent MDP, which however relies on strong assumptions of prior knowledge about the source MDPs.

\textbf{Personalization.} As mentioned in the discussions of related work, this work can be viewed as targeting at the ``meta-training'' process of offline meta-RL \citep{mitchell2021offline}, which extracts common knowledge from available data of multiple sources. While the extracted common knowledge has individual values, in many applications, an additional step of personalization is performed to further use such knowledge to benefit a specific task, which is called the ``meta-testing'' process of offline meta-RL \citep{mitchell2021offline}. Based on this work, it would be valuable to further study how to perform such a personalization step with theoretical guarantees.

\section{Proof of Theorem~\ref{thm:lower_bound}}
In this section, we provide the proof of Theorem~\ref{thm:lower_bound}, which is inspired by \citet{li2022settling} and \citet{shi2022distributionally} but more complicated as the datasets are collected from data sources instead of the target task itself.
\subsection{Construction of hard problem instances}
Let us first introduce two MDPs to be used in the following proofs:
\begin{align*}
    \left\{\Nc^{\chi} = \left(H, \Sc, \Ac, \Qb^{\chi} = \{\Qb^{\chi}_h:h\in [H] \} , r = \{r_h:h\in [H]\}\right): \chi \in \{0, 1\}\right\},
\end{align*}
where the state space is $\Sc = \{0, 1, \cdots, S-1\}$, and the action space is $\Ac = \{0, 1, 2\}$. 

The transition kernel $\Qb^{0}$ is defined as
\begin{align*}
    \Qb^{0}_1(s'|s,a) &= 
    \begin{cases}
        p'\oneb\{s' = 0\} + (1-p') \oneb\{s' = 1\} & \text{if $(s,a) = (0, 0)$}\\
        q'\oneb\{s' = 0\} + (1-q') \oneb\{s' = 1\} & \text{if $(s,a) = (0, 1)$}\\
        q\oneb\{s' = 0\} + (1-q) \oneb\{s' = 1\} & \text{if $(s,a) = (0, 2)$}\\
        \oneb\{s' = s\} & \text{if $s \geq 1$}
    \end{cases}\\
    \Qb^{0}_h(s'|s,a) &= \oneb\{s' = s\}, \qquad \forall (h, s, ,a)\in \{2, \cdots, H\}\times \Sc\times \Ac.
\end{align*}

The transition kernel $\Qb^{1}$ is defined as
\begin{align*}
    \Qb^{1}_1(s'|s,a) &= 
    \begin{cases}
        q'\oneb\{s' = 0\} + (1-q') \oneb\{s' = 1\} & \text{if $(s,a) = (0, 0)$}\\
        p'\oneb\{s' = 0\} + (1-p')\oneb\{s' = 1\} & \text{if $(s,a) = (0, 1)$}\\
        q\oneb\{s' = 0\} + (1-q) \oneb\{s' = 1\} & \text{if $(s,a) = (0, 2)$}\\
        \oneb\{s' = s\} & \text{if $s \geq 1$}
    \end{cases}\\
    \Qb^{1}_h(s'|s,a) &= \oneb\{s' = s\}, \qquad \forall (h, s, ,a)\in \{2, \cdots, H\}\times \Sc\times \Ac.
\end{align*}

The parameters $p', p, q$ and $q'$ are set to be
\begin{align*}
    p' = \frac{3}{4} - \frac{1}{H} + \Delta; \qquad p = p' -\alpha \Delta; \qquad q'= p' - \Delta = \frac{3}{4} - \frac{1}{H}; \qquad q= q' + \alpha\Delta
\end{align*}
for some $H$, $\Delta$ and $\alpha$ obeying 
\begin{align*}
    \frac{1}{H}< \frac{1}{4}; \qquad \Delta \leq \frac{1}{8}; \qquad \alpha\leq \frac{1}{2}.
\end{align*}
Thus,
\begin{align*}
    \frac{7}{8} > p' > p > q > q' \geq \frac{1}{2}.
\end{align*}

Moreover, for any $(h, s, a)\in [H]\times \Sc\times \Ac$, the reward function is defined as
\begin{align*}
    r_h(s,a) =
    \begin{cases}
        1 & \text{if $s= 0$}\\
        0 & \text{otherwise}.
    \end{cases}
\end{align*}
\subsubsection{Construction of a collection of hard target MDPs} Let us introduce another two MDPs as target MDPs:
\begin{align*}
    \left\{\Mc^{\phi} = \left(H, \Sc, \Ac, \Pb^{\phi} = \{\Pb^{\phi}_h:h\in [H], \}  , r = \{r_h:h\in [H]\}\right): \phi \in \{0, 1\}\right\},
\end{align*}
where the state space is $\Sc = \{0, 1, \cdots, S-1\}$, and the action space is $\Ac = \{0, 1, 2\}$. 

The transition kernel $\Pb^{0}$ is defined as
\begin{align*}
    \Pb^{0}_1(s'|s,a) &= 
    \begin{cases}
        p\oneb\{s' = 0\} + (1 - p) \oneb\{s' = 1\} & \text{if $(s,a) = (0, 0)$}\\
        q\oneb\{s' = 0\} + (1 - q) \oneb\{s' = 1\} & \text{if $(s,a) = (0, 1)$}\\
        q\oneb\{s' = 0\} + (1-q) \oneb\{s' = 1\} & \text{if $(s,a) = (0, 2)$}\\
        \oneb\{s' = s\} & \text{if $s \geq 1$}
    \end{cases}\\
    \Pb^{0}_h(s'|s,a) &= \oneb\{s' = s\}, \qquad \forall (h, s, ,a)\in \{2, \cdots, H\}\times \Sc\times \Ac.
\end{align*}

The transition kernel $\Pb^{1}$ is defined as
\begin{align*}
    \Pb^{1}_1(s'|s,a) &= 
    \begin{cases}
        q\oneb\{s' = 0\} + (1 - q) \oneb\{s' = 1\} & \text{if $(s,a) = (0, 0)$}\\
        p\oneb\{s' = 0\} + (1 - p) \oneb\{s' = 1\} & \text{if $(s,a) = (0, 1)$}\\
        q\oneb\{s' = 0\} + (1-q) \oneb\{s' = 1\} & \text{if $(s,a) = (0, 2)$}\\
        \oneb\{s' = s\} & \text{if $s \geq 1$}
    \end{cases}\\
    \Pb^{1}_h(s'|s,a) &= \oneb\{s' = s\}, \qquad \forall (h, s, ,a)\in \{2, \cdots, H\}\times \Sc\times \Ac.
\end{align*}

It can be observed that
\begin{align*}
    \Mc^{0} = (1-\alpha)\cdot \Nc^{0} + \alpha\cdot \Nc^{1}, \qquad \Mc^{1} = \alpha\cdot \Nc^{0} + (1-\alpha)\cdot \Nc^{1},
\end{align*}
where the weighted average is w.r.t. rewards and transition kernels.

\subsubsection{Construction of a source MDP generation distribution}
If the target MDP is $\Mc^0$, then with probability $1-\alpha$, the generated source MDP is $\Nc^0$, and with probability $\alpha$, the generated source MDP is $\Nc^1$.  If the target MDP is $\Mc^0$, then with probability $\alpha$, the generated source MDP is $\Nc^0$, and with probability $1 - \alpha$, the generated source MDP is $\Nc^1$.

\subsubsection{Construction of the offline dataset}
In the environment $\Nc^{\chi}$, a batch dataset is generated consisting of $K$ independent sample trajectories each of length $H$ based on an initial distribution 
\begin{align*}
    \xi^d(s) = \mu(s),
\end{align*}
where
\begin{align*}
    \mu(s) = \frac{1}{CS}\oneb\{s = 0\} + \left(1 - \frac{1}{CS}\right)\oneb\{s = 1\}, \qquad \text{with } \frac{1}{CS} \leq \frac{1}{4},
\end{align*}
and a behavior policy, which is specified in the following.

\textbf{Good behavior policy.} The good behavior policy $\rho^g$ uniformly selects actions $\{0,1\}$ as follows:
\begin{align*}
    \rho^g_h(a|s) = \frac{1}{2}, \qquad \forall (s, a, h)\in \Sc \times \{0, 1\} \times [H].
\end{align*}

It turns out that for any MDP $\Nc^{\chi}$, the occupancy distributions of the above batch dataset admit the following characterization:
\begin{align*}
    d^{\rho^g, \Nc^{\chi}}_1(0, a;\xi^d) = \frac{1}{2}\mu(0), &\qquad \forall a\in \{0,1\};\qquad \frac{\mu(s)}{4} \leq  d^{\rho^g, \Nc^\chi}_h(s, a;\xi^d) \leq \mu(s), \qquad \forall (s, a, h) \in \Sc\times \{0,1\} \times [H].
\end{align*}
In particular, for any $\Nc^\chi$ with $\chi\in \{0,1\}$ and the initial distribution as $\xi^d(s) = \mu(s)$, we have that
\begin{align*}
    d^{\rho^g,\Nc^\chi}_1(s;\xi^d) = \mu(s), \qquad \forall s\in \Sc,
\end{align*}
which leads to
\begin{align*}
    d^{\rho^g,\Nc^\chi}_1(0, 0;\xi^d) = \mu(0)\rho^g_1(0|0) = \frac{\mu(0)}{2}; \qquad d^{\rho^g,\Nc^\chi}_1(0, 1;\xi^d) = \mu(0)\rho^g_1(1|0) = \frac{\mu(0)}{2}.
\end{align*}
The state occupancy distribution at step $h = 2$ obeys that
\begin{align*}
    d^{\rho^g, \Nc^\chi}_2(0;\xi^d) = \mu(0) \left[\rho^g_1(\chi|0)p' + \rho^g_1(1-\chi|0)q' + \rho^g_1(2|0)q\right] = \frac{\mu(0)(p'+q')}{2},
\end{align*}
and
\begin{align*}
    d^{\rho^g, \Nc^\chi}_2(1;\xi^d) = \mu(1) + \mu(0) \left[\rho^g_1(\chi|0)(1-p') + \rho^g_1(1-\chi|0)(1-q') + \rho^g_1(2|0)(1-q)\right] = \mu(1) + \frac{\mu(0)(2 - p' -q')}{2}.
\end{align*}
The above results can be further bounded as
\begin{align*}
    \frac{\mu(0)}{2} \leq  d^{\rho^g, \Nc^\chi}_2(0;\xi^d) \leq \mu(0), \qquad  \mu(1) \leq d^{\rho^g, \Nc^\chi}_2(1;\xi^d)\leq 2\mu(1),
\end{align*}
which leads to that
\begin{align*}
    &\frac{\mu(0)}{4} \leq  d^{\rho^g, \Nc^\chi}_h(0, a;\xi^d) \leq \frac{\mu(0)}{2}, \qquad \forall (a, h) \in \{0,1\} \times [2, H];\\
    &\frac{\mu(1)}{2} \leq  d^{\rho^g, \Nc^\chi}_h(1, a;\xi^d) \leq \mu(1), \qquad \forall (a, h) \in \{0,1\} \times [2, H].
\end{align*}

\noindent\textbf{Bad behavior policy.} The bad behavior policy $\rho^b$ always selects action $2$ as follows:
\begin{align*}
    \rho^b_h(2|s) = 1, \qquad \forall (s, a, h)\in \Sc \times \{2\} \times [H].
\end{align*}

Correspondingly, for any MDP $\Nc^{\chi}$, it is easy to observe that the occupancy distributions of the above batch dataset follow that
\begin{align*}
    d^{\rho^b, \Nc^{\chi}}_h(s, a;\xi^d) = 0, \qquad \forall (s, a, h)\in \Sc\times \{0, 1\}\times [H].
\end{align*}

We then specify the initial distributions of all data sources as $\xi^d$, the behavior policies of the first $L^\ddagger$ data sources as the good ones, i.e., $\rho^g$, and the behavior policies of the other $L - L^\ddagger$ as the bad ones, i.e., $\rho^b$. 

\subsubsection{Value functions and optimal policies}
We choose the initial state distribution to be tested on as
\begin{align*}
    \xi(s) = \begin{cases}
        1, & \text{if $s = 0$}\\
        0, &\text{otherwise}.
    \end{cases}
\end{align*}
Then, the following lemma can be established.
\begin{lemma}
    For any $\phi \in \{0, 1\}$ and any policy $\pi$, it holds that
    \begin{align*}
        V^{\pi, \Mc^\phi}_1(0) = 1 + \pi_1(\phi|0) p (H - 1) + \pi_1(1-\phi|0) q (H-1) + \pi_1(2|0) q(H-1).
    \end{align*}
    In addition, there exists an optimal policy $\pi^{*,\Mc^\phi}$ such that its optimal value functions obey
    \begin{align*}
        &V^{\pi^*,\Mc^\phi}_1(0) = 1 + p (H - 1),\\
        \forall h \in [2, H]: \qquad & V^{\pi^*,\Mc^\phi}_h(0) = H - h +1,\\
        \forall h \in [H]: \qquad & \pi^{*,\Mc^\phi}_h(\phi|0) = 1, \qquad  \pi^{*,\Mc^\phi}_h(\phi|1) = 1, \qquad V^{\pi^*, \Mc^{\phi}}_h(1) = 0,
    \end{align*}
    where we denote $V^{\pi^{*},\Mc^\phi}_h(s):=V^{\pi^{*,\Mc^{\phi}},\Mc^\phi}_h(s)$ for simplicity.
    
    Furthermore, it holds that
    \begin{align*}
        L^\dagger = L^\ddagger; \qquad C^\dagger \in [C, 4C].
    \end{align*}
\end{lemma}
\begin{proof}
    For any policy $\phi$, it can be easily observed that
    \begin{align*}
        V^{\pi, \Mc^\phi}_h(0) = H - h + 1, &\qquad \forall h \in [2, H];\\
        V^{\pi, \Mc^\phi}_h(s) = 0, &\qquad \forall (s, h) \in \{1, \cdots, S-1\} \times [H],
    \end{align*}
    which immediately indicates that
    \begin{align*}
        &V^{\pi^*,\Mc^\phi}_1(0) = 1 + p (H - 1),\\
        \forall h \in [2, H]: \qquad & V^{\pi^*,\Mc^\phi}_h(0) = H - h +1,\\
        \forall h \in [H]: \qquad & \pi^{*,\Mc^\phi}_h(\phi|0) = 1, \qquad  \pi^{*,\Mc^\phi}_h(\phi|1) = 1, \qquad V^{\pi^*, \Mc^{\phi}}_h(1) = 0.
    \end{align*}
    For a policy $\pi$, it further holds that
    \begin{align*}
        V^{\pi, \Mc^\phi}_1(0) &= \Eb_{a\sim \pi_1(\cdot|0)}\left[r_h(0,a) + \left(\Pb_1 V^{\pi, \Mc^\phi}_1\right)(0, a)\right] \\
        &= 1 + \pi_1(\phi|0)\left(p V^{\pi, \Mc^\phi}_2(0) + (1-p)V^{\pi, \Mc^\phi}_2(1)\right) \\
        & + \pi_1(1-\phi|0)\left(q V^{\pi, \Mc^\phi}_2(0) + (1-q)V^{\pi, \Mc^\phi}_2(1)\right) \\
        & + \phi_1(2|0)\left(q V^{\pi, \Mc^\phi}_2(0) + (1-q)V^{\pi, \Mc^\phi}_2(1)\right)\\
        & = 1 + \pi_1(\phi|0) p (H - 1) + \pi_1(1-\phi|0) q (H-1) + \pi_1(2|0)q(H-1).
    \end{align*}
    Moreover, it holds that
    \begin{align*}
        d^{\pi^{*,\Mc^\phi},\Mc^\phi}_h(0, \phi; \xi) &= d^{\pi^{*,\Mc^\phi},\Mc^\phi}_h(0;\xi) = \bP\left\{s_h = 0| s_{h-1}\sim d^{\pi^{*,\Mc^\phi},\Mc^\phi}_{h-1}(\cdot|\xi), a_h\sim \pi^{*,\Mc^\phi}_h(\cdot|s_h)\right\}\\
        &= d^{\pi^{*,\Mc^\phi},\Mc^\phi}_{h-1}(0;\xi) = \cdots = d^{\pi^{*,\Mc^\phi},\Mc^\phi}_2(0;\xi) \geq p \xi(0) \geq \frac{1}{2}.
    \end{align*}
    Thus,
    \begin{align*}
        &\frac{\min\left\{d^{\pi^*, \Mc^{\phi}}_{h}(0,\phi; \xi), \frac{1}{S}\right\}}{d^{\rho^g, \Nc^{\chi}}_{h}(0,\phi; \xi^d)} = \frac{1/S}{d^{\rho^g, \Nc^{\chi}}_{h}(0,\phi; \xi^d)} \in \left[ \frac{1/S}{\mu(0)},  \frac{1/S}{\mu(0)/4}\right]= [C, 4C];\\
        &\frac{\min\left\{d^{\pi^*, \Mc^{\phi}}_{h}(1,\phi; \xi), \frac{1}{S}\right\}}{d^{\rho^g, \Nc^{\chi}}_{h}(1,\phi; \xi^d)} \leq \frac{1/S}{\mu(1)/2} = \frac{2}{S(1-\frac{1}{CS})} \leq \frac{8}{3S} \leq \frac{2C}{3},
    \end{align*}
    which concludes the proof.
\end{proof}

\subsection{Establishing the minimax lower bound}
\subsubsection{Converting the goal to estimate $\phi$}
We choose $\alpha$ and $\Delta$ such that
\begin{align*}
    (H-1)(1-2\alpha)\Delta \geq 2\varepsilon.
\end{align*}
Then, with the selected $\xi$, it holds that
\begin{align*}
    &V^{\pi^*, \Mc^\phi}_1(\xi) - V^{\hat{\pi}, \Mc^\phi}_1(\xi) \\
    &= V^{\pi^*, \Mc^\phi}_1(0) - V^{\hat{\pi}, \Mc^\phi}_1(0)\\
    &= 1 + p (H - 1) - \left[1 + \hat{\pi}_1(\phi|0) p (H - 1) +\hat{\pi}_1(1-\phi|0) q (H-1) + \hat{\pi}_1(2|0)q(H-1)\right]\\
    & = p (H - 1) - \hat{\pi}_1(\phi|0) p (H - 1) - \hat{\pi}_1(1-\phi|0) q (H-1) - \hat{\pi}_1(2|0)q(H-1)\\
    & = (H-1)(p- q) (1- \hat{\pi}_1(\phi|0))\\
    & = (H-1)(1-2\alpha)\Delta (1- \hat{\pi}_1(\phi|0))\\
    & \geq 2\varepsilon (1- \hat{\pi}_1(\phi|0)).
\end{align*}
Suppose that for any $\phi \in \{0,1\}$,
\begin{align*}
    \bP_{\phi}\left\{V^{\pi^*, \Mc^\phi}_1(\xi) - V^{\hat{\pi}, \Mc^\phi}_1(\xi)\leq \varepsilon\right\} \geq \frac{7}{8},
\end{align*}
we necessarily have that
\begin{align*}
    \bP_{\phi}\left\{\hat{\phi} = \phi\right\} =\bP_{\phi}\left\{\hat{\pi}_1(\phi|0) \geq \frac{1}{2}\right\} \geq \frac{7}{8},
\end{align*}
where
\begin{align*}
    \hat{\phi} = \begin{cases}
        \phi, & \text{if $\hat{\pi}_1(\phi|0)\geq \frac{1}{2}$}\\
        1- \phi & \text{if $\hat{\pi}_1(\phi|0)< \frac{1}{2}$}.
    \end{cases}
\end{align*}

\subsubsection{Probability of error in testing two hypotheses}
We focus on differentiating the two hypotheses $\phi\in \{0, 1\}$. Towards this, consider the minimax probability of error defined as follows:
\begin{align*}
    p_e: = \inf_{\psi}\max\left\{\bP_0(\psi\neq 0), \bP_1(\psi\neq 1)\right\},
\end{align*}
where the infimum is taken over all possible tests $\psi$ constructed from the batch dataset. Then, following the standard results (Theorem 2.2, \citet{tsybakov2009introduction}), it holds that
\begin{align*}
    p_e \geq \frac{1}{4} \exp\left(-\KL{\upsilon^{\Mc^0}||\upsilon^{\Mc^1}}\right),
\end{align*}
where $\upsilon^{\Mc^\phi}(\cdot)$ is the distribution of the sampled dataset with the target MDP as $\Mc^\phi$. Furthermore, it holds that
\begin{align*}
    \KL{\upsilon^{\Mc^0}||\upsilon^{\Mc^1}} &= \sum_{\Dc}\upsilon^{\Mc^0}(\Dc)\log\left(\frac{\upsilon^{\Mc^0}(\Dc)}{\upsilon^{\Mc^1}(\Dc)}\right)\\
    &\overset{(i)}{=} \sum_{l\in [L]}\sum_{\Dc_l}\upsilon^{\Mc^0}_{l}(\Dc_l)\log\left(\frac{\upsilon^{\Mc^0}_l(\Dc_l)}{\upsilon^{\Mc^1}_l(\Dc_l)}\right) \\
    & \overset{(ii)}{=} \sum_{l\in [L^\dagger]}\sum_{\Dc_l}\upsilon^{\Mc^0}_{l}(\Dc_l)\log\left(\frac{\upsilon^{\Mc^0}_l(\Dc_l)}{\upsilon^{\Mc^1}_l(\Dc_l)}\right)\\
    & \overset{(iii)}{=} \sum_{l\in [L^\dagger]}\sum_{\Dc_l}\left(\upsilon^{\Nc^0}_{l}(\Dc_l)(1-\alpha) + \upsilon^{\Nc^1}_{l}(\Dc_l)\alpha\right)\log\left(\frac{\upsilon^{\Nc^0}_{l}(\Dc_l)(1-\alpha) + \upsilon^{\Nc^1}_{l}(\Dc_l)\alpha}{\upsilon^{\Nc^0}_{l}(\Dc_l)\alpha + \upsilon^{\Nc^1}_{l}(\Dc_l)(1-\alpha)}\right)\\
    & \overset{(iv)}{\leq} \min\left\{\text{term (I)}, \text{term (II)}\right\}.
\end{align*}
where equation (i) is from the definition of $\upsilon_l^{\Mc^\phi}(\cdot)$ as the distribution of the sampled dataset at source $l$ with the target MDP as $\Mc^\phi$ and the property of KL-divergence for product measures. Equation (ii) is from the fact that the distribution of the last $L-L^{\ddagger}$ datasets are the same. Equation(iii) leverages the definition of  $\upsilon^{\Nc^\psi}_l(\cdot)$ as the distribution of the sampled dataset at source $l$ with the source MDP as $\Nc^\psi$ and the two-step dataset generation procedure (i.e., first randomly perturb target MDP as source MDP, and then randomly sample data from the source MDP). Inequality (iv) is from the log-sum inequality and the following definition:
\begin{align*}
    \text{term (I)} &:= \sum_{l\in [L^\dagger]}\sum_{\Dc_l}\upsilon^{\Nc^0}_{l}(\Dc_l)(1-\alpha) \log\left(\frac{\upsilon^{\Nc^0}_{l}(\Dc_l)(1-\alpha)}{\upsilon^{\Nc^0}_{l}(\Dc_l)\alpha}\right) + \sum_{l\in [L^\dagger]}\sum_{\Dc_l}\upsilon^{\Nc^1}_{l}(\Dc_l)\alpha\log\left(\frac{\upsilon^{\Nc^1}_{l}(\Dc_l)\alpha}{\upsilon^{\Nc^1}_{l}(\Dc_l)(1-\alpha)}\right)\\
    & = \sum_{l\in [L^\dagger]}(1-\alpha) \log\left(\frac{(1-\alpha)}{\alpha}\right) + \sum_{l\in [L^\dagger]}\alpha\log\left(\frac{\alpha}{(1-\alpha)}\right)\\
    & = L^\dagger (1-2\alpha)\log\left(\frac{(1-\alpha)}{\alpha}\right);\\
    \text{term (II)} &:= \sum_{l\in [L^\dagger]}\sum_{\Dc_l}\upsilon^{\Nc^0}_{l}(\Dc_l)(1-\alpha) \log\left(\frac{\upsilon^{\Nc^0}_{l}(\Dc_l)(1-\alpha) }{\upsilon^{\Nc^1}_{l}(\Dc_l)(1-\alpha)}\right) + \sum_{l\in [L^\dagger]}\sum_{\Dc_l}\upsilon^{\Nc^1}_{l}(\Dc_l)\alpha\log\left(\frac{ \upsilon^{\Nc^1}_{l}(\Dc_l)\alpha}{\upsilon^{\Nc^0}_{l}(\Dc_l)\alpha }\right)\\
    &= \sum_{l\in [L^\dagger]}\sum_{\Dc_l}\upsilon^{\Nc^0}_{l}(\Dc_l)(1-\alpha) \log\left(\frac{\upsilon^{\Nc^0}_{l}(\Dc_l)}{\upsilon^{\Nc^1}_{l}(\Dc_l)}\right) + \sum_{l\in [L^\dagger]}\sum_{\Dc_l}\upsilon^{\Nc^1}_{l}(\Dc_l)\alpha\log\left(\frac{ \upsilon^{\Nc^1}_{l}(\Dc_l)}{\upsilon^{\Nc^0}_{l}(\Dc_l)}\right)\\
    & = \alpha L^\dagger K \cdot \KL{\upsilon^{\Nc^0}_{1}(\tau)||\upsilon^{\Nc^1}_{1}(\tau)} + (1-\alpha) L^\dagger K \cdot \KL{\upsilon^{\Nc^1}_{1}(\tau)||\upsilon^{\Nc^0}_{1}(\tau)},
\end{align*}
where notation $\tau$ refers to a complete $H$-step trajectory.

Furthermore, it holds that
\begin{align*}
\KL{\upsilon^{\Nc^0}_{1}(\tau)||\upsilon^{\Nc^1}_{1}(\tau)} &= \frac{1}{2}\mu(0)\sum_{a\in \{0,1\}}\KL{\Qb^0(\cdot|0,a)||\Qb^1(\cdot|0,a)} \leq \mu(0)\frac{(p'-q')^2}{p'(1-p')} = \mu(0)\frac{\Delta^2}{p'(1-p')}\\
    \KL{\upsilon^{\Nc^1}_{1}(\tau)||\upsilon^{\Nc^0}_{1}(\tau)} &= \frac{1}{2}\mu(0)\sum_{a\in \{0,1\}}\KL{\Qb^1(\cdot|0,a)||\Qb^0(\cdot|0,a)}\leq \mu(0)\frac{(p'-q')^2}{p'(1-p')} = \mu(0)\frac{\Delta^2}{p'(1-p')},
\end{align*}
where the inequality is from the fact that $\Nc^0$ and $\Nc^1$ only differ at state-action pairs $(0,0)$ and $(0,1)$, and the inequality is from a basic property of KL divergence (see Lemma 10 in \cite{li2022settling}).

With $\varepsilon < \frac{H}{64}$, if it is designed that
\begin{align*}
    \alpha = \frac{1}{2} - \frac{16\varepsilon}{H} \in (\frac{1}{4}, \frac{1}{2}), \qquad \Delta = \frac{1}{8}
\end{align*}
it holds that
\begin{align*}
    (H-1)(1-2\alpha)\Delta = (H-1) \cdot \frac{32\varepsilon}{H}  \cdot  \frac{1}{8}\geq 2\varepsilon,
\end{align*}
and
\begin{align*}
    (1-2\alpha)\log\left(\frac{1-\alpha}{\alpha}\right) = (1-2\alpha)\log\left(1+ \frac{1-2\alpha}{\alpha}\right)\leq \frac{(1-2\alpha)^2}{\alpha}\leq \frac{4096\varepsilon^2}{H^2}.
\end{align*}
Thus, if
\begin{align*}
    L^\dagger \leq \frac{H^2\log(2)}{4096\varepsilon^2},
\end{align*}
it holds that
\begin{align*}
    p_e \geq \frac{1}{4} \exp\left(-\KL{\upsilon^{0}||\upsilon^{1}}\right) \geq \frac{1}{4} \exp\left(-L^\dagger \cdot \KL{1-\alpha||\alpha}\right) \geq \frac{1}{4} \exp\left(-\frac{H^2\log(2)}{4096\varepsilon^2} \cdot \frac{4096\varepsilon^2}{H^2}\right) = \frac{1}{8}.
\end{align*}

Similarly with $\varepsilon\leq \frac{H}{64}$, if it is designed that
\begin{align*}
    \alpha = \frac{1}{4}, \qquad \Delta = \frac{8\varepsilon}{H} \leq \frac{1}{8}
\end{align*}
it holds that
\begin{align*}
    (H-1)(1-2\alpha)\Delta = (H-1) \cdot  \frac{1}{2} \cdot \frac{8\varepsilon}{H}\geq 2\varepsilon,
\end{align*}
and
\begin{align*}
    \frac{\Delta^2}{p'(1-p')}\leq 16\Delta^2 = \frac{1024\varepsilon^2}{H^2}.
\end{align*}
Thus, if
\begin{align*}
    L^\dagger K \leq \frac{H^2\log(2)}{1024\varepsilon^2\mu(0)},
\end{align*}
it holds that
\begin{align*}
    p_e &\geq \frac{1}{4} \exp\left(-\KL{\upsilon^{0}||\upsilon^{1}}\right) \\
    &\geq \frac{1}{4} \exp\left(-\alpha L^\dagger K \cdot \KL{\upsilon^{\Nc^0}_{1}(\tau)||\upsilon^{\Nc^1}_{1}(\tau)} + (1-\alpha) L^\dagger K \cdot \KL{\upsilon^{\Nc^1}_{1}(\tau)||\upsilon^{\Nc^0}_{1}(\tau)}\right) \\
    &\geq \frac{1}{4} \exp\left(-\frac{H^2\log(2)}{1024\varepsilon^2\mu(0)} \cdot 
 \mu(0)\frac{1024\varepsilon^2}{H^2}\right) = \frac{1}{8}.
\end{align*}

\subsubsection{Putting things together}
Finally, suppose that there exists an estimator $\hat{\pi}$ such that
\begin{align*}
    \bP_{0}\left\{V^{\pi^*, \Mc^0}_1(\xi) - V^{\hat{\pi}, \Mc^0}_1(\xi)\leq \varepsilon\right\} \geq \frac{7}{8} \qquad \text{and} \qquad \bP_{1}\left\{V^{\pi^*, \Mc^1}_1(\xi) - V^{\hat{\pi}, \Mc^1}_1(\xi)\leq \varepsilon\right\} \geq \frac{7}{8}.
\end{align*}
The estimator $\hat{\phi}$ must satisfy that
\begin{align*}
    \bP_{0}\left\{\hat{\phi} \neq 0\right\} \leq \frac{1}{8} \qquad \text{and} \qquad \bP_{1}\left\{\hat{\phi} \neq 1\right\} \leq \frac{1}{8},
\end{align*}
which cannot happen if
\begin{align*}
    L^\dagger \leq \frac{H^2\log(2)}{4096\varepsilon^2}
\end{align*}
or
\begin{align*}
    L^\dagger K \leq \frac{H^2C^\dagger S\log(2)}{4096\varepsilon^2}\leq \frac{H^2CS\log(2)}{1024\varepsilon^2} = \frac{H^2\log(2)}{1024\varepsilon^2\mu(0)}
\end{align*}
under the correspondingly designed scenarios.

\section{Details of HetPEVI and Proof of Theorem~\ref{thm:HetPEVI}}
\subsection{The Subsampling Procedure}\label{app:subsampling}
The detailed subsampling procedure can be found in \citet{li2022settling}. Here we state their obtained main result as follows.
\begin{lemma}\label{lem:subsampling}
    With probability at least  $1-8\delta$, the output dataset from the two-fold subsampling scheme in \citet{li2022settling} is distributionally equivalent to independently sampled from the data source MDP and 
    \begin{align*}
        N_{h,l}(s,a) \geq \frac{Kd^{\rho_l,\Mc_l}_h(s,a)}{8} - 5 \sqrt{K d^{\rho_l,\Mc_l}_h(s,a) \log\left(\frac{KHL}{\delta}\right)}.
    \end{align*}
    for all $(h, s, a, l)\in [H]\times \Sc\times \Ac\times [L]$.
\end{lemma}

\subsection{Core Lemmas}
\begin{lemma}\label{lem:concentration}
    For all $(s, a, h)\in \Sc\times \Ac\times [H]$, and any function $V: \Sc  \to [0, H]$ independent of $\hat{\Pb}_h$, with probability at least $1-4\delta$, it holds that
    \begin{align*}
        \left|\left(\hat{\Bb}_h V\right)(s,a) - \left(\Bb_h V\right)(s,a)\right| \leq \Gamma_h(s,a) := c \sqrt{\sum_{l\in \hat{\Lc}_h(s,a)}\frac{H^2\log(SAH/\delta)}{\left(\hat{L}_h(s,a)\right)^2 N_{h,l}(s,a)}}+ c\sqrt{\frac{H^2\log(SAH/\delta)}{\hat{L}_h(s,a)}}.
    \end{align*}
\end{lemma}
\begin{proof}
    For a fixed $(s,a,h)$, it holds that
    \begin{align*}
    &\left|\left(\hat{\Bb}_h V\right)(s,a) - \left(\Bb_h V\right)(s,a)\right|\\
    &= \left|\hat{r}_h(s,a) + \left(\hat{\Pb}_h V\right)(s,a) - r_{h}(s,a)  - \left(\Pb_h V\right)(s,a)\right| \\
    & = \left|\sum_{l\in \hat{\Lc}_h(s,a)}\frac{1}{\hat{L}_h(s,a)}\left(r_{h,l}(s,a) + \left(\hat{\Pb}_{h,l}V\right)(s,a)\right) - \left(r_h(s,a) + \Pb_h V(s,a)\right)\right|\\
    &\leq \left|\sum_{l\in \hat{\Lc}_h(s,a)}\frac{1}{\hat{L}_h(s,a)}\left(r_{h,l}(s,a) + \left(\hat{\Pb}_{h,l}V\right)(s,a)\right) - \sum_{l\in \hat{\Lc}_h(s,a)}\frac{1}{\hat{L}_h(s,a)}\left(r_{h,l}(s,a) + \left(\Pb_{h,l}V\right)(s,a)\right)\right|\\
    &+ \left|\sum_{l\in \hat{\Lc}_h(s,a)}\frac{1}{\hat{L}_h(s,a)}\left(r_{h,l}(s,a) + \left(\Pb_{h,l}V\right)(s,a)\right) - \left(r_h(s,a) + \Pb_h V(s,a)\right)\right|\\
    &\leq \sqrt{\sum_{l\in \hat{\Lc}_h(s,a)}\frac{2H^2\log(SAH/\delta)}{\left(\hat{L}_h(s,a)\right)^2 N_{h,l}(s,a)}}+ \sqrt{\frac{2H^2\log(SAH/\delta)}{\hat{L}_h(s,a)}}
    \end{align*}
    where the last step holds with probability at least $1-4\delta/(SAH)$ due to Hoeffding inequality. The lemma can then be established via a union bound over $(s, a, h) \in \Sc\times \Ac \times [H]$.
\end{proof}

\subsection{Main Proofs}
In the following, we establish Theorem~\ref{thm:HetPEVI}. The proof framework is inspired by \citet{li2022settling} but is uniquely adapted to handle randomly perturbed data sources.

\noindent\textbf{Step 1: establishing the pessimism property.} Armed with Lemma~\ref{lem:concentration}, with probability at least $1-\delta$, the following relation holds
\begin{align}\label{eqn:pess_value}
    \hat{Q}_h(s,a) \leq Q^{\hat{\pi},\Mc}_h(s,a), \qquad \hat{V}_h(s,a) \leq V^{\hat{\pi},\Mc}_h(s,a), \qquad \forall (s, a, h)\in \Sc\times \Ac\times [H].
\end{align}
Towards this, it is first observed that
\begin{align*}
    \hat{Q}_{H+1}(s,a) = Q^{\hat{\pi},\Mc}_{H+1}(s,a) = 0, \qquad \forall (s, a)\in \Sc\times \Ac.
\end{align*}
Then, suppose that $\hat{Q}_{h+1}(s,a) \leq  Q^{\hat{\pi},\Mc}_{h+1}(s,a)$ for all $(s,a)\in \Sc\times \Ac$ at some step $h\in [H]$, we can observe that by the update rule in HetPEVI, it holds that
\begin{align*}
    0 \leq \hat{V}_{h+1}(s) = \max_{a\in \Ac} \hat{Q}_{h+1}(s,a) \leq \max_{a\in \Ac} Q^{\hat{\pi}, \Mc}_{h+1}(s,a) = V^{\hat{\pi},\Mc}_{h+1}(s)\leq H, \qquad \forall s\in \Sc,
\end{align*}
If $\hat Q_h(s,a) = 0$, the claim naturally holds. If not, we can obtain that
\begin{align*}
     \hat Q_h(s,a) &\leq \left(\hat{\Bb}_h \hat{V}_{h+1}\right)(s,a)- \Gamma_h(s,a)\\
     & \leq \left(\Bb_h \hat{V}_{h+1}\right)(s,a) + \left|\left(\hat{\Bb}_h \hat{V}_{h+1}\right)(s,a) - \left(\Bb_h \hat{V}_{h+1}\right)(s,a)\right| - \Gamma_h(s,a)\\
     & \overset{(i)}{\leq} \left(\Bb_h \hat{V}_{h+1}\right)(s,a)\overset{(ii)}{\leq} \left(\Bb_h V^{\hat{\pi},\Mc}_{h+1}\right)(s,a)= Q^{\hat{\pi},\Mc}_h(s,a).
\end{align*}
The above inequality (i) is from Lemma~\ref{lem:concentration} and leverages the fact that $\hat{V}_{h+1}(\cdot)$ is independent of $\hat{\Pb}_h$ and takes value in $[0, H]$. Inequality (ii) is from the obtained fact that $\hat{V}_{h+1}(s)\leq V^{\hat{\pi},\Mc}_{h+1}(s)$. The desired claim Eqn.~\eqref{eqn:pess_value} can be verified by induction.

\noindent\textbf{Step 2: bounding the performance difference.}
From Eqn.~\eqref{eqn:pess_value}, we can observe that
\begin{align*}
    0 \leq V^{\pi^*, \Mc}_h(s) - V^{\hat{\pi}, \Mc}_h(s) \leq V^{\pi^*, \Mc}_h(s) - \hat{V}_h(s) \leq Q^{\pi^*,\Mc}_h(s, \pi^*_h(s)) - \hat{Q}_h(s, \pi^*_h(s)).
\end{align*}
With
\begin{align*}
    Q^{\pi^*,\Mc}_h(s, \pi^*_h(s)) &= \left(\Bb_hV^{*,\Mc}_{h+1}\right)(s, \pi^*_h(s))\\
    \hat{Q}_h(s, \pi^*_h(s)) & = \max\left\{\left(\hat{\Bb}_h\hat{V}_{h+1}\right)(s, \pi^*_h(s)) - \Gamma_h(s, \pi^*_h(s)), 0\right\},
\end{align*}
we can further obtain that
\begin{align*}
    V^{\pi^*, \Mc}_h(s) - \hat{V}_h(s) &\leq \left(\Bb_hV^{*,\Mc}_{h+1}\right)(s, \pi^*_h(s)) - \left(\hat{\Bb}_h\hat{V}_{h+1}\right)(s, \pi^*_h(s)) + \Gamma_h(s, \pi^*_h(s))\\
    & = \left(\Bb_hV^{*,\Mc}_{h+1}\right)(s, \pi^*_h(s)) - \left(\Bb_h\hat{V}_{h+1}\right)(s, \pi^*_h(s)) \\
    & + \left(\Bb_h\hat{V}_{h+1}\right)(s, \pi^*_h(s)) - \left(\hat{\Bb}_h\hat{V}_{h+1}\right)(s, \pi^*_h(s)) + \Gamma_h(s, \pi^*_h(s))\\
    & \overset{(i)}{\leq} \left(\Bb_hV^{*,\Mc}_{h+1}\right)(s, \pi^*_h(s)) - \left(\Bb_h\hat{V}_{h+1}\right)(s, \pi^*_h(s))  + 2\Gamma_h(s, \pi^*_h(s))\\
    & = \left(\Pb_hV^{*,\Mc}_{h+1}\right)(s, \pi^*_h(s)) - \left(\Pb_h\hat{V}_{h+1}\right)(s, \pi^*_h(s))  + 2\Gamma_h(s, \pi^*_h(s))
\end{align*}
where inequality (i) holds with probability at least $1-\delta$ according to Lemma~\ref{lem:concentration}. If applying the above argument iteratively, we can further obtain that
\begin{align*}
    V^{\pi^*, \Mc}_1(\xi) - \hat{V}_1(\xi) \leq 2\sum_{h \in  [H]}\sum_{s\in \Sc}d^{\pi^*, \Mc}_{h}(s) \Gamma_h(s, \pi^*_h(s)).
\end{align*}

\noindent\textbf{Step 3: completing the proof with concentrability.} Let us consider $(s, h)\in \Sc \times [H]$ such that $d^{\pi^*,\Mc}_h(s)>0$. We can then obtain that for all $l\in \Lc_h(s,\pi^*_h(s))$, 
\begin{align*}
    N_{h,l}(s,\pi^*_h(s)) &\overset{(i)}{\geq} \frac{Kd^{\rho_l,\Mc_l}_h(s,\pi^*_h(s))}{8} - 5 \sqrt{K d^{\rho_l,\Mc_l}_h(s,\pi^*_h(s)) \log\left(\frac{KHL}{\delta}\right)} \overset{(ii)}{\geq} \frac{Kd^{\rho_l,\Mc_l}_h(s,\pi^*_h(s))}{16}  \overset{(iii)}{\geq} 1.
\end{align*}
where inequality (i) is from Lemma~\ref{lem:subsampling}; and inequalities (ii) and (iii) are from the condition that
\begin{align*}
    K \geq \frac{c\log\left(\frac{KH}{\delta}\right)}{d^{\min}} \geq \frac{c\log\left(\frac{KH}{\delta}\right)}{d^{\rho_l,\Mc_l}_h(s,\pi^*_h(s))}.
\end{align*}
Thus, it holds that
\begin{align*}
    \hat{L}_h(s, \pi^*_h(s)) = \sum_{l\in [L]}\oneb\{N_{h,l}(s, \pi^*_h(s))\geq 1\} = L_h(s, \pi^*_h(s)).
\end{align*}
As a result, it holds that
\begin{align*}
    \Gamma_h(s, \pi^*_h(s))  &\leq c\sqrt{\sum_{l\in \hat{\Lc}_h(s,\pi^*_h(s))}\frac{H^2\log(SAH/\delta)}{\left(\hat{L}_h(s,\pi^*_h(s))\right)^2 N_{h,l}(s,\pi^*_h(s))}}+ c\sqrt{\frac{H^2\log(SAH/\delta)}{\hat{L}_h(s,\pi^*_h(s))}}\\
    &\leq  c\sqrt{\sum_{l\in \Lc_h(s,\pi^*_h(s))}\frac{H^2\log(SAH/\delta)}{\left(L_h(s,\pi^*_h(s))\right)^2Kd^{\rho_l,\Mc_l}_h(s,\pi^*_h(s))}} + c\sqrt{\frac{H^2\log(SAH/\delta)}{L_h(s,\pi^*_h(s))}}.
\end{align*}

We can then obtain that
\begin{align*}
    &\sum_{h \in  [H]}\sum_{s\in \Sc}d^{\pi^*, \Mc}_{h}(s) \Gamma_h(s, \pi^*_h(s))\\
    &\leq  c\sum_{h \in  [H]}\sum_{s\in \Sc}d^{\pi^*, \Mc}_{h}(s) \sqrt{\frac{C^{\dagger}H^2\log(SAH/\delta)}{L^\dagger K\min\left\{d^{\pi^*,\Mc}_h(s), \frac{1}{S}\right\}}}+ c\sum_{h \in  [H]}\sum_{s\in \Sc}d^{\pi^*, \Mc}_{h}(s)\sqrt{\frac{H^2\log(SAH/\delta)}{L^\dagger}}\\
    &\leq  c\sum_{h \in  [H]}\sqrt{\sum_{s\in \Sc}d^{\pi^*, \Mc}_{h}(s)\frac{C^{\dagger}H^2\log(SAH/\delta)}{L^{\dagger}  K \min\left\{d^{\pi^*,\Mc}_h(s), \frac{1}{S}\right\}}}\sqrt{\sum_{s\in \Sc}d^{\pi^*, \Mc}_{h}(s)}+ c\sqrt{\frac{H^4\log(SAH/\delta)}{L^\dagger}}\\
    & \leq cH \sqrt{\frac{C^{\dagger}SH^2\log(SAH/\delta)}{L^\dagger K }} + c\sqrt{\frac{H^4\log(SAH/\delta)}{L^\dagger}}.
\end{align*}

Putting these results together, it can then be established that
\begin{align*}
    V^{\pi^*, \Mc}_1(\xi) - V^{\hat{\pi}, \Mc}_1(\xi) &\leq V^{\pi^*, \Mc}_1(\xi) - \hat{V}_1(\xi)\leq 2\sum_{h \in  [H]}\sum_{s\in \Sc}d^{\pi^*, \Mc}_{h}(s) \Gamma_h(s, \pi^*_h(s))\\
    &\leq cH\sqrt{\frac{C^{\dagger}SH^2\log(SAH/\delta)}{L^{\dagger} K }} + cH\sqrt{\frac{2H^2\log(SAH/\delta)}{L^{\dagger}}},
\end{align*}
which concludes the proof.

\section{Markov Game}
\subsection{Problem Formulation}
The following task--source relationship is considered for MG. It shares the same content as Assumption~\ref{asp:relationship} while we note that the overall action here consists of two individual actions from the max-player and min-player, i.e., $a = (a^1, a^2)$.
\begin{assumption}[Task--source Relationship]\label{asp:relationship_game}
    Data source MGs $\{\Gc_l = \{H, \Sc, \Ac, \Pb_{l}, r_l\}: l\in [L]\}$ are generated from an unknown set of distributions $g = \{g_h: h\in [H]\}$ such that for each $(l, h)\in [L]\times [H]$, the reward and transition $\{r_{h,l}, \Pb_{h,l}\}$ are independently sampled from the distribution $g_h(\cdot)$ whose expectation is $\{r_{h}, \Pb_{h}\}$ of the target MDP $\Gc:=\{H, \Sc, \Ac, \Pb, r\}$.
\end{assumption}

\subsection{Algorithm Details of HetPEVI-Game}\label{app_sub:alg_detail_game}
The complete description of HetPEVI-Game can be found in Algorithm~\ref{alg:HetPEVI_Game}.

 \begin{algorithm}[tbh]
	\caption{HetPEVI-Game}
	\label{alg:HetPEVI_Game}
	\begin{algorithmic}[1]
	    \STATE {\bfseries Input:} Dataset $\Dc = \{D_l:l\in[L]\}$
     \STATE Obtain $\Dc'_l \gets \text{subsampling}(\Dc_l), \forall l\in [L]$
     \FOR{$\forall (s, a, h, s') \in \Sc\times \Ac\times [H] \times \Sc$}
     \STATE $\forall l\in [L], N_{h, l}(s,a) \gets$ number of visitations on $(s, a, h)$ in $\Dc'_l$
     \STATE $\forall l\in [L], N_{h, l}(s,a, s') \gets$ number of visitations on $(s, a, h, s')$ in $\Dc'_l$ 
     \STATE $\hat{\Lc}_h(s,a)\gets \{l\in [L]: N_{h,l}(s,a)>0\}$
     \STATE $\forall l\in \hat{\Lc}_h(s,a), \hat{r}_{h,l}(s,a) \gets r_{h,l}(s,a), \hat{\Pb}_{h,l}(s'|s,a) \gets N_{h,l}(s,a,s')/N_{h,l}(s,a)$
     \STATE $\hat{r}_h(s,a)  \gets \sum_{l\in \hat{\Lc}_h(s,a)}\hat{r}_{h,l}(s,a)/(\hat{L}_h(s,a)\vee 1), \hat{\Pb}_h(s'|s,a) \gets \sum_{l\in \hat{\Lc}_h(s,a)}\hat{\Pb}_{h,l}(s'|s,a)/(\hat{L}_h(s,a)\vee 1)$
     \ENDFOR
     \STATE Initialize $\hat{V}_{H+1}(s) \gets 0, \forall s\in \Sc$ 
	    \FOR{$h = H, H-1,\cdots, 1$} 
        \FOR{$(s, a) \in \Sc \times \Ac$}
        \STATE $\Gamma^g_h(s,a) \gets \min\left\{c\sqrt{\sum\nolimits_{l\in \Lc_h(s,a)}\frac{H^2\log(SAH/\delta)}{(\hat{L}_h(s,a))^2N_{h,l}(s,a)}} + c\sqrt{\frac{H^2\log(SAH/\delta)}{{\hat{L}_h(s,a)}}}, H\right\}$
        \STATE $\hat Q_h(s,a) \gets \max\left\{(\hat{\Bb}_h \hat{V}_{h+1})(s,a)- \Gamma^g_h(s,a)), 0\right\}$
        \ENDFOR
        \FOR{$s\in \Sc$}
        \STATE $(\hat{\mu}_h(\cdot|s), \hat{\nu}_h(\cdot|s)) \gets \text{NE}(\hat{Q}_h(s,\cdot))$
        \STATE $\hat{V}_h(s) \gets \Eb_{a\sim \hat{\mu}_h(\cdot|s)\times \hat{\nu}_h(\cdot|s)}[\hat{Q}_h(s,a)]$
	    \ENDFOR
        \ENDFOR
	    \STATE {\bfseries Output:} policy $\hat\pi = \{\hat\pi_h(s): (s,h)\in \Sc\times [H]\}$
	\end{algorithmic}
	\end{algorithm}

\subsection{Core Lemmas}
Following the same steps in the proof of Lemma~\ref{lem:concentration}, the following lemma can be established.
\begin{lemma}\label{lem:concentration_game}
    For all $(s, a, h)\in \Sc\times \Ac\times [H]$, and any function $V: \Sc  \to [0, H]$ independent of $\hat{\Pb}_h$, with probability at least $1-\delta$, it holds that
    \begin{align*}
        \left|\left(\hat{\Bb}_h V\right)(s,a) - \left(\Bb_h V\right)(s,a)\right| \leq \Gamma^g_h(s,a) : = c\sqrt{\sum_{l\in \Lc_h(s,a)}\frac{H^2\log(SAH/\delta)}{\left(\hat{L}_h(s,a)\right)^2 N_{h,l}(s,a)}}+ c\sqrt{\frac{H^2\log(SAH/\delta)}{\hat{L}_h(s,a)}}.
    \end{align*}
\end{lemma}
We especially note that the action $a$ in the above lemma and the following proofs stand for an action pair $(a^1, a^2)$.

\subsection{Main Proofs}
\begin{proof}[Proof of Theorem~\ref{thm:HetPEVI_game}] In the following, we establish Theorem~\ref{thm:HetPEVI_game}. The proof framework is inspired by \citet{yan2022model} but is uniquely adapted to handle randomly perturbed data sources.

\noindent\textbf{Step 1: establishing the pessimism property.} Armed with Lemma~\ref{lem:concentration_game}, with probability at least $1-\delta$, the following relation holds
\begin{align}\label{eqn:pess_value_game}
    \hat{Q}_h(s,a) \leq Q^{\hat{\mu}\times \ber{\hat{\mu}},\Gc}_h(s,a), \qquad \hat{V}_h(s,a) \leq V^{\hat{\mu}\times \ber{\hat{\mu}},\Gc}_h(s,a), \qquad \forall (s, a, h)\in \Sc\times \Ac\times [H].
\end{align}
Towards this, it is first observed that
\begin{align*}
    \hat{Q}_{H+1}(s,a) = Q^{\hat{\mu}\times \ber{\hat{\mu}},\Gc}_{H+1}(s,a) = 0, \qquad \forall (s, a)\in \Sc\times \Ac.
\end{align*}
Then, suppose that $\hat{Q}_{h+1}(s,a) \leq  Q^{\hat{\mu}\times \ber{\hat{\mu}},\Gc}_{h+1}(s,a)$ for all $(s,a)\in \Sc\times \Ac$ at some step $h\in [H]$, we can observe that
\begin{align*}
    0 \leq \hat{V}_{h+1}(s) &= \hat{Q}_{h+1}(s,\hat{\mu}\times \hat{\nu}) \leq \hat{Q}_{h+1}(s,\hat{\mu}\times \ber{\hat{\mu}}) \leq Q^{\hat{\mu}\times \ber{\hat{\mu}},\Gc}_{h+1}(s,\hat{\mu}\times \ber{\hat{\mu}}) = V^{\hat{\mu}\times \ber{\hat{\mu}},\Gc}_{h+1}(s)\leq H, \qquad \forall s\in \Sc.
\end{align*}
If $\hat Q_h(s,a) = 0$, the claim naturally holds. If not, we can obtain that
\begin{align*}
     \hat Q_h(s,a) &= \left(\hat{\Bb}_h \hat{V}_{h+1}\right)(s,a)- \Gamma_h(s,a)\\
     & \leq \left(\Bb_h \hat{V}_{h+1}\right)(s,a) + \left|\left(\hat{\Bb}_h \hat{V}_{h+1}\right)(s,a) - \left(\Bb_h \hat{V}_{h+1}\right)(s,a)\right| - \Gamma^g_h(s,a)\\
     & \overset{(i)}{\leq} \left(\Bb_h \hat{V}_{h+1}\right)(s,a)\overset{(ii)}{\leq} \left(\Bb_h V^{\hat{\mu}\times \ber{\hat{\mu}},\Gc}_{h+1}\right)(s,a)= Q^{\hat{\mu}\times \ber{\hat{\mu}},\Gc}_{h+1}(s,a).
\end{align*}
The above inequality (i) is from Lemma~\ref{lem:concentration_game} and leverages the fact that $\hat{V}_{h+1}(\cdot)$ is independent of $\hat{\Pb}_h$ and takes value in $[0, H]$. Inequality (ii) is from the obtained fact that $\hat{V}_{h+1}(s)\leq V^{\hat{\mu}\times \ber{\hat{\mu}},\Gc}_{h+1}(s)$. The desired claim Eqn.~\eqref{eqn:pess_value_game} can be verified by induction.
\end{proof}

\noindent\textbf{Step 2: bounding the performance difference.}
From Eqn.~\eqref{eqn:pess_value_game}, we can observe that
\begin{align*}
    0 &\leq V^{\mu^*\times\nu^*, \Gc}_h(s) - V^{\hat{\mu}\times \ber{\hat{\mu}}, \Gc}_h(s) \leq V^{\mu^*\times\bar{\nu}, \Gc}_h(s)- \hat{V}_h(s) \\
    &\leq  Q^{\mu^*\times\bar{\nu}, \Gc}_h(s, \mu^*\times\bar{\nu}) - \hat{Q}_h(s, \mu^*\times \hat{\nu}) \leq Q^{\mu^*\times\bar{\nu}, \Gc}_h(s, \mu^*\times\bar{\nu}) - \hat{Q}_h(s, \mu^*\times \bar{\nu}),
\end{align*}
where
\begin{align*}
    \bar{\nu} = \{\bar{\nu}_h: h\in [H]\}, \qquad \bar{\nu}_h(s) = \argmin_{a^2\in \Ac^2} \Eb_{a^1\sim \mu^*_h(\cdot|s)}\hat{Q}_h(s, (a^1, a^2)).
\end{align*}
We particularly note that $\bar{\nu}$ is a deterministic policy.

With
\begin{align*}
    Q^{\mu^*\times\bar{\nu}, \Gc}_h(s, \mu^*\times\bar{\nu}) &= \Eb_{a\sim \mu^*_h(\cdot|s)\times \bar{\nu}_h(s)}\left[\left(\Bb_hV^{\mu^*\times \bar{\nu},\Gc}_{h+1}\right)(s, a)\right]\\
    \hat{Q}_h(s, \mu^*\times\bar{\nu}) & = \Eb_{a\sim \mu^*_h(\cdot|s)\times \bar{\nu}_h(s)}\left[\max\left\{\left(\hat{\Bb}_h\hat{V}_{h+1}\right)(s,a) - \Gamma^g_h(s, a), 0\right\}\right],
\end{align*}
we can further obtain that
\begin{align*}
    V^{\mu^*\times\bar{\nu}, \Gc}_h(s)- \hat{V}_h(s) &\leq \Eb_{a\sim \mu^*_h(\cdot|s)\times \bar{\nu}_h(s)}\left[\left(\Bb_hV^{\mu^*\times \bar{\nu},\Gc}_{h+1}\right)(s, a) - \left(\hat{\Bb}_h\hat{V}_{h+1}\right)(s,a) + \Gamma^g_h(s, a)\right]\\
    & = \Eb_{a\sim \mu^*_h(\cdot|s)\times \bar{\nu}_h(s)}\left[\left(\Bb_hV^{\mu^*\times \bar{\nu},\Gc}_{h+1}\right)(s, a) - \left(\Bb_h\hat{V}_{h+1}\right)(s,a)\right]\\
    & + \Eb_{a\sim \mu^*_h(\cdot|s)\times \bar{\nu}_h(s)}\left[\left(\Bb_h \hat{V}_{h+1}\right)(s, a) - \left(\hat{\Bb}_h\hat{V}_{h+1}\right)(s,a) + \Gamma^g_h(s, a)\right]\\
    & \overset{(i)}{\leq} \Eb_{a\sim \mu^*_h(\cdot|s)\times \bar{\nu}_h(s)}\left[\left(\Bb_hV^{\mu^*\times \bar{\nu},\Gc}_{h+1}\right)(s, a) - \left(\Bb_h\hat{V}_{h+1}\right)(s,a) + 2\Gamma^g_h(s,a)\right]\\
    & = \Eb_{a\sim \mu^*_h(\cdot|s)\times \bar{\nu}_h(s)}\left[\left(\Pb_hV^{\mu^*\times\bar{\nu},\Gc}_{h+1}\right)(s, a) - \left(\Pb_h\hat{V}_{h+1}\right)(s,a) + 2\Gamma^g_h(s,a)\right]
\end{align*}
where inequality (i) holds with probability at least $1-\delta$ according to Lemma~\ref{lem:concentration_game}. If applying the above argument iteratively, we can further obtain that
\begin{align*}
    \sum_{s\in \Sc} d^{\mu^*\times\bar{\nu}, \Gc}_h(s)\left(V^{\mu^*\times\bar{\nu}, \Gc}_h(s)- \hat{V}_h(s)\right) \leq 2\sum_{h' = h}^H\sum_{s\in \Sc}d^{\mu^*\times \bar{\nu}, \Gc}_{h'}(s) \Gamma_{h'}^g(s, \mu^*\times \bar{\nu}),
\end{align*}
which indicates that
\begin{align*}
   V^{\mu^*\times \bar{\nu}, \Gc}_1(\xi)- \hat{V}_1(\xi) \leq 2\sum_{h \in  [H]}\sum_{s\in \Sc}d^{\mu^*\times \bar{\nu}, \Gc}_{h}(s) \Gamma_h(s, \mu^*\times\bar{\nu}).
\end{align*}

\noindent\textbf{Step 3: completing the proof with concentrability.} Let us consider $(s, a^1, h)\in \Sc \times \Ac^1 \times [H]$ such that $d^{\mu^*\times \bar{\nu},\Gc}_h(s, (a^1, \bar{\nu}_h(s)))>0$. We can then obtain that for all $l\in \Lc_h(s,(a^1, \bar{\nu}_h(s)))$, 
\begin{align*}
    N_{h,l}(s, (a^1, \bar{\nu}_h(s))) &\overset{(i)}{\geq} \frac{Kd^{\rho_l,\Gc_l}_h(s, (a^1, \bar{\nu}_h(s)))}{8} - 5 \sqrt{K d^{\rho_l,\Gc_l}_h(s, (a^1, \bar{\nu}_h(s))) \log\left(\frac{KHL}{\delta}\right)}\\
    &\overset{(ii)}{\geq} \frac{Kd^{\rho_l,\Gc_l}_h(s, (a^1, \bar{\nu}_h(s)))}{16}  \overset{(iii)}{\geq} 1.
\end{align*}
where inequality (i) is from Lemma~\ref{lem:subsampling}; and inequalities (ii) and (iii) are from the condition that
\begin{align*}
    K \geq \frac{c\log\left(\frac{KH}{\delta}\right)}{d^{\min}_g} \geq \frac{c\log\left(\frac{KH}{\delta}\right)}{d^{\rho_l,\Gc_l}_h(s, (a^1, \bar{\nu}_h(s)))}.
\end{align*}
Thus, it holds that
\begin{align*}
    \hat{L}_h(s, (a^1, \bar{\nu}_h(s))) = \sum_{l\in [L]}\oneb\{N_{h,l}(s,(a^1, \bar{\nu}_h(s)))\geq 1\} = L_h(s,(a^1, \bar{\nu}_h(s))).
\end{align*}
As a result, it holds that
\begin{align*}
    \Gamma_h(s, (a^1, \bar{\nu}_h(s)))  &\leq c\sqrt{\sum_{l\in \hat{\Lc}_h(s, (a^1, \bar{\nu}_h(s)))}\frac{H^2\log(SAH/\delta)}{\left(\hat{L}_h(s, (a^1, \bar{\nu}_h(s)))\right)^2 N_{h,l}(s, (a^1, \bar{\nu}_h(s)))}}+ c\sqrt{\frac{H^2\log(SAH/\delta)}{\hat{L}_h(s, (a^1, \bar{\nu}_h(s)))}}\\
    &\leq  c\sqrt{\sum_{l\in \Lc_h(s, (a^1, \bar{\nu}_h(s)))}\frac{H^2\log(SAH/\delta)}{\left(L_h(s, (a^1, \bar{\nu}_h(s)))\right)^2Kd^{\rho_l,\Gc_l}_h(s, a)}} + c\sqrt{\frac{H^2\log(SAH/\delta)}{L_h(s, (a^1, \bar{\nu}_h(s)))}}.
\end{align*}

We can then obtain that
\begin{align*}
    &\sum_{h \in  [H]}\sum_{(s, a^1)\in \Sc\times \Ac^1}d^{\mu^*\times \bar{\nu}, \Gc}_{h}(s, (a^1, \bar{\nu}_h(s))) \Gamma_h(s, (a^1, \bar{\nu}_h(s)))\\
    &\leq  c\sum_{h \in  [H]}\sum_{(s, a^1)\in \Sc\times \Ac^1}d^{\mu^*\times \bar{\nu}, \Gc}_{h}(s, (a^1, \bar{\nu}_h(s))) \sqrt{\frac{C^{\dagger}_g H^2\log(SAH/\delta)}{L^\dagger_g K\min\left\{d^{\mu^*\times \bar{\nu},\Gc}_h(s, (a^1, \bar{\nu}_h(s))), \frac{1}{SA_1}\right\}}}\\
    & + c\sum_{h \in  [H]}\sum_{(s, a^1)\in \Sc\times \Ac^1}d^{\mu^*\times \bar{\nu}, \Gc}_{h}(s, (a^1, \bar{\nu}_h(s)))\sqrt{\frac{H^2\log(SAH/\delta)}{L^{\dagger}_g}}\\
    &\leq  c\sum_{h \in  [H]}\sqrt{\sum_{(s, a^1)\in \Sc\times \Ac^1}\frac{d^{\mu^*\times \bar{\nu}, \Gc}_{h}(s, (a^1, \bar{\nu}_h(s)))C^{\dagger}_g H^2\log(SAH/\delta)}{L^{\dagger}_g  K \min\left\{d^{\mu^*\times \bar{\nu},\Gc}_h(s, (a^1, \bar{\nu}_h(s))), \frac{1}{SA_1}\right\}}}\sqrt{\sum_{(s,a^1)\in \Sc\times \Ac^1}d^{\pi^*, \Mc}_{h}((a^1, \bar{\nu}_h(s)))}\\
    & + c\sqrt{\frac{H^4\log(SAH/\delta)}{L^{\dagger}_g}}\\
    & \leq cH \sqrt{\frac{C^{\dagger}_g SA_1H^2\log(SAH/\delta)}{L^{\dagger}_g K }} + c\sqrt{\frac{H^4\log(SAH/\delta)}{L^{\dagger}_g}}.
\end{align*}

Putting these results together, it can then be established that
\begin{align*}
    V^{\mu^*\times \nu^*, \Gc}_1(\xi) - V^{\hat{\mu}\times \ber{\hat{\mu}}, \Gc}_1(\xi) &\leq V^{\mu^*\times \bar{\nu}, \Gc}_1(\xi) - \hat{V}_1(\xi)\leq 2\sum_{h \in  [H]}\sum_{s\in \Sc}d^{\mu^*\times \bar{\nu}, \Gc}_{h}(s) \Gamma_h(s, \mu^*\times\bar{\nu})\\
    &\leq  cH \sqrt{\frac{C^{\dagger}_g SA_1H^2\log(SAH/\delta)}{L^{\dagger}_g K }} + c\sqrt{\frac{H^4\log(SAH/\delta)}{L^{\dagger}_g}},
\end{align*}
which concludes the proof.

\section{Robust RL}
\subsection{Problem Formulation}
The following task--source relationship is considered for offline robust RL with perturbed data sources. It shares a similar content as Assumption~\ref{asp:relationship} while an additional mild constraint is added to have the transition probabilities of data source MDPs bounded in a regime around that of the nominal MDP. This constraint simplifies later analysis while it is left for future works to investigate its necessity.
\begin{assumption}[Task--source Relationship, Robust MDP]\label{asp:relationship_robust}
    Data source MDPs $\{\Mc_l = \{H, \Sc, \Ac, \Pb_{l}, r_l\}: l\in [L]\}$ are generated from an unknown set of distributions $g = \{g_h: h\in [H]\}$ such that for each $(l, h)\in [L]\times [H]$, the reward and transition $\{r_{h,l}, \Pb_{h,l}\}$ are independently sampled from the distribution $g_h(\cdot)$ whose expectation is $\{r_{h}, \Pb^c_{h}\}$ of the nominal MDP $\Mc^c = \{H, \Sc, \Ac, \Pb^c, r\}$. In addition, $\{r'_{h}, \Pb'_{h}\}\sim g_h(\cdot)$ satisfies that $\Pb'_h(s'|s,a) \in [T_l \cdot \Pb_h(s'|s,a), T_u \cdot \Pb_h(s'|s,a)]$ for constants $T_u<1, T_l>1$ at each $(s, a, h, s') \in \Sc\times \Ac\times [H]\times \Sc$.
\end{assumption}

\subsection{Algorithm Details of HetPEVI-Robust}\label{app_sub:alg_detail_robust}
The complete description of HetPEVI-Game can be found in Algorithm~\ref{alg:HetPEVI_Robust}.
 \begin{algorithm}[tbh]
	\caption{HetPEVI-Robust}
	\label{alg:HetPEVI_Robust}
	\begin{algorithmic}[1]
	    \STATE {\bfseries Input:} Dataset $\Dc = \{D_l:l\in[L]\}$
     \STATE Obtain $\Dc'_l \gets \text{subsampling}(\Dc_l), \forall l\in [L]$
     \FOR{$\forall (s, a, h, s') \in \Sc\times \Ac\times [H] \times \Sc$}
     \STATE $\forall l\in [L], N_{h, l}(s,a) \gets$ number of visitations on $(s, a, h)$ in $\Dc'_l$
     \STATE $\forall l\in [L], N_{h, l}(s,a, s') \gets$ number of visitations on $(s, a, h, s')$ in $\Dc'_l$ 
     \STATE $\hat{\Lc}_h(s,a)\gets \{l\in [L]: N_{h,l}(s,a)>0\}$
     \STATE $\forall l\in \hat{\Lc}_h(s,a), \hat{r}_{h,l}(s,a) \gets r_{h,l}(s,a), \hat{\Pb}_{h,l}(s'|s,a) \gets N_{h,l}(s,a,s')/N_{h,l}(s,a)$
     \STATE $\hat{r}_h(s,a)  \gets \sum_{l\in \hat{\Lc}_h(s,a)}\hat{r}_{h,l}(s,a)/(\hat{L}_h(s,a)\vee 1), \hat{\Pb}_h(s'|s,a) \gets \sum_{l\in \hat{\Lc}_h(s,a)}\hat{\Pb}_{h,l}(s'|s,a)/(\hat{L}_h(s,a)\vee 1)$
     \ENDFOR
     \STATE Initialize $\hat{V}_{H+1}(s) \gets 0, \forall s\in \Sc$ 
	    \FOR{$h = H, H-1,\cdots, 1$} 
        \FOR{$s \in \Sc$}
        \STATE $\Gamma^\sigma_h(s,a) \gets \min\left\{\frac{c}{\sigma \hat{\Pb}^{\min}_{h}(s,a)}\sqrt{\sum_{l\in \hat{\Lc}_h(s,a)}\frac{H^2\log(SAH/\delta)}{(\hat{L}_h(s,a))^2 N_{h,l}(s,a)}}
        + \frac{c}{\sigma \hat{\Pb}^{\min}_{h}(s,a)}\sqrt{\frac{H^2\log(SAH/\delta)}{\hat{L}_h(s,a)}} + c\sqrt{\frac{\log(SAH/\delta)}{\hat{L}_h(s,a)}}, H\right\}$
        \STATE $\hat{Q}_h(s,a) \gets \max\left\{\hat{r}_h(s,a) + \sup_{\lambda \geq 0}\left\{-\lambda \log\left(\left[\hat{\Pb}_h\exp\left(-\hat{V}_{h+1}/\lambda\right)\right](s,a)\right) - \lambda \sigma\right\} - \Gamma^\sigma_h(s,a), 0\right\}$
        \ENDFOR
        \FOR{$s\in \Sc$}
        \STATE $\hat{\pi}_h(s) \gets \argmax_{a\in \Ac}(s,a)$
        \STATE $\hat{V}_h(s) \gets \hat{Q}_h(s,\hat{\pi}_h(s))$
	    \ENDFOR
        \ENDFOR
	    \STATE {\bfseries Output:} policy $\hat\pi = \{\hat\pi_h(s): (s,h)\in \Sc\times [H]\}$
	\end{algorithmic}
	\end{algorithm}

\subsection{Core Lemmas}
First, we introduce the following notations:
\begin{align*}
    \Cc &:= \left\{(s, a, h): \exists l\in [L] \text{ s.t. } d^{\rho_l, \Mc_l}_h(s,a) >0\right\};\\
    \Pb^{\min}_{h}(s,a)&:= \min\left\{\Pb^c_h(s'|s,a): s' \text{ s.t. } \Pb^c_h(s'|s,a)>0\right\};\\
    \Pb^{\min}_{h,l}(s,a)&:= \min\left\{\Pb_{h,l}(s'|s,a): s' \text{ s.t. } \Pb_{h,l}(s'|s,a)>0\right\};\\
    \Pb^{\min}_{\sigma}&:= \min\left\{\Pb^c_h(s'|s,a): (s, a, h, s') \text{ s.t. } \exists l\in [L], d^{\rho_l, \Mc_l}_h(s,a) >0, \Pb^c_h(s'|s,a)>0\right\}\\
    & = \min\left\{\Pb^c_h(s'|s,a): (s, a, h, s' ) \text{ s.t. } (s, a, h) \in \Cc, \Pb^c_h(s'|s,a)>0\right\};\\
    \Pb^{\min}_{*}&:= \min\left\{\Pb^c_h(s'|s,a): (s, a, h, s') \text{ s.t. } d^{\pi^*, \Mc^c}_h(s,a) >0, \Pb^c_h(s'|s,a)>0\right\}\\
    d^{\pi^*, \Rc}_h(s)&: = \left\{d^{\pi^*,\Mc^\sigma}_h(s): \Mc^\sigma \in \Rc\right\};\\
    d^{\min}_{\sigma}&: = \min\left\{d^{\rho_l,\Mc_l}_h(s,a): (s, a,h, l) \text{ s.t. } d^{\rho_l,\Mc_l}_h(s,a)>0\right\}.
\end{align*}

A core lemma is then presented in the following.
\begin{lemma}\label{lem:concentration_robust}
For all $(s, a, h)\in \Sc\times \Ac\times [H]$, and any function $V: \Sc\to [0, H]$ independent of $\hat{\Pb}_h$, with probability at least $1-\delta$, it holds that
    \begin{align}\label{eqn:concentration_robust}
        \left|\hat{r}_h(s,a) + \inf_{\hat{\Pb}^\sigma_h(\cdot|s,a)\in \Uc^\sigma(\hat{\Pb}_h(\cdot|s,a))} \left(\hat{\Pb}^\sigma_h V\right)(s,a) - r_h(s,a) - \inf_{\Pb^\sigma_h(\cdot|s,a)\in \Uc^\sigma(\Pb^c_h(\cdot|s,a))} \left(\Pb^\sigma_h V\right)(s,a)\right| \leq \Gamma^\sigma_h(s,a).
    \end{align}
    Moreover, for all $(s, a, h)\in \Cc$, with probability at least $1-\delta$, it holds that
    \begin{align}\label{eqn:p_min}
    \Pb^{\min}_h(s,a) \geq \frac{1}{T_u}\frac{\hat{\Pb}^{\min}_h(s,a)}{e^2} \geq \frac{T_l}{T_u}\frac{\Pb^{\min}_h(s,a)}{8e^2\log(KHLSA/\delta)}.
    \end{align}
\end{lemma}
\begin{proof}[Proof of Lemma~\ref{lem:concentration_robust}]
We first prove the following fact that
\begin{align}\label{eqn:robust_visit_number}
    \forall (s, a, h) \in \Cc, l\in \Lc_h(s,a): \qquad N_{h,l}(s,a) \geq \frac{cT_l^2\log(KHLSA/\delta)}{16(\Pb^{\min}_{h,l}(s,a))^2} \geq - \frac{\log(2KHLSA/\delta)}{\log(1-\Pb^{\min}_{h,l}(s,a))}
\end{align}
In particular, with
\begin{align*}
    K \geq \frac{c\log(KHLSA/\delta)}{d^{\min}_{\sigma}(\Pb^{\min}_{\sigma})^2},
\end{align*}
it holds that
\begin{align*}
    K d^{\rho_l, \Mc_l}_h(s,a) \geq \frac{cd^{\rho_l, \Mc_l}_h(s,a)\log(KHLSA/\delta)}{d^{\min}_{\sigma} (\Pb^{\min}_{\sigma})^2} \geq \frac{c\log(KHLSA/\delta)}{(\Pb^{\min}_{\sigma})^2} \geq \frac{c\log(KHLSA/\delta)}{(\Pb^{\min}_h(s,a))^2} \geq \frac{cT_l^2\log(KHLSA/\delta)}{(\Pb^{\min}_{h,l}(s,a))^2}.
\end{align*}
Lemma~\ref{lem:subsampling} then indicates that with probability at least $1-8\delta$, 
\begin{align*}
    N_{h,l}(s,a) &\geq \frac{Kd^{\rho_l, \Mc_l}_h(s,a)}{8} - 5\sqrt{Kd^{\rho_l, \Mc_l}_h(s,a)\log\left(\frac{KHL}{\delta}\right)}\geq \frac{Kd^{\rho_l, \Mc_l}_h(s,a)}{16} \geq \frac{cT_l^2\log(KHLSA/\delta)}{16(\Pb^{\min}_{h,l}(s,a))^2}.
\end{align*}
Furthermore, with $x\leq - \log(1-x)$ for all $x\in [0, 1]$ and a suitable $c$, it holds that
\begin{align*}
    \frac{cT_l^2\log(KHLSA/\delta)}{16(\Pb^{\min}_{h,l}(s,a))^2} \geq \frac{cT_l^2\log(KHLSA/\delta)}{16\Pb^{\min}_{h,l}(s,a)} \geq - \frac{\log(2KHLSA/\delta)}{\log(1-\Pb^{\min}_{h,l}(s,a))},
\end{align*}
which completes the proof of Eqn.~\eqref{eqn:robust_visit_number}.

Then, with Lemma~\ref{lem:binomial_concentration}, we can obtain that
\begin{align*}
    \Pb_{h,l}(s'|s,a) \geq \frac{\hat{\Pb}_{h,l}(s'|s,a)}{e^2} \geq \frac{\Pb_{h,l}(s'|s,a)}{8e^2\log(KHLSA/\delta)},
\end{align*}
This result further indicates that
\begin{align*}
    \Pb^{\min}_{h}(s,a) &= \Pb^c_{h}(s_1|s,a) \geq \frac{1}{T_u}\sum_{l\in \hat{\Lc}_h(s,a)}\frac{\Pb_{h,l}(s_1|s,a)}{\hat{L}_h(s,a)} \geq \frac{1}{T_u}\sum_{l\in \hat{\Lc}_h(s,a)}\frac{\hat{\Pb}_{h,l}(s_1|s,a)}{e^2\hat{L}_h(s,a)}\\
    &\geq \frac{1}{T_u}\sum_{l\in \hat{\Lc}_h(s,a)}\frac{\hat{\Pb}_{h,l}(s_2|s,a)}{e^2\hat{L}_h(s,a)} = \frac{1}{T_u}\frac{\hat{\Pb}^{\min}_{h}(s,a)}{e^2} \geq \frac{1}{T_u}\sum_{l\in \hat{\Lc}_h(s,a)}\frac{\Pb_{h,l}(s_2|s,a)}{8e^2\hat{L}_h(s,a)\log(KHLSA/\delta)} \\
    &\geq \frac{T_l}{T_u}\sum_{l\in \hat{\Lc}_h(s,a)}\frac{\Pb^c_{h}(s_2|s,a)}{8e^2\hat{L}_h(s,a)\log(KHLSA/\delta)} \geq \frac{T_l}{T_u}\frac{\Pb^c_{\min,h}(s,a)}{8e^2\log(KHLSA/\delta)},
\end{align*}
where $\Pb^c_{h}(s_1|s,a) = \Pb^{\min}_{h}(s,a)$ and $\hat{\Pb}_{h}(s_2|s,a) = \hat{\Pb}^{\min}_{h}(s,a)$. Eqn.~\eqref{eqn:p_min} is thus proved.

Then, we prove the first part in this lemma, i.e., Eqn~\eqref{eqn:concentration_robust}. It can be first observed that
\begin{align}
    \left|\hat{\Pb}_h(s'|s,a) - \Pb_h(s'|s,a)\right| &= \left|\sum_{l\in \hat{\Lc}_h(s,a)} \frac{1}{\hat{L}_h(s,a)}\hat{\Pb}_{h,l}(s'|s,a) - \Pb_h(s'|s,a)\right| \notag\\
    &\leq \left|\sum_{l\in \hat{\Lc}_h(s,a)} \frac{1}{\hat{L}_h(s,a)}\hat{\Pb}_{h,l}(s'|s,a) - \sum_{l\in \hat{\Lc}_h(s,a)} \frac{1}{\hat{L}_h(s,a)}\Pb_{h,l}(s'|s,a)\right| \notag\\
    & + \left|\sum_{l\in \hat{\Lc}_h(s,a)} \frac{1}{\hat{L}_h(s,a)}\Pb_{h,l}(s'|s,a) - \Pb_h(s'|s,a)\right|\notag \\
    &\leq \sqrt{\sum_{l\in \hat{\Lc}_h(s,a)}\frac{2\log(S^2AH/\delta)}{(\hat{L}_h(s,a))^2 N_{h,l}(s,a)}}+ \sqrt{\frac{2\log(S^2AH/\delta)}{\hat{L}_h(s,a)}} \label{eqn:im_concentration_robust}.
\end{align}
where the last step holds with probability at least $1-2\delta/(S^2AH)$ according to the Hoeffidng inequality. This inequality thus holds with probability at least $1-2\delta$ for all $(s, a, h, s')\in \Sc\times \Ac\times [H] \times \Sc$. It is further indicated that for a function $V: \Sc\to [0, H]$
\begin{align*}
    &\frac{\left|\left[\hat{\Pb}_h\exp\left(-\frac{V}{\lambda}\right)\right](s,a) - \left[\Pb^c_h\exp\left(-\frac{V}{\lambda}\right)\right](s,a)\right|}{\left[\Pb^c_h\exp\left(-\frac{V}{\lambda}\right)\right](s,a)} \\
    &\leq \max_{s' \in \text{supp}(\Pb^c_h(\cdot|s,a))}\frac{\left|\hat{\Pb}_h(s'|s,a) - \Pb^c_h(s'|s,a)\right|}{\Pb^c_h(s'|s,a)}\\
    &\leq \frac{1}{\Pb^{\min}_{h}(s,a)}\left(\sqrt{\sum_{l\in \hat{\Lc}_h(s,a)}\frac{2\log(S^2AH/\delta)}{(\hat{L}_h(s,a))^2 N_{h,l}(s,a)}}+ \sqrt{\frac{2\log(S^2AH/\delta)}{\hat{L}_h(s,a)}}\right)\\
    &\leq\frac{1}{2}
\end{align*}
where the last inequality holds due to that with a suitable $c$,
\begin{align*}
    &N_{h,l}(s,a) \geq \frac{cT_l^2\log(KHLSA/\delta)}{16(\Pb^{\min}_{h,l}(s,a))^2}\geq \frac{8\log(S^2AH/\delta)}{(\Pb^{\min}_{h,l}(s,a))^2}\geq 1, \qquad \forall l\in \Lc_h(s,a) \qquad \Rightarrow \Lc_h(s,a) = \hat{\Lc}_h(s,a);\\
    &|\hat{\Lc}_h(s,a)| = |\Lc_h(s,a)| \geq \frac{8\log(S^2AH/\delta)}{(\Pb^{\min}_{h}(s,a))^2}.
\end{align*}

From Lemma~\ref{lem:dual}, it can then be observed that
\begin{align*}
        &\left|\inf_{\Pb^\sigma_h(\cdot|s,a)\in \Uc^\sigma(\hat{\Pb}_h(\cdot|s,a))} \left(\Pb^\sigma_h V\right)(s,a) - \inf_{\Pb^\sigma_h(\cdot|s,a)\in \Uc^\sigma(\Pb^c_h(\cdot|s,a))} \left(\Pb^\sigma_h V\right)(s,a)\right|\\
        &= \left|\sup_{\lambda >0}\left\{-\lambda \log\left(\left[\hat{\Pb}_h\exp\left(-\frac{V}{\lambda}\right)\right](s,a)\right) - \lambda \sigma\right\} - \sup_{\lambda >0}\left\{-\lambda \log\left(\left[\Pb^c_h\exp\left(-\frac{V}{\lambda}\right)\right](s,a)\right) - \lambda \sigma\right\}\right|.
    \end{align*}
Denote 
\begin{align*}
    \hat{\lambda}_h(s,a) = \argmax_{\lambda \geq 0}\left\{-\lambda \log\left(\left[\hat{\Pb}_h\exp\left(-\frac{V}{\lambda}\right)\right](s,a)\right) - \lambda \sigma\right\}\\
        \lambda_h(s,a) = \argmax_{\lambda \geq 0}\left\{-\lambda \log\left(\left[\Pb^c_h\exp\left(-\frac{V}{\lambda}\right)\right](s,a)\right) - \lambda \sigma\right\},
\end{align*}
with Lemma~\ref{lem:dual_bound}, we can further obtain that
\begin{align*}
        \hat{\lambda}_h(s,a) \in \left[0, \frac{H}{\sigma}\right], \qquad \lambda_h(s,a) \in \left[0, \frac{H}{\sigma}\right].
\end{align*}
In the following, we consider several different cases.

\textit{Case (I): $ \hat{\lambda}_h(s,a)> 0$ and $ \lambda_h(s,a)>0$.} In this case, it follows that
\begin{align*}
    &\left|\sup_{\lambda >0}\left\{-\lambda \log\left(\left[\hat{\Pb}_h\exp\left(-\frac{V}{\lambda}\right)\right](s,a)\right) - \lambda \sigma\right\} - \sup_{\lambda >0}\left\{-\lambda \log\left(\left[\Pb^c_h\exp\left(-\frac{V}{\lambda}\right)\right](s,a)\right) - \lambda \sigma\right\}\right|\\
    &\leq \max\Bigg\{-\hat{\lambda}_h(s,a) \log\left(\left[\hat{\Pb}_h\exp\left(-\frac{V}{\hat{\lambda}_h(s,a)}\right)\right](s,a)\right) + \hat{\lambda}_h(s,a) \log\left(\left[\Pb^c_h\exp\left(-\frac{V}{\hat{\lambda}_h(s,a)}\right)\right](s,a)\right),\\
    & -\lambda_h(s,a) \log\left(\left[\Pb^c_h\exp\left(-\frac{V}{\lambda_h(s,a)}\right)\right](s,a)\right) + \lambda_h(s,a) \log\left(\left[\hat{\Pb}_h\exp\left(-\frac{V}{\lambda_h(s,a)}\right)\right](s,a)\right)\Bigg\}\\
    &\leq \max_{\lambda\in \{\hat{\lambda}_h(s,a), \lambda_h(s,a)\}} \lambda \left|\log\left(\left[\hat{\Pb}_h\exp\left(-\frac{V}{\lambda}\right)\right](s,a)\right) - \log\left(\left[\Pb^c_h\exp\left(-\frac{V}{\lambda}\right)\right](s,a)\right)\right|\\
    & \leq \max_{\lambda\in \{\hat{\lambda}_h(s,a), \lambda_h(s,a)\}} \lambda \left|\log\left(1 + \frac{\left[\hat{\Pb}_h\exp\left(-\frac{V}{\lambda}\right)\right](s,a) - \left[\Pb^c_h\exp\left(-\frac{V}{\lambda}\right)\right](s,a)}{\left[\Pb^c_h\exp\left(-\frac{V}{\lambda}\right)\right](s,a)}\right)\right|\\
    &\leq \max_{\lambda\in \{\hat{\lambda}_h(s,a), \lambda_h(s,a)\}} 2\lambda \cdot \frac{\left|\left[\hat{\Pb}_h\exp\left(-\frac{V}{\lambda}\right)\right](s,a) - \left[\Pb^c_h\exp\left(-\frac{V}{\lambda}\right)\right](s,a)\right|}{\left[\Pb^c_h\exp\left(-\frac{V}{\lambda}\right)\right](s,a)}\\
    & \leq \frac{2H}{\sigma}\cdot \frac{1}{\Pb^{\min}_{h}(s,a)}\left(\sqrt{\sum_{l\in \hat{\Lc}_h(s,a)}\frac{2\log(S^2AH/\delta)}{(\hat{L}_h(s,a))^2 N_{h,l}(s,a)}}+ \sqrt{\frac{2\log(S^2AH/\delta)}{\hat{L}_h(s,a)}}\right)\\
    &\leq \frac{2H}{\sigma}\cdot \frac{T_u e^2}{\hat{\Pb}^{\min}_{h}(s,a)}\left(\sqrt{\sum_{l\in \Lc_h(s,a)}\frac{2\log(S^2AH/\delta)}{(\hat{L}_h(s,a))^2 N_{h,l}(s,a)}}+ \sqrt{\frac{2\log(S^2AH/\delta)}{\hat{L}_h(s,a)}}\right).
\end{align*}

\textit{Case (II): $ \hat{\lambda}_h(s,a)> 0$ and $ \lambda_h(s,a) = 0$; $ \hat{\lambda}_h(s,a)= 0$ and $ \lambda_h(s,a) > 0$.}
We consider the sub-case that $ \hat{\lambda}_h(s,a)>0$ and $ \lambda_h(s,a) = 0$ while the other sub-case can be proved similarly. In particular, with Lemma~\ref{lem:dual_bound} and Lemma~\ref{lem:essinf}, we can obtain that
\begin{align*}
    \sup_{\lambda \geq 0}\left\{-\lambda \log\left(\left[\hat{\Pb}_h\exp\left(-\frac{V}{\lambda}\right)\right](s,a)\right) - \lambda \sigma\right\} &\geq 
    \lim_{\lambda \to 0}\left\{-\lambda \log\left(\left[\hat{\Pb}_h\exp\left(-\frac{V}{\lambda}\right)\right](s,a)\right) - \lambda \sigma\right\}\\
    & = \essinf_{s' \sim \hat{\Pb}_h(\cdot|s,a)} V(s') = \inf \essinf_{s'\sim \hat{\Pb}_{h,l}(\cdot|s,a)}V(s')\\
    & = \inf \essinf_{s'\sim \Pb_{h,l}(\cdot|s,a)}V(s') = \essinf_{s'\sim \Pb^c_{h}(\cdot|s,a)}V(s') \\
    & = \sup_{\lambda \geq 0}\left\{-\lambda \log\left(\left[\Pb^c_h\exp\left(-\frac{V}{\lambda}\right)\right](s,a)\right) - \lambda \sigma\right\}.
\end{align*}
As a result, it holds that
\begin{align*}
    &\left|\sup_{\lambda >0}\left\{-\lambda \log\left(\left[\hat{\Pb}_h\exp\left(-\frac{V}{\lambda}\right)\right](s,a)\right) - \lambda \sigma\right\} - \sup_{\lambda >0}\left\{-\lambda \log\left(\left[\Pb^c_h\exp\left(-\frac{V}{\lambda}\right)\right](s,a)\right) - \lambda \sigma\right\}\right|\\
    &= \sup_{\lambda >0}\left\{-\lambda \log\left(\left[\hat{\Pb}_h\exp\left(-\frac{V}{\lambda}\right)\right](s,a)\right) - \lambda \sigma\right\} - \sup_{\lambda >0}\left\{-\lambda \log\left(\left[\Pb^c_h\exp\left(-\frac{V}{\lambda}\right)\right](s,a)\right) - \lambda \sigma\right\}\\
    & \leq -\hat{\lambda}_{h}(s,a) \log\left(\left[\hat{\Pb}_h\exp\left(-\frac{V}{\hat{\lambda}_{h}(s,a)}\right)\right](s,a)\right) - \hat{\lambda}_{h}(s,a) \sigma \\
    & +\hat{\lambda}_{h}(s,a) \log\left(\left[\Pb^c_h\exp\left(-\frac{V}{\hat{\lambda}_{h}(s,a)}\right)\right](s,a)\right) + \hat{\lambda}_{h}(s,a) \sigma\\
    & = \hat{\lambda}_h(s,a)\left[\log\left(\left[\Pb^c_h\exp\left(-\frac{V}{\hat{\lambda}_{h}(s,a)}\right)\right](s,a)\right) - \hat{\lambda}_{h}(s,a) \log\left(\left[\hat{\Pb}_h\exp\left(-\frac{V}{\hat{\lambda}_{h}(s,a)}\right)\right](s,a)\right)\right],
\end{align*}
which can be bounded via the same steps in Case (I).

\textit{Case (III): $ \hat{\lambda}_h(s,a)= \lambda_h(s,a) = 0$.} With Lemma~\ref{lem:dual_bound} and Lemma~\ref{lem:essinf}, it holds that
\begin{align*}
    &\left|\sup_{\lambda >0}\left\{-\lambda \log\left(\left[\hat{\Pb}_h\exp\left(-\frac{V}{\lambda}\right)\right](s,a)\right) - \lambda \sigma\right\} - \sup_{\lambda >0}\left\{-\lambda \log\left(\left[\Pb^c_h\exp\left(-\frac{V}{\lambda}\right)\right](s,a)\right) - \lambda \sigma\right\}\right|\\
    & = \left|\essinf_{s' \sim \hat{\Pb}_h(\cdot|s,a)} V(s') - \essinf_{s'\sim \Pb^c_{h}(\cdot|s,a)}V(s')\right| = \left|\inf_{l\in [L]}\essinf_{s' \sim \hat{\Pb}_{h,l}(\cdot|s,a)} V(s') - \essinf_{s'\sim \Pb^c_{h}(\cdot|s,a)}V(s')\right|\\
    & = \left|\inf_{l\in [L]}\essinf_{s' \sim {\Pb}_{h,l}(\cdot|s,a)} V(s') - \essinf_{s'\sim \Pb^c_{h}(\cdot|s,a)}V(s')\right| = 0.
\end{align*}

Together with the fact that with probability at least $1-\delta$,
\begin{align*}
    \hat{r}_h(s,a) - r_h(s,a)\leq \sqrt{\frac{\log(1/(HSA\delta))}{\hat{L}_h(s,a)}}, \qquad \forall (s, a, h)\in \Sc\times \Ac\times [H].
\end{align*}
Eqn.~\eqref{eqn:concentration_robust} is shown to be valid.
\end{proof}

\subsection{Main Proofs}
\begin{proof}[Proof of Theorem~\ref{thm:HetPEVI_robust}]
In the following, we establish Theorem~\ref{thm:HetPEVI_robust}. The proof framework is inspired by \citet{shi2022distributionally} but is uniquely adapted to handle randomly perturbed data sources.

\noindent\textbf{Step 1: establishing the pessimism property.} With Lemma~\ref{lem:concentration_robust}, we first show that the following inequalities hold with probability at least $1-\delta$:
\begin{align}\label{eqn:pess_value_robust}
    \hat{Q}_h(s,a) \leq Q^{\hat{\pi},\Rc}_h(s,a), \qquad \hat{V}_h(s,a) \leq V^{\hat{\pi},\Rc}_h(s,a), \qquad \forall (s, a, h)\in \Sc\times \Ac\times [H].
\end{align}
Towards this, it is first observed that
\begin{align*}
    \hat{Q}_{H+1}(s,a) = Q^{\hat{\pi},\Rc}_{H+1}(s,a) = 0, \qquad \forall (s, a)\in \Sc\times \Ac.
\end{align*}
Then, suppose that $\hat{Q}_{h+1}(s,a) \leq  Q^{\hat{\pi},\Rc}_{h+1}(s,a)$ for all $(s,a)\in \Sc\times \Ac$ at some step $h\in [H]$, we can observe that by the update rule in HetPEVI-Game, it holds that
\begin{align*}
    0 \leq \hat{V}_{h+1}(s) = \max_{a\in \Ac} \hat{Q}_{h+1}(s,a) \leq \max_{a\in \Ac} Q^{\hat{\pi}, \Rc}_{h+1}(s,a) = V^{\hat{\pi},\Rc}_{h+1}(s)\leq H, \qquad \forall s\in \Sc,
\end{align*}
If $\hat Q_h(s,a) = 0$, the claim naturally holds. If not, we can obtain that
\begin{align*}
     \hat Q_h(s,a) & = \hat{r}_h(s,a) + \sup_{\lambda >0}\left\{-\lambda \log\left(\left[\hat{\Pb}_h\exp\left(-\frac{\hat{V}_{h+1}}{\lambda}\right)\right](s,a)\right) - \lambda \sigma\right\} - \Gamma^\sigma_h(s,a)\\
     & = \hat{r}_h(s,a) + \inf_{\hat{\Pb}_h^\sigma(\cdot|s,a)\in \Uc^\sigma(\hat{\Pb}_h(\cdot|s,a))} \left(\hat{\Pb}_h^\sigma \hat{V}_{h+1}\right)(s,a) - \Gamma^\sigma_h(s,a)\\
     & \leq r_h(s,a) + \inf_{\Pb_h^\sigma(\cdot|s,a)\in \Uc^\sigma(\Pb^c_h(\cdot|s,a))} \left(\Pb_h^\sigma \hat{V}_{h+1}\right)(s,a) - \Gamma^\sigma_h(s,a) \\
     & + \left|r_h(s,a) + \inf_{\Pb_h^\sigma(\cdot|s,a)\in \Uc^\sigma(\Pb^c_h(\cdot|s,a))} \left(\Pb_h^\sigma \hat{V}_{h+1}\right)(s,a) - \hat{r}_h(s,a) - \inf_{\hat{\Pb}_h^\sigma(\cdot|s,a)\in \Uc^\sigma(\hat{\Pb}_h(\cdot|s,a))} \left(\hat{\Pb}_h^\sigma \hat{V}_{h+1}\right)(s,a)\right| \\
     & \overset{(i)}{\leq} r_h(s,a) + \inf_{\Pb_h^\sigma(\cdot|s,a)\in \Uc^\sigma(\Pb^c_h(\cdot|s,a))} \left(\Pb_h^\sigma \hat{V}_{h+1}\right)(s,a)\\
     & \overset{(ii)}{\leq} r_h(s,a) + \inf_{\Pb_h^\sigma(\cdot|s,a)\in \Uc^\sigma(\Pb^c_h(\cdot|s,a))} \left(\Pb_h^\sigma V^{\hat{\pi}, \Rc}_{h+1}\right)(s,a)= Q^{\hat{\pi},\Rc}_h(s,a).
\end{align*}
The above inequality (i) is from Lemma~\ref{lem:concentration_robust} and leverages the fact that $\hat{V}_{h+1}(\cdot)$ is independent of $\hat{\Pb}_h$ and takes value in $[0, H]$. Inequality (ii) is from the obtained fact that $\hat{V}_{h+1}(s)\leq V^{\hat{\pi},\Rc}_{h+1}(s)$. The desired claim Eqn.~\eqref{eqn:pess_value_robust} can be verified by induction.

\noindent\textbf{Step 2: bounding the performance difference.}
From Eqn.~\eqref{eqn:pess_value_robust}, we can observe that
\begin{align*}
    0 \leq V^{\pi^*, \Rc}_h(s) - V^{\hat{\pi}, \Rc}_h(s) \leq V^{\pi^*, \Rc}_h(s) - \hat{V}_h(s) \leq Q^{\pi^*,\Rc}_h(s, \pi^*_h(s)) - \hat{Q}_h(s, \pi^*_h(s)).
\end{align*}
With
\begin{align*}
    Q^{\pi^*,\Rc}_h(s, \pi^*_h(s)) &= r_h(s, \pi^*_h(s)) + \inf_{\Pb^\sigma_{h}(\cdot|s,a)\in \Uc^\sigma(\Pb^c_h(\cdot|s,a)}\left(\Pb^\sigma_{h}V^{\pi^*,\Rc}_{h+1}\right)(s,a)\\
    \hat{Q}_h(s, \pi^*_h(s)) & = \max\left\{\hat{r}_h(s, \pi^*_h(s)) + \inf_{\hat{\Pb}^\sigma_{h}(\cdot|s,a)\in \Uc^\sigma(\hat{\Pb}_h(\cdot|s,a)}\left(\hat{\Pb}^\sigma_{h}\hat{V}_{h+1}\right)(s,a) - \Gamma^\sigma_h(s, \pi^*_h(s)), 0\right\},
\end{align*}
we can further obtain that
\begin{align*}
    V^{\pi^*, \Rc}_h(s) - \hat{V}_h(s) &\leq r_h(s, \pi^*_h(s)) + \inf_{\Pb^\sigma_{h}(\cdot|s,\pi^*_h(s))\in \Uc^\sigma(\Pb^c_h(\cdot|s,\pi^*_h(s))}\left(\Pb^\sigma_{h}V^{\pi^*,\Rc}_{h+1}\right)(s,\pi^*_h(s))\\
    & - \hat{r}_h(s, \pi^*_h(s)) - \inf_{\hat{\Pb}^\sigma_{h}(\cdot|s,\pi^*_h(s))\in \Uc^\sigma(\hat{\Pb}_h(\cdot|s,\pi^*_h(s))}\left(\hat{\Pb}^\sigma_{h}\hat{V}_{h+1}\right)(s,\pi^*_h(s)) + \Gamma^\sigma_h(s, \pi^*_h(s))\\
    &= r_h(s, \pi^*_h(s)) + \inf_{\Pb^\sigma_{h}(\cdot|s,\pi^*_h(s))\in \Uc^\sigma(\Pb^c_h(\cdot|s,\pi^*_h(s))}\left(\Pb^\sigma_{h}V^{\pi^*,\Rc}_{h+1}\right)(s,\pi^*_h(s))\\
    & - r_h(s, \pi^*_h(s)) - \inf_{\Pb^\sigma_{h}(\cdot|s,\pi^*_h(s))\in \Uc^\sigma(\Pb^c_h(\cdot|s,\pi^*_h(s))}\left(\Pb^\sigma_{h}\hat{V}_{h+1}\right)(s,\pi^*_h(s)) \\
    & + r_h(s, \pi^*_h(s)) + \inf_{\Pb^\sigma_{h}(\cdot|s,\pi^*_h(s))\in \Uc^\sigma(\Pb^c_h(\cdot|s,\pi^*_h(s))}\left(\Pb^\sigma_{h}\hat{V}_{h+1}\right)(s,\pi^*_h(s)) \\
    & - \hat{r}_h(s, \pi^*_h(s)) - \inf_{\hat{\Pb}^\sigma_{h}(\cdot|s,\pi^*_h(s))\in \Uc^\sigma(\hat{\Pb}_h(\cdot|s,\pi^*_h(s))}\left(\hat{\Pb}^\sigma_{h}\hat{V}_{h+1}\right)(s,\pi^*_h(s))  + \Gamma^\sigma_h(s, \pi^*_h(s))\\
    &\overset{(i)}{\leq} \inf_{\Pb^\sigma_{h}(\cdot|s,\pi^*_h(s))\in \Uc^\sigma(\Pb^c_h(\cdot|s,\pi^*_h(s))}\left(\Pb^\sigma_{h}V^{\pi^*,\Rc}_{h+1}\right)(s,\pi^*_h(s))\\
    & -  \inf_{\Pb^\sigma_{h}(\cdot|s,\pi^*_h(s))\in \Uc^\sigma(\Pb^c_h(\cdot|s,\pi^*_h(s))}\left(\Pb^\sigma_{h}\hat{V}_{h+1}\right)(s,\pi^*_h(s))  + 2\Gamma^\sigma_h(s, \pi^*_h(s))\\
    &\overset{(ii)}{\leq} \left(\Pb^{\inf}_{h}V^{\pi^*,\Rc}_{h+1}\right)(s,\pi^*_h(s))-  \left(\Pb^{\inf}_{h}\hat{V}_{h+1}\right)(s,\pi^*_h(s))  + 2\Gamma^\sigma_h(s, \pi^*_h(s))
\end{align*}
where inequality (i) holds with probability at least $1-\delta$ according to Lemma~\ref{lem:concentration_robust} and inequality (ii) holds with the notation
\begin{align*}
    \Pb^{\inf}_{h}(\cdot|s,a) = \argmin_{\Pb^\sigma_h\in \Uc^\sigma(\Pb^c_h(\cdot|s,a))} \left(\Pb^{\sigma}_{h}\hat{V}_{h+1}\right)(s,a).
\end{align*}
If applying the above argument iteratively, we can further obtain that
\begin{align*}
    \sum_{s\in \Sc} d^{\inf}_h(s)\left(V^{\pi*, \Rc}_h(s) - \hat{V}_h(s)\right) \leq 2\sum_{h' = h}^H\sum_{s\in \Sc}d^{\inf}_{h'}(s) \Gamma_{h'}(s, \pi^*_h(s)),
\end{align*}
where $d^{\inf}_h(s)$ denotes the visitation probability induced by optimal policy $\pi^*$ and $\Pb^{\inf} = \{\Pb^{\inf}_h: h\in [H]\}$. Finally, it can be obtained that
\begin{align*}
    V^{\pi*, \Rc}_1(\xi) - \hat{V}_1(\xi) \leq 2\sum_{h \in  [H]}\sum_{s\in \Sc}d^{\inf}_{h}(s) \Gamma_h(s, \pi^*_h(s)),
\end{align*}
and $d^{\inf}_h(s) \in d^{\pi^*, \Rc}_h(s)$.

\noindent\textbf{Step 3: completing the proof with concentrability.} Let us consider $(s, h)\in \Sc \times [H]$ such that $d^{\inf}_h(s)>0$. We can then obtain that for all $l\in \Lc_h(s,\pi^*_h(s))$, 
\begin{align*}
    N_{h,l}(s,\pi^*_h(s)) &\overset{(i)}{\geq} \frac{Kd^{\rho_l,\Mc_l}_h(s,\pi^*_h(s))}{8} - 5 \sqrt{K d^{\rho_l,\Mc_l}_h(s,\pi^*_h(s)) \log\left(\frac{KHL}{\delta}\right)} \overset{(ii)}{\geq} \frac{Kd^{\rho_l,\Mc_l}_h(s,\pi^*_h(s))}{16}  \overset{(iii)}{\geq} 1.
\end{align*}
where inequality (i) is from Lemma~\ref{lem:subsampling}; and inequalities (ii) and (iii) are from the condition that
\begin{align*}
    K \geq \frac{c\log(KHLSA/\delta)}{d^{\min}_{\sigma}(\Pb^{\min}_{\sigma})^2} \geq \frac{c\log(KHL/\delta)}{d^{\rho_l, \Mc_l}_h(s,a)}.
\end{align*}
Thus, it holds that
\begin{align*}
    \hat{L}_h(s, \pi^*_h(s)) = \sum_{l\in [L]}\oneb\{N_{h,l}(s, \pi^*_h(s))\geq 1\} = L_h(s, \pi^*_h(s)).
\end{align*}
As a result, it holds that
\begin{align*}
    \Gamma_h(s, \pi^*_h(s)) &\leq \frac{c}{\sigma \hat{\Pb}^{\min}_{h}(s,\pi^*_h(s))}\sqrt{\sum_{l\in \hat{\Lc}_h(s,a)}\frac{H^2\log(SAH/\delta)}{(\hat{L}_h(s,\pi^*_h(s)))^2 N_{h,l}(s,\pi^*_h(s))}}
        \\
        &+ \frac{c}{\sigma \hat{\Pb}^{\min}_{h}(s,\pi^*(s))}\sqrt{\frac{H^2\log(SAH/\delta)}{\hat{L}_h(s,\pi^*_h(s))}} + c\sqrt{\frac{\log(SAH/\delta)}{\hat{L}_h(s,\pi^*_h(s))}}\\
        &\leq \frac{c}{\sigma \hat{\Pb}^{\min}_{h}(s,\pi^*_h(s))}\sqrt{\sum_{l\in \Lc_h(s,a)}\frac{H^2\log(SAH/\delta)}{(L_h(s,\pi^*_h(s)))^2 K d^{\rho_l, \Mc_l}(s,\pi^*_h(s))}}
        \\
        &+ \frac{c}{\sigma \hat{\Pb}^{\min}_{h}(s,\pi^*(s))}\sqrt{\frac{H^2\log(SAH/\delta)}{L_h(s,\pi^*_h(s))}} + c\sqrt{\frac{\log(SAH/\delta)}{L_h(s,\pi^*_h(s))}}.
\end{align*}

We can then obtain that
\begin{align*}
    &\sum_{h \in  [H]}\sum_{s\in \Sc}d^{\inf}_{h}(s) \Gamma_h(s, \pi^*_h(s))\\
    &\leq  c\sum_{h \in  [H]}\sum_{s\in \Sc}d^{\inf}_{h}(s) \frac{1}{\sigma \cdot \hat{\Pb}^{\min}_{h}(s,\pi^*_h(s))}\sqrt{\frac{C^{\dagger}_{\sigma} H^2\log(SAH/\delta)}{L^\dagger_{\sigma} K\min\left\{d^{\inf}_h(s), \frac{1}{S}\right\}}}\\
    & + \sum_{h \in  [H]}\sum_{s\in \Sc}d^{\inf}_{h}(s)\left(1 + \frac{H}{\sigma \cdot \hat{\Pb}^{\min}_{h}(s,\pi^*_h(s))}\right)\sqrt{\frac{H^2\log(SAH/\delta)}{L^\dagger_{\sigma}}}\\
    &\leq  c\sum_{h \in  [H]}\sum_{s\in \Sc}d^{\inf}_{h}(s) \frac{H\log(KHLSA/\delta)}{\sigma \cdot \Pb^{\min}_{h}(s,\pi^*_h(s))}\sqrt{\frac{C^{\dagger}_{\sigma} \log(SAH/\delta)}{L^\dagger_{\sigma} K\min\left\{d^{\inf}_h(s), \frac{1}{S}\right\}}}\\
    & + c\sum_{h \in  [H]}\sum_{s\in \Sc}d^{\inf}_{h}(s)\left(1 + \frac{H\log(KHLSA/\delta)}{\sigma \cdot \Pb^{\min}_{h}(s,\pi^*_h(s))}\right)\sqrt{\frac{\log(SAH/\delta)}{L^\dagger_{\sigma}}}\\
    &\leq  c\sum_{h \in  [H]}\frac{H^2}{\sigma \cdot \Pb^{\min}_*}\sqrt{\sum_{s\in \Sc}d^{\inf}_{h}(s)\frac{C^{\dagger}_{\sigma}\log(SAH/\delta)}{L^{\dagger}_{\sigma}  K \min\left\{d^{\inf}_h(s), \frac{1}{S}\right\}}}\sqrt{\sum_{s\in \Sc}d^{\inf}_{h}(s)}\\
    & + cH\left(1 + \frac{H\log(KHLSA/\delta)}{\sigma \cdot \Pb^{\min}_{*}}\right)\sqrt{\frac{\log(SAH/\delta)}{L^\dagger_{\sigma}}}\\
    & \leq c\frac{H^2}{\sigma \cdot \Pb^{\min}_*}  \sqrt{\frac{C^{\dagger}_{\sigma}S\log(SAH/\delta)}{L^\dagger_{\sigma} K }} + cH\left(1 + \frac{H\log(KHLSA/\delta)}{\sigma \cdot \Pb^{\min}_{*}}\right)\sqrt{\frac{\log(SAH/\delta)}{L^\dagger_{\sigma}}}.
\end{align*}

Putting these results together, it can then be established that
\begin{align*}
    V^{\pi^*, \Rc}_1(\xi) - V^{\hat{\pi}, \Rc}_1(\xi) &\leq V^{\pi^*, \Rc}_1(\xi) - \hat{V}_1(\xi)\leq 2\sum_{h \in  [H]}\sum_{s\in \Sc}d^{\inf}_{h}(s) \Gamma_h(s, \pi^*_h(s))\\
    &\leq c\frac{H^2}{\sigma \cdot \Pb^{\min}_*}  \sqrt{\frac{C^{\dagger}_{\sigma}S\log(SAH/\delta)}{L^\dagger_{\sigma} K }} + cH\left(1 + \frac{H\log(KHLSA/\delta)}{\sigma \cdot \Pb^{\min}_{*}}\right)\sqrt{\frac{\log(SAH/\delta)}{L^\dagger_{\sigma}}},
\end{align*}
which concludes the proof.
\end{proof}

\subsection{Auxiliary Lemmas}
\begin{lemma}[Lemma 8, \citet{shi2022pessimistic}]\label{lem:binomial_concentration}
    Suppose $N\sim \text{Binomial}(n,p)$, where $n\geq 1$ and $p\in [0, 1]$. For any $\delta\in [0,1]$, we have
    \begin{align*}
        N \geq \frac{np}{8\log(1/\delta)}, \qquad \text{if } np \geq 8\log(1/\delta); \qquad 
        N  \leq \begin{cases}
            e^2 np & \text{if } np \geq \log(1/\delta),\\
            2e^2 \log(1/\delta) & \text{if } np \leq 2\log(1/\delta)
        \end{cases}
    \end{align*}
    hold with probability at least $1-\delta$.
\end{lemma}

\begin{lemma}[Theorem 1, \citet{hu2013kullback}]\label{lem:dual}
    Suppose $f(x)$ has a finite moment generating function in some neighborhood around $x=0$, then for any $\sigma>0$ and a nominal distribution $\Pb^c$, we have
    \begin{align*}
        \sup_{\Pb \in \Uc^{\sigma}(\Pb^c)} \Eb_{X\sim \Pb} [f(X)] = \inf_{\lambda = 0} \left\{\lambda \log \left(\Eb_{X\sim \Pb^c}\left[\exp\left(\frac{f(X)}{\lambda}\right)\right]\right) + \lambda \sigma\right\}.
    \end{align*}
\end{lemma}

\begin{lemma}[Lemma 4, \citet{zhou2021finite}]\label{lem:dual_bound}
    Let $X\sim \Pb$ be a bounded random variable with $X\in [0, M]$. Let $\sigma>0$ be any uncertainty level and the corresponding optimal dual variable be
    \begin{align*}
        \lambda^* \in \argmax_{\lambda>0} f(\lambda, \Pb), \qquad \text{where } f(\lambda, \Pb):= -\lambda \log \left(\Eb_{X\sim \Pb} \left[\exp\left(-\frac{X}{\lambda}\right)\right]\right) - \lambda \sigma.
    \end{align*}
    Then the optimal value $\lambda^*$ obeys
    \begin{align*}
        \lambda^* \in \left[0, \frac{M}{\sigma}\right].
    \end{align*}
    Moreover, when $\lambda^*=0$, we have
    \begin{align*}
        \lim_{\lambda \to 0} f(\lambda, \Pb) = \essinf X.
    \end{align*}
\end{lemma}

\begin{lemma}[\citet{zhou2021finite}]\label{lem:essinf}
    Let $X\sim \Pb$ be a discrete bounded random variable with $X\in [0, M]$. Let $\Pb_n$ denote the empirical distribution constructed from $n$ independent samples $X_1, X_2, \cdots, X_n$ and let $\hat{X}\sim \Pb_n$. Denote $\Pb_{\min}:= \min\{\Pb_{X=x}: x\in \text{supp}(X)\}$. Then for any $\delta\in (0,1)$, with probability at least $1-\delta$, we have
    \begin{align*}
        \min_{i\in [n]} X_i = \essinf \hat{X} = \essinf X,
    \end{align*}
    as long as
    \begin{align*}
        n \geq - \frac{\log(2/\delta)}{\log(1-\Pb_{\min})}.
    \end{align*}
\end{lemma}

\section{Experimental Setups}\label{app:exp}
The target MDP is set to have $H = 20$ steps, $S=2$ states (labelled as $\{1, 2\}$), $A = 20$ actions (labelled as $\{1, 2, \cdots, 20\}$). The reward and transitions are specified as follows:
\begin{align*}
    &\forall h\in [H], r_h(s,a) = \begin{cases}
        0.9 & \text{if $(s,a) = (1, 1)$}\\
        0.1 & \text{otherwise}
    \end{cases}\\
    &\forall h\in [H], \Pb_h(1|s,a) = \begin{cases}
        0.9 & \text{if $(s,a) = (1, 1)$}\\
        0.5 & \text{otherwise}.
    \end{cases}\\
    &\forall h\in [H], \Pb_h(1|s,a) = \begin{cases}
        0.1 & \text{if $(s,a) = (1, 1)$}\\
        0.5 & \text{otherwise}.
    \end{cases}
\end{align*}
The rewards of the data source are independently sampled from Bernoulli distributions, i.e., $r_{h,l}(s,a) \sim \text{Bernoulli}(r_h(s,a))$, while the transitions are independently generated with standard Dirichlet distributions \citep{marchal2017sub}, i.e., $\Pb_{h,l}(\cdot|s,a) \sim \text{Dirichlet}(\Pb_{h}(\cdot|s,a))$. The behavior policy is shared by all data sources, which at each $(s, h)\in \Sc\times [H]$, selects action $1$ with probability $0.2$ and otherwise randomly chooses from other actions. The  results plotted in Fig.~\ref{fig:performance} are averaged from $100$ independently repeated experiments.

\end{document}